\documentclass[twoside,11pt]{article}

\usepackage{jmlr2e}
\usepackage{mathrsfs}
\usepackage{hyperref}
\usepackage{courier}
\usepackage{epsfig}
\usepackage{graphicx}
\usepackage{amsmath}
\usepackage{mathtools}
\usepackage{amssymb}
\usepackage{verbatim}
\usepackage{booktabs}
\usepackage{algpseudocode}
\usepackage{amscd}
\usepackage{times}
\usepackage{subfigure} 
\usepackage{algorithm}
\usepackage{color}
\usepackage{enumerate}
\usepackage{tikz}
\usepackage{epstopdf}
\usepackage{multirow}
\usepackage[
singlelinecheck=false 
]{caption}

\algrenewcommand\textproc{}

\newtheorem{thm}{Theorem}

\newtheorem{lem}[thm]{Lemma}
\newtheorem{prp}[thm]{Proposition}
\newtheorem{dfn}[thm]{Definition}
\newtheorem{prb}[thm]{Problem}

\def\B{\boldsymbol{B}}
\def\bC{\boldsymbol{C}}

\def\bE{\boldsymbol{E}}

\def\bL{\boldsymbol{L}}

\def\bO{\boldsymbol{\mathcal{O}}}
\def\T{\mathscr{T}}
\def\bT{\boldsymbol{T}}
\def\X{\mathcal{X}}
\def\Y{\mathcal{Y}}
\def\bY{\boldsymbol{\Y}}
\def\bX{\boldsymbol{\mathcal{X}}}
\def\btX{\boldsymbol{\tilde{\X}}}
\def\cS{\mathcal{S}}
\def\Sp{\mathbb{S}}
\def\H{\mathcal{H}}

\def\bI{\boldsymbol{I}}
\def\J{\mathcal{J}}

\def\Q{\mathscr{Q}}
\def\bP{\boldsymbol{P}}
\def\Re{\mathbb{R}}
\def\R{\mathcal{R}}
\def\bR{\boldsymbol{R}}

\def\U{\mathcal{U}}
\def\bU{\boldsymbol{U}}

\def\V{\mathcal{V}}

\def\bW{\boldsymbol{W}}
\def\a{\boldsymbol{a}}
\def\b{\boldsymbol{b}}
\def\ba{\boldsymbol{\alpha}}

\def\c{\boldsymbol{c}}
\def\eX{\epsilon_{\bX,N}}
\def\eO{\epsilon_{\bO,M}}

\def\bzeta{\boldsymbol{\zeta}}
\def\boldeta{\boldsymbol{\eta}}
\def\bxi{\boldsymbol{\xi}}

\def\d{\boldsymbol{d}}
\def\f{\boldsymbol{f}}
\def\h{\boldsymbol{h}}

\def\bn{\boldsymbol{n}}
\def\r{\boldsymbol{r}}
\def\o{\boldsymbol{o}}
\def\t{\boldsymbol{t}}
\def\u{\boldsymbol{u}}
\def\v{\boldsymbol{v}}
\def\y{\boldsymbol{y}}
\def\x{\boldsymbol{x}}
\def\w{\boldsymbol{w}}
\def\z{\boldsymbol{z}}
\def\xt{\tilde{x}}
\def\xtb{\boldsymbol{\xt}}
\def\e{\boldsymbol{e}}
\def\c{\boldsymbol{c}}

\def\0{\boldsymbol{0}}
\def\1{\boldsymbol{1}}
\def\transpose{\top}
\DeclareMathOperator*{\Trace}{Trace}
\DeclareMathOperator*{\Rank}{Rank}
\DeclareMathOperator*{\card}{Card}
\DeclareMathOperator*{\codim}{codim}
\DeclareMathOperator*{\Sgn}{Sgn}
\DeclareMathOperator*{\Span}{Span}
\DeclareMathOperator*{\Sign}{Sign}
\DeclareMathOperator*{\Conv}{Conv}
\DeclareMathOperator*{\Diag}{Diag}
\DeclareMathOperator*{\argmin}{argmin}
\DeclarePairedDelimiter{\ceil}{\lceil}{\rceil}

\newcommand{\myparagraph}[1]{\smallskip\noindent\textbf{#1.}}

\newcommand{\ra}[1]{\renewcommand{\arraystretch}{#1}}

\ShortHeadings{Dual Principal Component Pursuit}{Tsakiris and Vidal}
\firstpageno{1}

\begin{document} 

\graphicspath{{figures/}}

\title{Dual Principal Component Pursuit}

\author{Manolis C. Tsakiris \email mtsakiris@shanghaitech.edu.cn \\
School of Information Science and Technology \\
ShanghaiTech University \\
Pudong, Shanghai, China \\
\AND
Ren\'e Vidal \email rvidal@jhu.edu \\
Center for Imaging Science \\
Johns Hopkins University \\
Baltimore, MD, 21218, USA}

\editor{TBD}

\maketitle

\begin{abstract}
 
We consider the problem of learning a linear subspace from data corrupted by outliers. Classical approaches are typically designed for the case in which the subspace dimension is small relative to the ambient dimension. Our approach works with a dual representation of the subspace and hence aims to find its orthogonal complement; as such, it is particularly suitable for subspaces whose dimension is close to the ambient dimension (subspaces of high relative dimension). We pose the problem of computing normal vectors to the inlier subspace as a non-convex $\ell_1$ minimization problem on the sphere, which we call Dual Principal Component Pursuit (DPCP) problem. We provide theoretical guarantees under which every global solution to DPCP is a vector in the orthogonal complement of the inlier subspace. Moreover, we relax the non-convex DPCP problem to a recursion of linear programs whose solutions are shown to converge in a finite number of steps to a vector orthogonal to the subspace. In particular, when the inlier subspace is a hyperplane, the solutions to the recursion of linear programs converge to the global minimum of the non-convex DPCP problem in a finite number of steps. We also propose algorithms based on alternating minimization and iteratively re-weighted least squares, which are suitable for dealing with large-scale data. Experiments on synthetic data show that the proposed methods are able to handle more outliers and higher relative dimensions than current state-of-the-art methods, while experiments in the context of the three-view geometry problem in computer vision suggest that the proposed methods can be a useful or even superior alternative to traditional RANSAC-based approaches for computer vision and other applications.
\end{abstract}


\section{Introduction}\label{section:Introduction}
Principal Component Analysis (PCA) is one of the oldest \citep{Pearson-1901,Hotelling-1933} and most fundamental techniques in data analysis,
with ubiquitous applications in 
engineering \citep{moore1981principal}, 
economics and sociology \citep{Vyas:HPP2006}, 
chemistry \citep{Ku:Chemometrics95}, 
physics \citep{Loyd:NaturePhysics2014} and 
genetics \citep{Price:NatureGenetics2006} to name a few; see \cite{Jolliffe-2002} for more applications. Given a data matrix $\bX \in \Re^{D \times L}$ of $L$ data points of coordinate dimension $D$, PCA gives a closed
form solution to the problem of finding a $d$-dimensional linear subspace $\hat \cS$ that is closest, in the Euclidean sense, to the columns of $\bX$. Although the optimization problem associated with PCA is non-convex, it does admit a simple solution by means of the Singular Value Decomposition (SVD) of $\bX$. In fact, $\hat \cS$ 
is the subspace spanned by the first $d$ left singular vectors~of~$\bX$. 

Using $\hat{\cS}$ as a model for the data $\bX$ is meaningful when the data are known to have an approximately linear structure of underlying dimension $d$, i.e., they lie close to a $d$-dimensional subspace ${\cS}$. In practice, the principal components of $\bX$ are known to be well-behaved under mild levels of noise, i.e., the principal angles between $\hat{{\cS}}$ and ${\cS}$ are relatively small, and in fact, $\hat{{\cS}}$ is optimal when the noise is Gaussian \citep{Jolliffe-2002}. 
However, very often in applications the data are corrupted by outliers, i.e., the data matrix has the form $\btX = \left[\bX \, \bO\right] \boldsymbol{\Gamma}$, where the $M$ columns of $\bO \in \Re^{D \times M}$ are points of $\Re^D$ whose angles from the underlying ground truth subspace ${\cS}$ associated with the inlier points $\bX$ are large, and $\boldsymbol{\Gamma}$ is an unknown permutation. In such cases, the principal angles between ${\cS}$ and its PCA estimate $\hat{{\cS}}$ will in general be large, even when $M$ is small. This is to be expected since, by definition, the principal components of $\btX$ are orthogonal directions of maximal correlation with \emph{all} the points of $\btX$. This phenomenon, together with the fact that outliers are almost always present in real datasets, has given rise to the important problem of outlier detection in PCA.

Traditionally, outlier detection has been a major area of study in robust statistics with notable methods being \emph{Influence-based Detection}, \emph{Multivariate Trimming}, $M$\emph{-Estimators}, \emph{Iteratively Reweighted Least Squares} (IRLS) and \emph{Random Sampling Consensus} (RANSAC) \citep{Huber-1981,Jolliffe-2002}. These methods are usually based on non-convex optimization problems and in practice converge only to a local minimum. In addition, their theoretical analysis is usually limited and their computational complexity may be large (e.g., in the case of RANSAC). Recently, two attractive methods have 
appeared \citep{Xu:TIT12,Soltanolkotabi:AS12} that are directly based on convex optimization and are inspired by
\emph{low-rank representation} \citep{Liu:ICML10} and \emph{compressed sensing} \citep{Candes:SPM08}. Even though both of these methods admit theoretical guarantees and efficient implementations, they are in principle applicable only in the case of subspaces of small relative dimension (i.e., $d/D \ll 1$). On the other hand, the theoretical guarantees of the recent REAPER method of \cite{Lerman:FCM15} seem to suggest that the method is able to handle any subspace dimension.

In this paper we adopt a \emph{dual} approach to the problem of robust PCA in the presence of outliers, which allows us to explicitly transcend the low relative dimension regime of modern methods such as \cite{Xu:TIT12} or \cite{Soltanolkotabi:AS12}, and even be able to handle as many as  $70\%$ outliers for hyperplanes (subspaces of maximal relative dimension $(D-1)/D$), a regime where other modern \citep{Lerman:FCM15} or classic \citep{Huber-1981} methods fail. The key idea of our approach comes from the fact that, in the absence of noise, the inliers $\bX$ lie inside any hyperplane $\H_1=\Span(\b_1)^\perp$ that contains the underlying linear subspace ${\cS}$ associated with the inliers. This suggests that, instead of attempting to fit directly a low-dimensional linear subspace to the entire dataset $\btX$, as done e.g. in \cite{Xu:TIT12}, we can search for a \emph{maximal hyperplane} $\H_1$ that contains as many points of the dataset as possible. When the inliers $\bX$ are in general position (to be made precise shortly) inside $\cS$, and the outliers $\bO$ are in general position in $\Re^D$, such a maximal hyperplane will contain the entire set of inliers together with possibly a few outliers. Then one may remove all points that lie outside this hyperplane and be left with an easier robust PCA problem that could potentially be addressed by existing methods. Alternatively, one can continue by finding a second maximal hyperplane $\H_2=\Span(\b_2)^\perp$, with the new dual principal component $\b_2$ perpendicular to the first one, i.e., $\b_2  \perp \b_1$, and so on, until $c:=D-d$ such maximal hyperplanes $\H_1,\dots,\H_c$ have been found, leading to a \emph{Dual Principal Component Analysis (DPCA)} of $\btX$. In such a case, the inlier subspace is precisely equal to $\bigcap_{i=1}^c \H_i$, and a point is an outlier if and only if it lies outside this intersection. 

We formalize the problem of searching for maximal hyperplanes with respect to $\btX$ as an $\ell_0$ cosparsity-type problem \citep{nam2013cosparse}, which we relax to a non-convex $\ell_1$ problem on the sphere, referred to as the Dual Principal Component Pursuit (DPCP) problem. We provide theoretical guarantees under which every global solution of the DPCP problem is a vector orthogonal to the linear subspace associated with the inliers, i.e., it is a dual principal component. Moreover, we relax the non-convex DPCP problem to a recursion of linear programming problems and we show that, under mild conditions, their solutions converge to a dual principal component in a finite number of steps. In particular, when the inlier subspace is a hyperplane, then the solutions of the linear programming recursion converge to the global minimum of the non-convex problem in a finite number of steps. Furthermore, we propose algorithms based on alternating minimization and IRLS that are suitable for dealing with large-scale data. Extensive experiments on synthetic data show that the proposed methods are able to handle more outliers and subspaces of higher relative dimension $d/D$ than state-of-the-art methods \citep{RANSAC,Xu:TIT12,Soltanolkotabi:AS12,Lerman:FCM15}, while experiments with real face and object images show that our DPCP-based methods perform on par with state-of-the-art methods.

\paragraph{Notation} The shorthand RHS stands for \emph{Right-Hand-Side} and similarly for LHS. The notation $\cong$ stands for \emph{isomorphism} in whatever category the objects lying to the LHS and RHS of the symbol belong to. The notation $\simeq$ denotes approximation. For any positive integer $n$ let $[n]:=\left\{1,2,\hdots,n\right\}$. For any positive number $\alpha$ let $\ceil{\alpha}$ denote the smallest integer that is greater than $\alpha$. For sets $\mathcal{A}, \mathcal{B}$, the set $\mathcal{A} \setminus \mathcal{B}$ is the set of all elements of $\mathcal{A}$ that do not belong to $\mathcal{B}$. If $\cS$ is a subspace of $\Re^D$, then $\dim(\cS)$ denotes the dimension of $\cS$ and $\pi_{\cS}: \Re^D \rightarrow \cS$ is the orthogonal projection of $\Re^D$ onto $\cS$. For vectors $\b, \b' \in \Re^D$ we let $\angle \b,\b'$ be the acute angle between $\b$ and $\b'$, defined as the unique angle $\theta \in \left[0 \, \, 90^\circ\right]$ such that $\cos\theta = \left|\b^\transpose \b' \right|$. If $\b$ is a vector of $\Re^D$ and $\cS$ a linear subspace of $\Re^D$, the principal angle of $\b$ from $\cS$ is  $\angle \b,\pi_{\cS}(\b)$. The symbol $\oplus$ denotes direct sum of subspaces. The 
orthogonal complement of a subspace $\cS$ in $\Re^D$ is $\cS^\perp$. If $\y_1, \dots, \y_s$ are elements of $\Re^D$, we denote by $\Span(\y_1,\dots,\y_s)$ the subspace of $\Re^D$ spanned by these elements. $\Sp^{D-1}$ denotes the unit sphere of $\Re^D$. For a vector $\w \in \Re^D$ we define $\hat{\w}:=\w/\big \|\w\big \|_2$, if $\w \neq \0$, and $\hat{\w}:=0$ otherwise. Given a square matrix $\bC$, $\Diag(\bC)$ denotes the vector of diagonal elements of $\bC$. Given a square matrix $\bP$, the notation $\0 \le \bP \le \bI$ indicates that $\bP, \bI -\bP$ are positive semi-definite matrices. With a mild abuse of notation we will be treating on several occasions matrices as sets, i.e., if $\bX$ is $D \times N$ and $\x$ a point of $\Re^D$, the notation $\x \in \bX$ signifies that $\x$ is a column of $\bX$. Similarly, if $\bO$ is a $D \times M$ matrix, the notation $\bX \cap \bO$ signifies the points of $\Re^D$ that are common columns of $\bX$ and $\bO$. The notation $\Sign$ denotes the sign function $\Sign: \Re \rightarrow \left\{-1,0,1\right\}$ defined as
\begin{align}
\Sign(x) = \left\{
\begin{array}{ll}
x/|x| &  \text{if} \, \, \,    x \neq 0, \\
0 &  \text{if} \, \, \,    x = 0. 
\end{array}
\right.
\end{align}
Finally, we note that the $i$th entry of the subdifferential of the $\ell_1$-norm $\|\z\|_1=\sum_{i=1}^D |z_i|$ of a vector $\z=(z_1,\dots,z_D)^\transpose$ is a set-valued function on $\Re^D$ defined as
\begin{align}
\Sgn(z_i) = 
\begin{cases}
\Sign(z_i) &  \text{if} \, \, \,    z_i \neq 0, \\
\left[-1,1\right] & \text{if} \, \, \,  z_i=0. 
\end{cases}
\end{align}

\section{Prior Art} \label{section:RelatedWork}
We begin by briefly reviewing some state-of-the-art methods for learning a linear subspace from data $\btX = \left[\bX \, \bO\right] \boldsymbol{\Gamma}$ in the presence of outliers. The literature on this subject is vast and our account is far from exhaustive; with a few exceptions, we mainly focus on modern methods based on convex optimization. For methods from robust statistics see \cite{Huber-1981, Jolliffe-2002}, for
online subspace learning methods see \cite{Balzano:Allerton10,Feng:NIPS13}, for regression-type methods see \cite{WangCamps:CVPR15}, while for fast and other methods the reader is referred to the excellent literature review of \cite{Lerman:CA14} or the recent survey by \cite{Lerman:arXiv18}. 

\noindent Finally, we note that preliminary results associated with the present work have been published in the form of a conference paper\footnote{We note that the proof of Theorem $2$ in \cite{Tsakiris:DPCPICCV15} contained an inaccuracy which makes its statement incomplete. The complete statement is Theorem \ref{thm:DiscreteNonConvex} in the present paper.} \citep{Tsakiris:DPCPICCV15}. While the present paper was under review, we extended our approach to clustering data from multiple subspaces \citep{Tsakiris:ICML17}, which can also be thought of as a robust PCA problem but with structured outliers. This is a continuation of the present work, which certainly builds on the concepts and algorithms presented here, yet requires sufficiently distinct machinery to be fully established.

\paragraph{RANSAC} One of the oldest and most popular outlier detection methods for PCA is \emph{Random Sampling Consensus (RANSAC)} \citep{RANSAC}. The idea behind RANSAC is simple: alternate between randomly sampling a subset of cardinality $d$ from the dataset and computing a $d$-dimensional subspace from this subset, until a subspace $\hat{\cS}$ is found that maximizes the number of points in the entire dataset that approximately lie in $\hat{\cS}$ within some error. RANSAC is typically used when the ambient dimension $D$ is small (say $D \le 20$), yielding high quality subspace estimates regardless of the subspace relative dimension $d/D$. However, as $D$ increases, RANSAC becomes inefficient for large values of $d/D$, except when the outlier ratio is very small, as otherwise a prohibitive number of trials may be required in order to obtain outlier-free samples and thus furnish reliable models. Additionally, RANSAC requires as input an estimate for the dimension of the subspace as well as a thresholding parameter, which is used to distinguish outliers from inliers; naturally the performance of RANSAC is very sensitive to these two parameters. 

\paragraph{$\ell_{2,1}$-RPCA} Unlike RANSAC, modern methods for outlier detection in PCA are primarily based on convex optimization. One of the earliest and most important such methods is the $\ell_{2,1}$-RPCA method of \cite{Xu:TIT12}, which is in turn inspired by the Robust Principal Component Analysis (RPCA) algorithm of \cite{Candes:ACM11}. $\ell_{2,1}$-RPCA  computes a $(\ell_*+\ell_{2,1})$-norm decomposition\footnote{Here $\ell_*$ denotes the nuclear norm, which is the sum of the singular values of the matrix. Also, $\ell_{2,1}$ is defined as the sum of the Euclidean norms of the columns of a matrix.} of the data matrix, instead of the $(\ell_*+\ell_{1})$-decomposition in \cite{Candes:ACM11}. More specifically, 
$\ell_{2,1}$-RPCA solves the optimization problem
\begin{align}
\min_{\bL,\bE: \, \btX=\bL+\bE} \, \, \, \big \|\bL\big \|_* + \lambda \, \big \|\bE\big \|_{2,1}, \label{eq:RPCA21}
\end{align} which attempts to decompose the data matrix $\btX=[\bX \, \bO] \boldsymbol{\Gamma}$ into the sum of a low-rank matrix $\bL$, and a matrix $\bE$ that has only a few non-zero columns. The idea is that $\bL$ is associated with the inliers, having the form $\bL = [\bX \, \0_{D \times M}] \boldsymbol{\Gamma}$, and $\bE$ is associated with the outliers, having the form $\bE = [\0_{D \times N} \, \bO] \boldsymbol{\Gamma}$. The optimization problem \eqref{eq:RPCA21} is convex and admits theoretical guarantees and efficient algorithms based on the alternating direction method of multipliers (ADMM) \citep{Gabay:CMA76}. However, it is expected to succeed only when the intrinsic dimension $d$ of the inliers is small enough (otherwise $[\bX \, \0_{D \times M}]$ will not be low-rank), and the outlier ratio is not too large (otherwise $[\0_{D \times N} \, \bO]$ will not be column-sparse). Finally, notice that $\ell_{2,1}$-RPCA does not require as input
the subspace dimension $d$, because it does not directly compute an estimate for the subspace. Rather, the subspace can be obtained subsequently by applying classic PCA on $\bL$, and now one does need an estimate for $d$.

\paragraph{SE-RPCA} Separating outliers from low-rank inlier points can also be achieved by exploiting the \emph{self-expressiveness} (SE) property of the data matrix, a notion popularized by the work of \cite{Elhamifar:CVPR11,Elhamifar:TPAMI13} in the area of subspace clustering \citep{Vidal:SPM11-SC}. Specifically, if a column $\tilde\x$ of $\btX$ is an inlier, then it can be expressed as a linear combination of $d$ other inliers in $\btX$, while if $\tilde\x$ is an outlier, then in principle it requires $D$ columns of $\btX$. The coefficient matrix $\bC$ can be obtained as the solution to the convex optimization problem
\begin{align}
\min_{\bC} \, \, \, \|\bC\|_1 \, \, \, \text{s.t.} \, \, \, \btX = \btX \bC, \, \, \, \Diag(\bC)=\0,
\end{align} 
where the extra constraint prevents the trivial solution $\bC = I$. Given $\bC$, and under the hypothesis that $d/D$ is small, a column of $\btX$ is declared as an outlier, if the $\ell_1$ norm of the corresponding column of $\bC$ is large; see \cite{Soltanolkotabi:AS12} for an explicit formula. SE-RPCA admits theoretical guarantees \citep{Soltanolkotabi:AS12} and efficient ADMM implementations \citep{Elhamifar:TPAMI13}. Moreover, the recent work of \cite{You:CVPR17} has demonstrated that the information contained in the self-expressive matrix $\bC$ can be further exploited to identify the outliers by means of a random walk on the directed affinity graph defined by $\bC$, thus yielding superior results than the simple thresholding of the norms of the columns of $\bC$. In contrast to $\ell_{2,1}$-RPCA, which in principle fails in the presence of a very large number of outliers, SE-RPCA is still expected to perform well, since the existence of sparse \emph{subspace-preserving} self-expressive patterns does not depend on the number of outliers present. Also, similarly to $\ell_{2,1}$-RPCA, SE-RPCA does not directly require an estimate for the subspace dimension $d$. Nevertheless, knowledge of $d$ is necessary if one wants to furnish an actual subspace estimate, which entails removing the outliers (a judiciously chosen threshold would also be necessary here) and applying PCA.

\paragraph{REAPER} Another robust subspace learning method that admits an interesting theoretical analysis is \emph{REAPER} \citep{Lerman:FCM15}, which is conceptually associated with the optimization problem
\begin{align}
\min_{\bP} \sum_{j=1}^{L} \big \|(\bI_D - \bP)  \xtb_j \big \|_2 \, \, \, \text{s.t.} \, \, \, \bP \, \, \, \text{is an orthogonal projection ~~ and} \, \, \,  \Trace(\bP) = d. \label{eq:ReaperNonConvex}
\end{align} 
Here the vector $\xtb_j$ denotes the $j$-th column of $\btX$ and the matrix $\bP$ denotes the orthogonal projection onto a $d$-dimensional linear subspace $\cS$. Notice that $\bP$ can be thought of as the product $\bP = \bU \bU^\transpose$, where the columns of $\bU \in \Re^{D \times d}$ form an orthonormal basis for $\cS$. Since the problem in \eqref{eq:ReaperNonConvex} is non-convex, \cite{Lerman:FCM15} relaxed it to the convex semi-definite program
\begin{align}
\min_{\bP} \sum_{j=1}^{L} \big \|(\bI_D - \boldsymbol{\bP})  \xtb_j \big \|_2 \, \, \, \text{s.t.} \, \, \, \0 \le \boldsymbol{\bP} \le \bI_D, \, \, \,  \Trace\left(\boldsymbol{\bP} \right) = d, \label{eq:ReaperConvex}
\end{align} 
and obtained an approximate solution $\bP^*$ to \eqref{eq:ReaperNonConvex} as the rank-$d$ orthogonal projector that is closest to the global solution of \eqref{eq:ReaperConvex}, as measured by the nuclear norm. It was shown by \cite{Lerman:FCM15} that the orthoprojector $\bP^*$ obtained in this way is within a neighborhood of the orthoprojector corresponding to the true underlying inlier subspace. In practice, the semi-definite program \eqref{eq:ReaperConvex} may become prohibitively expensive to solve even for moderate values of the ambient dimension $D$. As a consequence, the authors proposed an \emph{Iteratively Reweighted Least Squares (IRLS)} scheme to obtain a numerical solution of \eqref{eq:ReaperConvex}, whose objective value was shown to converge to a neighborhood of the optimal objective value of problem \eqref{eq:ReaperConvex}. 

One advantage of REAPER with respect to $\ell_{2,1}$-RPCA and SE-RPCA, is that its theoretical conditions allow for the subspace to have arbitrarily large relative dimension, providing that the outlier ratio is sufficiently small. It is interesting to note here that this is precisely the condition under which RANSAC \citep{Huber-1981} can handle large relative dimensions; the main difference though between RANSAC and REAPER is that the latter employs convex optimization, and for a fixed relative dimension and computational budget REAPER can tolerate considerably higher outlier ratios than RANSAC (see Fig. \ref{figure:separation}, \S \ref{section:Experiments}). 

\paragraph{COHERENCE PURSUIT (CoP)} The recent work of \cite{Rahmani:arXiv17} analyzes a simple yet efficient algorithm for detecting the inlier space from pairwise point coherences, hence called \emph{Coherence Pursuit (CoP)}. The main insight behind CoP is that, under the hypothesis that the inliers lie in a low-dimensional subspace, inlier points tend to have significantly higher coherence with the rest of the points in the dataset, than outlier points. Hence the columns of the pairwise coherence matrix $\btX^\top \btX$ that have large norm are expected to correspond to inliers. Indeed, CoP orders the columns of $\btX^\top \btX$ in descending values for a given norm and selects sufficiently many of them from the top, until their span yields a $d$-dimensional subspace. Similarly to SE-RPCA, the performance of CoP is not expected to be significantly degraded as the number of outliers increases, as long as each outlier remains sufficiently incoherent with the rest of the dataset. As demonstrated by \cite{Rahmani:arXiv17}, CoP has a competitive performance and admits an extensive theoretical analysis.

\paragraph{L1-PCA$^*$} Finally, we mention the method L1-PCA$^*$ of \cite{Brooks:CSDA13}, since it works with the orthogonal complement of the subspace, similarly to the proposed method of the present paper. Nevertheless, L1-PCA$^*$ is slightly unusual in that it learns $\ell_1$ hyperplanes, i.e., hyperplanes that minimize the $\ell_1$ distance to the points, as opposed to the Euclidean distance used by methods such as PCA and REAPER. Overall, no theoretical guarantees seem to be known for L1-PCA$^*$, as far as the subspace learning problem is concerned. In addition, L1-PCA$^*$ requires solving $\mathcal{O}(D^2)$ linear programs, where $D$ is the ambient dimension, which makes it computationally expensive.

\section{Problem Formulation}
 \label{section:ProblemFormulation}
 
 
In this section we formulate the problem addressed in this paper. We describe our data model (\S \ref{subsection:DataModel}), and motivate the problem at a conceptual (\S \ref{subsection:ConceptualFormulation}) and computational level (\S \ref{subsection:ComputationalFormulation}).

\subsection{Data model} \label{subsection:DataModel}
We employ a deterministic noise-free data model, under which the given data is
\begin{align}
\btX = [\bX \, \, \, \bO] \boldsymbol{\Gamma} = [
\xtb_1, \dots, \xtb_{L}
] \in \Re^{D \times L},
\end{align} 
where the $N$ inliers $\bX = [\x_1, \dots, \x_N] \in \Re^{D \times N}$ lie in the intersection of the unit sphere $\Sp^{D-1}$ with an unknown proper subspace ${\cS}$ of $\Re^D$ of unknown dimension $1 \le d \le D-1$, and the $M$ outliers $\bO = [\o_1, \dots, \o_M] \in \Re^{D \times M}$ lie on the sphere $\Sp^{D-1}$. The unknown permutation $\boldsymbol{\Gamma}$ indicates that we do not know which point is an inlier and which point is an outlier. Finally, we assume that the points $\btX$ are in \emph{general position}, in the sense that there are no relations among the columns of $\btX$ except for those implied by the inclusions $\bX \subset \cS$ and $\btX \subset \Re^D$. In particular, every $D$-tuple of columns of $\btX$ such that at most $d$ points come from $\bX$ is linearly independent. Notice that as a consequence every $d$-tuple of inliers and every $D$-tuple of outliers are linearly independent, and also $\bX \cap \bO = \emptyset$. Finally, to avoid degenerate situations we will assume that $N \ge d+1$ and $M \ge D-d$.\footnote{If the number of outliers is less than $D-d$, then the entire dataset is degenerate because it lies in a proper hyperplane of the ambient space, hence we can reduce the coordinate representation of the data and eventually satisfy the stated condition.} 

\subsection{Conceptual formulation} \label{subsection:ConceptualFormulation}

Given $\btX$, we consider the problem of partitioning its columns into those that lie in $\cS$ and those that don't. Since we have made no assumption about the dimension of $\cS$, this problem is however not well posed because $\cS$ can be anything from a line to a $(D-1)$-dimensional hyperplane, and hence $\bX$ lies inside every subspace that contains $\cS$, which in turn may contain some elements of $\bO$. Instead, it is meaningful to search for a linear subspace of $\Re^D$ that contains all of the inliers and perhaps a few outliers. Since we do not know the intrinsic dimension $d$ of the inliers, a natural choice is to search for a hyperplane of $\Re^D$ that contains all the inliers.

\begin{prb} \label{prp:FriendlyHyperplane}
	Given the dataset $\btX = \left[
	\bX \, \, \,  \bO
	\right] \boldsymbol{\Gamma}$, find a hyperplane $\H$ that contains all the inliers $\bX$. 
\end{prb} 

Notice that hyperplanes that contain all the inliers always exist: any non-zero vector $\b$ in the orthogonal complement $\cS^\perp$ of the linear subspace $\cS$ associated with the inliers defines a hyperplane (with normal vector $\b$) that contains all inliers $\bX$. Having such a hyperplane $\H_1$ at our disposal, we can partition our dataset as $\btX = \btX_1 \cup \btX_2$, where $\btX_1$ are the points of $\btX$ that lie in $\H_1$ and $\btX_2$ are the remaining points. Then by definition of $\H_1$, we know that $\btX_2$ will consist purely of outliers, in which case we can safely replace our original dataset $\btX$ with $\btX_1$ and reconsider the problem of robust PCA on $\btX_1$. We emphasize that $\btX_1$ will contain all the inliers $\bX$ together with at most $D-d-1$ outliers,\footnote{This comes from the assumption of general position.} a number which may be dramatically smaller than the original number of outliers. Then one may apply existing methods such as
\cite{Xu:TIT12}, \cite{Soltanolkotabi:AS12} or \cite{RANSAC} to finish the task of identifying the remaining outliers, as the following example demonstrates.

\begin{example} Suppose we have $N=1000$ inliers lying in general position in a linear subspace of $\Re^{100}$ of dimension $d=90$. Suppose that the dataset is corrupted by $M=1000$ outliers lying in general position in $\Re^{100}$. Let $\H$ be a hyperplane that contains all $1000$ inliers. Since the dimensionality of the inliers is $90$ and the dimensionality of the hyperplane is $99$, there are only $99-90=9$ linearly independent directions left for the hyperplane to fit, i.e., $\H$ will contain at most $9$ outliers (it can not contain more outliers since this would violate the general position hypothesis). If we remove the points of the dataset that do not lie in $\H$, then we are left with $1000$ inliers and at most~$9$~outliers. A simple application of RANSAC is expected to identify the remaining outliers in only a few trials.
\end{example}

Alternatively, if the true dimension $d$ is known, one may keep working with the entire dataset $\btX$ (i.e., no point removal takes place) and search for a second hyperplane $\H_2$ that contains all the inliers, such that its normal vector $\b_2$ is linearly independent (e.g., orthogonal) from the normal vector $\b_1$ of $\H_1$. Then $\H_1 \cap \H_2$ is a linear subspace of dimension $D-2$ that contains all the inliers $\bX$, and as a consequence $\Span(\bX) = \cS \subset \H_1 \cap \H_2$. Then a third hyperplane $\H_3 \supset \bX$ may be sought for, such that its normal vector $\b_3$ is not in $\Span(\b_1,\b_2)$, and so on. Repeating this process $c=\codim {\cS} = D-d$ times, until $c$ linearly independent hyperplanes\footnote{By the hyperplanes being linearly independent we mean that their normal vectors are linearly independent.} $\H_1,\dots,\H_c$ have been found, each containing $\bX$, we arrive at a situation where $\bigcap_{k=1}^c \H_k$ is a subspace of dimension $d=\dim{\cS}$ that contains $\cS$ and thus it must be the case that $\bigcap_{k=1}^c \H_k = {\cS}$. Hence we may declare a point to be an inlier if and only if the point lies in the intersection of these $c$ hyperplanes.

\subsection{Hyperplane pursuit by $\ell_1$ minimization} \label{subsection:ComputationalFormulation}
In this section we propose an optimization framework for the computation of a hyperplane that solves Problem \ref{prp:FriendlyHyperplane}, i.e., a hyperplane that contains all the inliers. To proceed, we need a definition.

\begin{dfn}
	A hyperplane $\H$ of $\Re^D$ is called maximal with respect to the dataset $\btX$, if it contains a maximal number of data points in $\btX$, i.e., if for any other hyperplane $\H^\dagger$ of $\Re^D$ we have that $\card(\btX \cap \H) \ge \card(\btX \cap \H^\dagger)$.
\end{dfn} In principle, hyperplanes that are maximal with respect to $\btX$, always solve Problem \ref{prp:FriendlyHyperplane}, as the next proposition shows (see \S \ref{subsection:ProofMaximalFriendlyPlanes} for the proof).

\begin{prp} \label{prp:MaximalFriendlyPlanes}
	Suppose that $N \ge d+1$ and $M \ge D-d$, and let $\H$ be a hyperplane that is maximal with respect to the dataset $\btX$. Then $\H$ contains all the inliers $\bX$. 
\end{prp} In view of Proposition \ref{prp:MaximalFriendlyPlanes}, we may restrict our search for hyperplanes that contain all the inliers $\bX$ to the subset of hyperplanes that are maximal with respect to the dataset $\btX$. 
The advantage of this approach is immediate: the set of hyperplanes that are maximal with respect to $\btX$ is in principle computable, since it is precisely the set of solutions of the following optimization problem 
\begin{align}
\min_{\b} \, \big \|\btX^\transpose \b \big \|_0 \, \, \, \text{s.t.} \, \, \, \b \neq 0. \label{eq:ell0}
\end{align} 
The idea behind \eqref{eq:ell0} is that a hyperplane $\H = \Span(\b)^\perp$ contains a maximal number of columns of $\btX$ if and only if its normal vector $\b$ has a maximal \emph{cosparsity} level with respect to the matrix $\btX^\transpose$, i.e., the number of non-zero entries of $\btX^\transpose \b$ is minimal.
Since \eqref{eq:ell0} is a combinatorial problem admitting no efficient solution, we consider its natural relaxation
\begin{align}
\min_{\b} \, \big \|\btX^\transpose \b \big\|_1 \, \, \, \text{s.t.} \, \, \, \big \|\b\big \|_2 = 1, 
\label{eq:ell1}
\end{align}  
which in our context we will be referring to as \emph{Dual Principal Component Pursuit} (DPCP). A major question that arises, to be answered in Theorem \ref{thm:DiscreteNonConvex}, is under what conditions every global solution of \eqref{eq:ell1} is orthogonal to the inlier subspace $\Span(\bX)$. A second major question, raised by the non-convexity of the constraint $\b \in \Sp^{D-1}$, is how to efficiently solve
\eqref{eq:ell1} with theoretical guarantees.

We emphasize here that the optimization problem \eqref{eq:ell1} is far from new; interestingly, its earliest appearance in the literature that we are aware of is in \cite{Spath:Numerische87}, where the authors proposed to solve it by means of the recursion of convex problems given by\footnote{Being unaware of the work of \cite{Spath:Numerische87}, we independently proposed the same recursion in \citep{Tsakiris:DPCPICCV15}.}
\begin{align}
\bn_{k+1} := \argmin_{\b^\transpose \hat{\bn}_k=1} \big \|\btX^\transpose \b \big\|_1. 
\label{eq:ConvexRelaxations}
\end{align} 
Notice that at each iteration of \eqref{eq:ConvexRelaxations} the problem that is solved is computationally equivalent to a linear program; this makes the recursion \eqref{eq:ConvexRelaxations} a very appealing candidate for solving 
the non-convex \eqref{eq:ell1}. Even though \cite{Spath:Numerische87} proved the very interesting result that \eqref{eq:ConvexRelaxations} converges to a critical point of \eqref{eq:ell1} in a finite number of steps (see Appendix \ref{appendix:Spath}), there is no reason to believe that in general \eqref{eq:ConvexRelaxations} converges to a global minimum of \eqref{eq:ell1}.

 Other works in which optimization problem \eqref{eq:ell1} appears are \cite{spielman2013exact,Qu:NIPS14,Sun:ICML15,Sun:SampTA15,Sun:CompleteIarXiv15,Sun:CompleteIIarXiv15}. More specifically, \cite{spielman2013exact} propose to solve \eqref{eq:ell1} by replacing the quadratic constraint $\b^\transpose \b=1$ with a linear constraint $\b^\transpose \w =1$ for some vector $\w$. In \cite{Qu:NIPS14,Sun:CompleteIarXiv15} \eqref{eq:ell1} is approximately solved by alternating minimization, while a Riemannian trust-region approach is employed in \cite{Sun:CompleteIIarXiv15}. Finally, we note that problem \eqref{eq:ell1} is closely related to the non-convex problem \eqref{eq:ReaperNonConvex} associated with REAPER. To see this, suppose that the REAPER orthoprojector $\boldsymbol{\Pi}$ appearing in \eqref{eq:ReaperNonConvex}, represents the orthogonal projection  to a hyperplane $\U$ with unit-$\ell_2$ normal vector $\b$. In such a case $\bI_D - \boldsymbol{\Pi} = \b \b^\transpose$ and it readily follows that problem \eqref{eq:ReaperNonConvex} becomes identical to problem \eqref{eq:ell1}. 


\section{Dual Principal Component Pursuit Theory} 
\label{section:TheoreticalContributions}
In this section we establish our analysis framework and discuss our main theoretical results regarding the global optimum of the non-convex problem \eqref{eq:ell1} as well as the recursion of convex relaxations in \eqref{eq:ConvexRelaxations}. We begin our theoretical investigation in \S \ref{subsection:Continuous} by establishing a connection between the \emph{discrete} problems \eqref{eq:ell1} and \eqref{eq:ConvexRelaxations} and certain underlying \emph{continuous} problems. The continuous problems do not depend on a finite set of inliers and outliers, rather on uniform distributions on the respective inlier and outlier spaces, and as such, are easier to analyze. The analysis of Theorems \ref{thm:ContinuousNonConvex} and \ref{thm:ConvexRelaxationsContinuous} reveals that the optimal solutions of the continuous analogue of \eqref{eq:ell1} are orthogonal to the inlier space, and that the solutions of the continuous recursion corresponding to \eqref{eq:ConvexRelaxations} converge to a normal vector to the inlier space, respectively.
This suggests that under certain conditions on the distribution of the data, the same must be true for the \emph{discrete} problem \eqref{eq:ell1} and the \emph{discrete} recursion \eqref{eq:ConvexRelaxations}, where the adjective \emph{discrete} refers to the fact that these problems depend on a finite set of points. Our analysis of the discrete problems is inspired by the analysis of their continuous counterparts and the link between the two is formally captured through certain \emph{discrepancy} bounds that we introduce in \S \ref{subsection:DiscrepancyBounds}. In turn, these allow us to prove conditions under which we can characterize the global optimal of problem \eqref{eq:ell1} as well as the convergence of recursion \eqref{eq:ConvexRelaxations}; this is done in \S \ref{subsection:ConditionsDiscrete} and the main results are Theorems \ref{thm:DiscreteNonConvex} and \ref{thm:DiscreteConvexRelaxations}, which are analogues of Theorems \ref{thm:ContinuousNonConvex} and \ref{thm:ConvexRelaxationsContinuous}.
These theorems suggest that both \eqref{eq:ell1} and \eqref{eq:ConvexRelaxations} are natural formulations for computing the orthogonal complement of a linear subspace in the presence of outliers. The proofs of all theorems as well as intermediate results are deferred to \S \ref{section:Proofs}.

\subsection{Formulation and theoretical analysis of the underlying continuous problems} \label{subsection:Continuous}  

In this section we show that the problems of interest \eqref{eq:ell1} and
\eqref{eq:ConvexRelaxations} can be viewed as discrete versions of certain continuous problems, which are easier to analyze. To begin with, consider given
outliers $\bO = [\o_1,\dots,\o_M] \subset \Sp^{D-1}$ and inliers $\bX=[\x_1,\dots,\x_N] \subset \cS \cap \Sp^{D-1}$, and recall the notation $\btX = [\bX \, \bO] \boldsymbol{\Gamma}$, where $\boldsymbol{\Gamma}$ is an unknown permutation. Next, for any $\b \in \Sp^{D-1}$ define
the function $f_{\b}: \Sp^{D-1} \rightarrow \Re$ by $f_{\b}(\z) = \left| \b^\transpose \z \right|$.
Define also \emph{discrete} measures $\mu_{\bO}$ and $\mu_{\bX}$ on $\Sp^{D-1}$ associated with the outliers and inliers respectively, as
\begin{align}
\mu_{\bO}(\z) = \frac{1}{M} \sum_{j=1}^M \delta(\z - \o_j) \, \, \, \text{and} \, \, \,  \mu_{\bX}(\z)=  \frac{1}{N} \sum_{j=1}^N \delta(\z - \x_j), \label{eq:DiscreteMeasures}
\end{align} 
where $\delta(\cdot)$ is the Dirac function on $\Sp^{D-1}$, satisfying
\begin{align} 
\int_{\z \in \Sp^{D-1}} g(\z) \delta(\z - \z_0) d \mu_{\Sp^{D-1}} = g(\z_0),
\end{align} for every $g: \Sp^{D-1} \rightarrow \Re$ and every $\z_0 \in \Sp^{D-1}$;   $\mu_{\Sp^{D-1}}$ is the uniform measure on $\Sp^{D-1}$.
 
With these definitions, we have that the objective
function $\big \| \btX^\transpose \b \big\|_1$ appearing in \eqref{eq:ell1} and \eqref{eq:ConvexRelaxations} is the sum of the weighted expectations of the function
$f_{\b}$ under the measures $\mu_{\bO}$ and $\mu_{\bX}$, i.e.,
\begin{align}
\big\| \btX^\transpose \b \big\|_1 &= \big\| \bO^\transpose \b \big\|_1+
\big\| \bX^\transpose \b \big\|_1 = \sum_{j=1}^M \big| \b^\transpose \o_j \big| +  \sum_{j=1}^N \big| \b^\transpose \x_j \big| \\
&=\sum_{j=1}^ M \int_{\z \in \Sp^{D-1}} \big| \b^\transpose \z \big| \delta(\z - \o_j) d \mu_{\Sp^{D-1}} + \sum_{j=1}^ N \int_{\z \in \Sp^{D-1}} \big| \b^\transpose \z \big| \delta(\z - \x_j) d \mu_{\Sp^{D-1}} \\
&= \int_{\z \in \Sp^{D-1}} \big | \b^\transpose \z \big| \sum_{j=1}^ M \delta(\z - \o_j)  d \mu_{\Sp^{D-1}} + 
\int_{\z \in \Sp^{D-1}} \big | \b^\transpose \z \big | \sum_{j=1}^ N \delta(\z - \x_j)  d \mu_{\Sp^{D-1}}  \\
&=  M \, \mathbb{E}_{\mu_{\bO}}(f_{\b}) + N \, \mathbb{E}_{\mu_{\bX}}(f_{\b}).
\end{align} 
Hence, the optimization problem \eqref{eq:ell1}, which we repeat here for convenience,
\begin{align}
\min_{\b} \,  \big\| \btX^\transpose \b \big\|_1 \, \, \, \text{s.t.} \, \, \, \b^\transpose \b = 1 \label{eq:ell1_repeated},
\end{align} is equivalent to the problem 
\begin{align}
\min_{\b} \, \left[M \,  \mathbb{E}_{\mu_{\bO}}(f_{\b}) + N \, \mathbb{E}_{\mu_{\bX}}(f_{\b}) \right]
 \, \, \, \text{s.t.} \, \, \, \b^\transpose \b = 1. \label{eq:ell1_DiscreteMeasure}
\end{align} Similarly, the recursion \eqref{eq:ConvexRelaxations}, repeated here for convenience,
\begin{align}
\bn_{k+1} = \argmin_{\b} \,  \big \| \btX^\transpose \b \big\|_1 \, \, \, \text{s.t.} \, \, \, \b^\transpose \hat{\bn}_k = 1, \label{eq:ConvexRelaxations_repeated}
\end{align} 
is equivalent to the recursion
\begin{align}
\bn_{k+1} = \argmin_{\b} \,  \left [ M \, \mathbb{E}_{\mu_{\bO}}(f_{\b}) + N \, \mathbb{E}_{\mu_{\bX}}(f_{\b}) \right]
 \, \, \, \text{s.t.} \, \, \, \b^\transpose \hat{\bn}_k = 1. \label{eq:ConvexRelaxations_DiscreteMeasure}
\end{align}
Now, the discrete measures $\mu_{\bO}, \mu_{\bX}$ of \eqref{eq:DiscreteMeasures},
are discretizations of the continuous measures $\mu_{\Sp^{D-1}}$, and $\mu_{\Sp^{D-1} \cap \cS}$ respectively, where the latter is the uniform measure on $\Sp^{D-1} \cap \cS$.
Hence, for the purpose of 
understanding the properties of the global minimizer of \eqref{eq:ell1_DiscreteMeasure} and the limiting point of \eqref{eq:ConvexRelaxations_DiscreteMeasure}, it is meaningful to 
replace in \eqref{eq:ell1_DiscreteMeasure} and \eqref{eq:ConvexRelaxations_DiscreteMeasure} the discrete measures $\mu_{\bO}$ and $\mu_{\bX}$ by their continuous counterparts $\mu_{\Sp^{D-1}}$ and $\mu_{\Sp^{D-1} \cap \cS}$, and study the resulting \emph{continuous} problems
\begin{align}
& \min_{\b} \, \left[M \, \mathbb{E}_{\mu_{\Sp^{D-1}}}(f_{\b}) + N \, \mathbb{E}_{\mu_{\Sp^{D-1} \cap \cS}}(f_{\b}) \right] \, \, \, \text{s.t.} \, \, \, \b^\transpose \b = 1, \label{eq:ell1_ContinuousMeasure} \\
& \bn_{k+1} = \argmin_{\b} \, \left[M \, \mathbb{E}_{\mu_{\Sp^{D-1}}}(f_{\b}) + N \, \mathbb{E}_{\mu_{\Sp^{D-1} \cap \cS}}(f_{\b}) \right]
 \, \, \, \text{s.t.} \, \, \, \b^\transpose \hat{\bn}_k = 1. \label{eq:ConvexRelaxations_ContinuousMeasure}
\end{align} 
It is important to note that if these two continuous problems have the geometric properties of interest, i.e., if every global solution of \eqref{eq:ell1_ContinuousMeasure} is a vector orthogonal to the inlier subspace, and similarly, if the sequence of vectors $\left\{ \bn_k \right\}$ produced by \eqref{eq:ConvexRelaxations_ContinuousMeasure} converges to a vector $\bn_{k^*}$ orthogonal to the inlier subspace, then this \emph{correctness} of the 
continuous problems can be viewed as a first theoretical verification of the correctness of the \emph{discrete} formulations \eqref{eq:ell1} and \eqref{eq:ConvexRelaxations}. The objective of the rest of this section is to establish that this is precisely the case. 

Before discussing our main two results in this direction, we note that 
the continuous objective function appearing in \eqref{eq:ell1_ContinuousMeasure} and \eqref{eq:ConvexRelaxations_ContinuousMeasure} can be re-written in a more suggestive form.
To see what that is, define $c_D$ as the average height of the unit hemisphere of $\Re^D$, directly computed as
\begin{align}
c_D &:= \int_{\z \in \Sp^{D-1}} |z_1| d\mu_{\Sp^{D-1}}= \frac{(D-2)!!}{(D-1)!!} \cdot 
\left\{
\begin{array}{ll}
\frac{2}{\pi} &  \text{if} \, \, \,    D \, \, \,  \mbox{even}, \\
1 & \text{if} \, \, \,  D \, \, \,  \text{odd},
\end{array}
\right. \label{eq:cD}
\end{align} where $z_1$ is the first coordinate of the vector $\z$, and the double factorial is defined as
\begin{align}
k!! :=\left\{
\begin{array}{ll}
k(k-2)(k-4)\cdots 4 \cdot 2 &  \text{if} \, \, \,    k \, \, \,  \mbox{even}, \\
k(k-2)(k-4)\cdots 3 \cdot 1 & \text{if} \, \, \,  k \, \, \,  \text{odd}.
\end{array}
\right.
\end{align} 
Then we have the following result, whose proof can be found in \S \ref{subsection:ProofContinuousForm}.

\begin{prp} \label{prp:ContinuousForm} 
The objective function of the continuous problem \eqref{eq:ell1_ContinuousMeasure} can be rewritten as:
\begin{align}
M \, \mathbb{E}_{\mu_{\Sp^{D-1}}}(f_{\b}) + N \, \mathbb{E}_{\mu_{\Sp^{D-1} \cap \cS}}(f_{\b}) = \big \|\b\big \|_2 \left(M c_D + N c_d \cos(\phi) \right), \label{eq:ContinuousObjective_DPCP}
\end{align} 
where $\phi$ is the principal angle between $\b$ and the subspace $\cS$.
\end{prp}

As a consequence of this result, when $\b \in \Sp^{D-1}$, the first term (the outlier term) of the objective function becomes a constant ($Mc_D$) and hence the outliers do not affect the optimal solution of \eqref{eq:ell1_ContinuousMeasure}. Moreover, the second term (the inlier term) of the objective function depends only on the cosine of the principal angle between $\b$ and the subspace, which is minimized when $\b$ is orthogonal to the subspace ($\phi = \pi/2$). This leads to the following result about the continuous problem, whose proof can be found in \S \ref{subsection:Proof-thm-ContinuousNonConvex}.
 

\begin{thm} \label{thm:ContinuousNonConvex}
Any global solution to problem \eqref{eq:ell1_ContinuousMeasure}
must be orthogonal to $\cS$.
\end{thm} 
Observe that this result is true irrespective of the weight $M$ of the outlier term or the weight $N$ of the inlier term in the continuous objective function \eqref{eq:ContinuousObjective_DPCP}. 
Similarly, the next result, whose proof can be
found in \S \ref{subsection:Proof-thm-ConvexRelaxationsContinuous}, shows that the solutions to the continuous recursion in \eqref{eq:ConvexRelaxations_ContinuousMeasure} converge to a vector orthogonal to the inlier subspace in a finite number of steps, regardless of the outlier and inlier weights $M,N$.

\begin{thm} \label{thm:ConvexRelaxationsContinuous}
	Consider the sequence $\left\{\bn_k\right\}_{k \ge 0}$
	generated by recursion \eqref{eq:ConvexRelaxations_ContinuousMeasure}, with
	$\hat{\bn}_0 \in \Sp^{D-1}$. Let $\phi_0$
	be the principal angle of $\bn_0$ from $\cS$, and define
$\alpha:=Nc_d / M c_D$. Then, as long as $\phi_0>0$, the sequence $\left\{\bn_k\right\}_{k \ge 0}$ converges to a unit $\ell_2$-norm element of $\cS^\perp$ in a finite number $k^*$ of iterations, where $k^*=0$ if $\phi_0 = \pi/2$,  $k^* =1$ if  $\tan(\phi_0) \ge 1/\alpha$, and $k^* \le \ceil[\bigg]{\frac{\tan^{-1}(1/\alpha)-\phi_0}{\sin^{-1}(\alpha \sin(\phi_0))}} +1$ otherwise. 
\end{thm} Notice the remarkable fact that according to Theorem \ref{thm:ConvexRelaxationsContinuous}, the continuous recursion \eqref{eq:ConvexRelaxations_ContinuousMeasure} converges to a vector orthogonal
to the inlier subspace $\cS$ in a \emph{finite} number of steps. Moreover, if the relation
\begin{align}
\tan(\phi_0) \ge 1/\alpha = \frac{M}{N} \frac{c_D}{c_d}, 
\label{eq:AlphaPhi}
\end{align} 
holds true, then this convergence occurs in a single step. One way to interpret \eqref{eq:AlphaPhi} is to notice that as long as the angle $\phi_0$ of the initial estimate $\hat{\bn}_0$ from the inlier subspace is positive, and for any arbitrary but fixed number of outliers $M$, there is always a sufficiently large number $N$ of inliers, such that \eqref{eq:AlphaPhi} is satisfied and thus convergence occurs in one step. Likewise, condition \eqref{eq:AlphaPhi} can also be satisfied if $d/D$ is sufficiently small (so that $c_D/c_d$ is small). Conversely, for any fixed number of inliers $N$ and outliers $M$, there is always a sufficiently large angle $\phi_0$ such that \eqref{eq:AlphaPhi} is true, and thus \eqref{eq:ConvexRelaxations_ContinuousMeasure} again converges in a single step. 
More generally, even when \eqref{eq:AlphaPhi} is not true, the larger $\phi_0, N$ are, the smaller the quantity
\begin{align}
\ceil[\bigg]{\frac{\tan^{-1}(1/\alpha)-\phi_0}{\sin^{-1}(\alpha \sin(\phi_0))}}
\end{align} is, and thus according to Theorem \ref{thm:ContinuousNonConvex} the faster 
\eqref{eq:ConvexRelaxations_ContinuousMeasure} converges. 
	
\subsection{Discrepancy bounds between the continuous and discrete problems} \label{subsection:DiscrepancyBounds}

The heart of our analysis framework is to bound the deviation of some underlying geometric quantities, which we call the \emph{average outlier} and the \emph{average inlier with respect to} $\b$, from their continuous counterparts. To begin with,
recall our discrete objective function
\begin{align}
\J_{\text{discrete}}(\b) = \big \| \btX^\transpose \b \big \|_1 = \big \| \bO^\transpose \b \big \|_1 + \big \| \bX^\transpose \b \big \|_1 \label{eq:SingleDPCP_Jdiscrete}
\end{align} 
and its continuous counterpart
\begin{align}
\J_{\text{continuous}}(\b)=\big \|\b\big \|_2 \left(M c_D + N c_d \cos(\phi) \right).\label{eq:SingleDPCP_Jcontinuous}
\end{align} 
Now, notice that the term of the discrete objective that depends on the outliers $\bO$ can be written as
\begin{align}
\big \| \bO^\transpose \b \big \|_1 = \sum_{j=1}^M | \o_j^\transpose \b | = \sum_{j=1}^M \b^\transpose \Sign(\o_j^\transpose \b) \o_j = M \, \b^\transpose \,  \o_{\b},
\end{align} 
where $\Sign(\cdot)$ is the sign function and  $\o_b$ is the \emph{average outlier with respect to $\b$}, defined as
\begin{align}
\o_{\b}:= \frac{1}{M} \sum_{j=1}^M \Sign(\b^\transpose \o_j)\o_j . 
\label{eq:AverageOutlier}
\end{align} 
Defining a vector valued function $\f_{\b}: \Sp^{D-1} \rightarrow \Re^D$ by 
$\z \in \Sp^{D-1} \stackrel{\f_{\b}}{\longmapsto} \Sign(\b^\transpose \z) \z$, we notice that
\begin{align}
\o_{\b} = \frac{1}{M} \sum_{j=1}^M \f_{\b}(\o_j) 
= \frac{1}{M} \sum_{j=1}^M \int_{\z \in \Sp^{D-1}} f_{\b}(\z) \delta ( \z - \o_j) d\mu_{S^{D-1}} 
= \int_{\z \in \Sp^{D-1}} f_{\b}(\z) d\mu_{\bO}(\z)
,
\end{align} 
where $\mu_{\bO}(\z)$ is defined in \eqref{eq:DiscreteMeasures}, and so $\o_{\b}$ is a discrete approximation to the continuous integral $\int_{\z \in \Sp^{D-1}} \f_{\b}(\z) d \mu_{\Sp^{D-1}}$, whose value is given by the next Lemma (see \S \ref{subsection:Proof-lem-VectorIntegral} for the proof).

\begin{lem} \label{lem:VectorIntegral}
Recall the definition of $c_D$ in \eqref{eq:cD}. For any $\b \in \Sp^{D-1}$ we have
 \begin{align}
 \int_{\z \in \Sp^{D-1}} \f_{\b}(\z) d \mu_{\Sp^{D-1}}=\int_{\z \in \Sp^{D-1}} \Sign(\b^\transpose \z)\z d \mu_{\Sp^{D-1}} = c_D \, \b.
 \end{align} 
\end{lem} 	
In other words, the continuous average outlier with respect to $\b$ is $c_D \b$. We define $\epsilon_{\bO,M}$ to be the maximum error between the discrete and continuous average outliers as $\b$ varies on $\Sp^{D-1}$, i.e.,
\begin{align} 
\eO := \max_{\b \in \Sp^{D-1}} \big \|c_D \, \b - \o_{\b}\big \|_2, \label{eq:epsilonO}
\end{align} and we establish that the more \emph{uniformly distributed} $\bO=[\o_1,\dots,\o_M] \subset \Sp^{D-1}$ is the smaller $\eO$ becomes. The notion of uniformity of $\bO$ that we use here is a deterministic one and is captured by the \emph{spherical cap discrepancy} of the 
set $\bO$, defined as \citep{Grabner:MR97,Grabner:MathComp93}  
\begin{align}
\mathfrak{S}_{D,M}(\bO) := \sup_{\mathcal{C}} \Big |\frac{1}{M} \sum_{j=1}^M \mathbb{I}_{\mathcal{C}}(\o_j) - \mu_{\Sp^{D-1}}(\mathcal{C}) \Big|. \label{eq:SphericalCapDiscrepancy}
\end{align} 
In \eqref{eq:SphericalCapDiscrepancy} the supremum is taken over 
all spherical caps $\mathcal{C}$ of the sphere $\Sp^{D-1}$, where a spherical cap
is the intersection of $\Sp^{D-1}$ with a half-space of $\Re^D$, and $\mathbb{I}_{\mathcal{C}}(\cdot)$ is the indicator function of $\mathcal{C}$, which takes the value $1$ inside $\mathcal{C}$ and zero otherwise. The spherical cap discrepancy $\mathfrak{S}_{D,M}(\bO)$ is precisely the supremum among all errors in approximating integrals of indicator functions of spherical caps via averages of such indicator functions on the point set $\bO$. Intuitively, $\mathfrak{S}_{D,M}(\bO)$ captures how close the discrete measure $\mu_{\bO}$ (see equation \eqref{eq:DiscreteMeasures}) associated with $\bO$ is to the measure $\mu_{\Sp^{D-1}}$. We will say that $\bO$ is uniformly distributed on $\Sp^{D-1}$ if $\mathfrak{S}_{D,M}(\bO)$ is small. We note here that as a function of the number of points $M$, $\mathfrak{S}_{D,M}(\bO)$ decreases with a rate of \citep{Dick-AANT14,Beck-Mathematika84}
\begin{align}
\sqrt{\log(M)} M^{-\frac{1}{2}-\frac{1}{2(D-1)}}. \label{eq:DiscrepancyRate}
\end{align} 
As a consequence, to show that uniformly distributed points $\bO$ correspond to small $\eO$, it suffices to bound the maximum integration error $\eO$ from above by a quantity proportional to the spherical cap discrepancy $\mathfrak{S}_{D,M}(\bO)$. Inequalities that bound from above the approximation error of the integral of a function in terms of the variation of the function and 
the discrepancy of a finite set of points (not necessarily the spherical cap discrepancy; there are several types of discrepancies) are widely known as \emph{Koksma-Hlawka inequalites} \citep{Kuipers:UDS12,Hlawka:SPM71}. Even though such inequalities exist and are well-known for integration of functions on the unit hypercube $[0,1]^D$ \citep{Kuipers:UDS12,Hlawka:SPM71,Harman:UDT10}, similar inequalities for integration of functions on the unit sphere $\Sp^{D-1}$ seem not to be known in general \citep{Grabner:MathComp93}, except if one makes additional assumptions on the 
distribution of the finite set of points \citep{Grabner:MR97,Brauchart:JC15}. 
Nevertheless, the function $f_{\b}: \z \longmapsto |\b^\transpose \z |$ that is associated with $\eO$ is simple enough to allow for a Koksma-Hlawka inequality of its own, as described in the next lemma, whose proof can be found in \S \ref{subsection:Proof-lem-Koksma-eO}.\footnote{The authors are grateful to Prof. Glyn Harman for pointing out that such a result is possible as well as suggesting how to prove it.}

\begin{lem} \label{lem:Koksma-eO}
Let $\bO = [\o_1,\dots,\o_M]$ be a finite subset of $\Sp^{D-1}$. Then 
\begin{align}
\eO = \max_{\b \in \Sp^{D-1}} \big \| c_D \b - \o_{\b} \big \|_2 \le \sqrt{5} \mathfrak{S}_{D,M}(\bO),
\end{align} where $c_D, \o_{\b}$ and $\mathfrak{S}_{D,M}(\bO)$ are defined in 
\eqref{eq:cD}, \eqref{eq:AverageOutlier} and \eqref{eq:SphericalCapDiscrepancy} respectively. 
\end{lem}

We now turn our attention to the inlier term $\big \|\btX^\transpose \b \big\|_1$ of the discrete objective function \eqref{eq:SingleDPCP_Jdiscrete}, which is slightly more complicated than the outlier term. We have
\begin{align}
\big \| \bX^\transpose \b \big\|_1 = \sum_{j=1}^N \big | \x_j^\transpose \b \big | = \sum_{j=1}^N \b^\transpose \Sign(\x_j^\transpose \b) \x_j = N \, \b^\transpose \,  \x_{\b},
\end{align} 
where
\begin{align}
\x_{\b}:= \frac{1}{N} \sum_{j=1}^N \Sign(\b^\transpose \x_j)\x_j=\frac{1}{N} \sum_{j=1}^N \f_{\b}(\x_j) 
= \int_{\x \in \Sp^{D-1} \cap \cS}
f_{\b}(\x) d\mu_{\bX}(\x)
\label{eq:AverageInlier}
\end{align} 
is the \emph{average inlier with respect to $\b$}. Thus, $\x_{\b}$ is a discrete approximation of the integral 
\begin{align}
\int_{\x \in \Sp^{D-1} \cap \cS} \f_{\b}(\x) d \mu_{\Sp^{D-1}},
\end{align} 
whose value is given by the next lemma (see \S \ref{subsection:Proof-lem-VectorIntegralInliers} for the proof).

\begin{lem} \label{lem:VectorIntegralInliers}
 For any $\b \in \Sp^{D-1}$ we have	
 \begin{align}
 \int_{\x \in \Sp^{D-1} \cap \cS} \f_{\b}(\x) d \mu_{\Sp^{D-1}}=\int_{\x \in \Sp^{D-1} \cap \cS} \Sign(\b^\transpose \x)\x d \mu_{\Sp^{D-1}} = c_d \, \hat{\v},
 \end{align} where $c_d$ is given by \eqref{eq:cD} after replacing $D$ with $d$, and 
 $\v$ is the orthogonal projection of $\b$ onto $\cS$.  
\end{lem} 

In other words, the continuous average inlier with respect to $\b$ is $c_d \hat{\v}$. We define $\epsilon_{\bX}$ to be the maximum error between the discrete and continuous average inliers as $\b$ varies on $\Sp^{D-1}$, which is the same as the maximum error as $\b$ varies on $\Sp^{D-1} \cap \cS$, i.e.,
\begin{align} 
\eX :=  \max_{\b \in \Sp^{D-1}} \big \|c_d \, \widehat{\pi_{\cS}(\b)} - \x_{\b}\big \|_2= \max_{\b \in \Sp^{D-1} \cap \cS} \big \|c_d \, \b - \x_{\b}\big \|_2. \label{eq:epsilonX}
\end{align} Then an almost identical argument as the one that established Lemma \ref{lem:Koksma-eO} gives that
\begin{align}
\eX \le \sqrt{5} \mathfrak{S}_{d,N}(\bX), \label{eq:Koksma-eX}
\end{align} where now the discrepancy $\mathfrak{S}_{d,N}(\bX)$ of the inliers $\bX$ is defined exactly as in \eqref{eq:SphericalCapDiscrepancy} except that $M$ is replaced by $N$ and
the supremum 
is taken over all spherical caps of $\Sp^{D-1} \cap \cS \cong \Sp^{d-1}$.

\subsection{Conditions for global optimality and convergence of the discrete problems} \label{subsection:ConditionsDiscrete}
In this section we analyze the discrete problem \eqref{eq:ell1} and the associated
discrete recursion \eqref{eq:ConvexRelaxations}, where the adjective \emph{discrete} 
refers to the fact that \eqref{eq:ell1} and \eqref{eq:ConvexRelaxations} depend on a finite 
set of points $\btX = [\bX \, \bO] \boldsymbol{\Gamma}$ sampled from the union of the space of outliers $\Sp^{D-1}$ and the space of inliers $\Sp^{D-1} \cap \cS$. In \S \ref{subsection:Continuous} we showed that these two problems are discrete versions of the continuous problems \eqref{eq:ell1_ContinuousMeasure} and \eqref{eq:ConvexRelaxations_ContinuousMeasure}, respectively. We further showed that the continuous problems possess the geometric property of interest, i.e., every global minimizer of \eqref{eq:ell1_ContinuousMeasure} must be an element of $\cS^\perp \cap \Sp^{D-1}$ (Theorem \ref{thm:ContinuousNonConvex}) and the recursion \eqref{eq:ConvexRelaxations_ContinuousMeasure} produces a sequence of vectors that converges in a finite number of steps to an element of $\cS^\perp \cap \Sp^{D-1}$ (Theorem \ref{thm:ConvexRelaxationsContinuous}). In this section we use the discrepancy bounds of \S \ref{subsection:DiscrepancyBounds} to show that under some conditions on the uniformity of $\bX=[\x_1,\dots,\x_N]$ and $\bO=[\o_1,\dots,\o_M]$, a similar statement holds for problems \eqref{eq:ell1} and \eqref{eq:ConvexRelaxations}. We start with a definition. 
\begin{dfn} \label{dfn:DPCPsingle_circumradius}
Given a set $\boldsymbol{\mathcal{Y}} = [\y_1,\dots,\y_L] \subset \Sp^{D-1}$ and an integer $K \le L$,
define $\R_{\boldsymbol{\mathcal{Y}},K}$ to be the maximum circumradius among all polytopes
of the form 
\begin{align}
\left\{\sum_{i=1}^K \alpha_{j_i} \y_{j_i}: \, \alpha_{j_i} \in [-1,1] \right\},
\end{align} where $j_1,\dots,j_K$ are distinct integers in $[L]$, and the circumradius of a bounded subset of $\Re^D$ is the infimum over 
the radii of all Euclidean balls of $\Re^D$ that contain that subset. With that,  define
\begin{align}
\R_{\bO,\bX} := \max_{K_1+K_2 <D, \, K_2<d} (\R_{\bO,K_1}+\R_{\bX,K_2}).
\end{align}
\end{dfn}

The next theorem, proved in \S \ref{subsection:Proof-thm-DiscreteNonConvex}, states that if both inliers and outliers are sufficiently uniformly distributed, i.e., if the uniformity parameters $\eX$ and $\eO$ are sufficiently small, then every global solution of \eqref{eq:ell1} must be orthogonal to the inlier subspace $\cS$. More precisely,
\begin{thm} \label{thm:DiscreteNonConvex} Suppose that the ratio $\gamma$ of outliers to inliers satisfies
\begin{align}
\gamma &:=\frac{M}{N} <\frac{1}{2\eO} \min \left\{c_d-\eX,2\left(c_d-\eX -\frac{\R_{\bO,\bX}}{N}\right),\Gamma\right\} \label{eq:gammaUP}, \, \, \, \text{where} \\
\Gamma &:= \frac{c_d-\eX}{2(3c_d-\eX)} \left[\sqrt{\mathscr{P}^2+8(3c_d-\eX)(c_d-\eX -\frac{\R_{\bO,\bX}}{N})}-\mathscr{P}\right], \\
\mathscr{P}&:= 2\frac{\R_{\bO,\bX}}{N} +\eX+c_d. 
\end{align} Then any global solution $\b^*$ to \eqref{eq:ell1} must be orthogonal to $\Span(\bX)$.
\end{thm}

Towards interpreting Theorem \ref{thm:DiscreteNonConvex}, consider first the asymptotic case where we allow $N$ and $M$ to go to infinity, while keeping the ratio $\gamma$ constant. Assuming that both inliers and outliers are perfectly well distributed in the limit, i.e., under the hypothesis that 
$ \lim_{N \rightarrow \infty}\mathfrak{S}_{d,N}(\bX)=0$ and $ \lim_{M \rightarrow \infty}\mathfrak{S}_{D,M}(\bO)=0$, Lemma \ref{lem:Koksma-eO} and inequality \eqref{eq:Koksma-eX} give that $ \lim_{N \rightarrow \infty}\eX=0$ and $ \lim_{M \rightarrow \infty}\eO=0$, in which case \eqref{eq:gammaUP} is satisfied. This suggests the interesting fact that \eqref{eq:ell1} can possibly give a normal to the inliers even for arbitrarily many outliers, and irrespectively of the subspace dimension $d$. Along the same lines, for a given $\gamma$ and under the point set uniformity hypothesis, we can always increase the number of inliers and outliers (thus decreasing $\eX$ and $\eO$), while keeping $\gamma$ constant, until \eqref{eq:gammaUP} is satisfied, once again indicating that \eqref{eq:ell1} can possibly yield a normal to the space of inliers irrespectively of their intrinsic dimension; this becomes evident in the numerical evaluation of Figs. \ref{figure:DPCP_theory_N500}-\ref{figure:DPCP_theory_N7000}. Notice that the intrinsic dimension $d$ of the inliers manifests itself through the quantity $c_d$, which we recall is a decreasing function of $d$. Consequently, the smaller $d$ is the larger the RHS of \eqref{eq:gammaUP} becomes, and so the easier it is to satisfy \eqref{eq:gammaUP}. 

More explicitly (and less formally), because of \eqref{eq:DiscrepancyRate} the quantities $\eO, \eX$ decay at an approximate rate of $1/\sqrt{M}, 1/\sqrt{N}$ respectively. In turn, this shows that the conditions \eqref{eq:gammaUP} are satisfied if roughly $M < \mathcal{O}(N^2)$.
To see this, note, e.g., that the first inequality in \eqref{eq:gammaUP} reads 
\begin{align}
\frac{M}{N} < \frac{c_d-\eX}{2 \eO},
\end{align} which roughly says that 
\begin{align}
\text{constant} \cdot \sqrt{M} \le \text{constant} \cdot N - \text{constant} \cdot \sqrt{N},
\end{align} where by \emph{constant} here we mean independent of $M,N$. A similar conclusion can be drawn from the rest inequalities in \eqref{eq:gammaUP} \footnote{It is the subject of ongoing research to arrive at this conclusion by more formal means.}. In contrast, the analysis of the \emph{haystack model} of REAPER \citep{Lerman:FCM15} gives $M < \mathcal{O}(N)$.

A similar phenomenon holds for the recursion of convex relaxations \eqref{eq:ConvexRelaxations}. 
Notice that according to Theorem \ref{thm:ContinuousNonConvex}, the continuous recursion converges in a finite number of iterations to a vector that is orthogonal to $\Span(\bX)=\cS$, as long as the initialization $\hat{\bn}_0$ does not lie in $\cS$ (equivalently $\phi_0 >0$). Intuitively, one should expect that in the discrete case, the conditions for the \emph{discrete} recursion \eqref{eq:ConvexRelaxations} to be successful, should be at least as strong as the conditions of Theorem \ref{thm:DiscreteNonConvex}, and strictly stronger than the condition $\phi_0 >0$ of Theorem \ref{thm:ConvexRelaxationsContinuous}. Our next result, whose proof can be found in \S \ref{subsection:Proof-thm-DiscreteConvexRelaxations}, formalizes this intuition.

\begin{thm}	\label{thm:DiscreteConvexRelaxations}
	Suppose that condition \eqref{eq:gammaUP} holds true and consider the sequence $\left\{\bn_k\right\}_{k \ge 0}$ generated by the recursion \eqref{eq:ConvexRelaxations}. Let $\phi_0$ be the principal angle of $\hat{\bn}_0$ from $\Span(\bX)$ and suppose that
\begin{align}
\cos(\phi_0) &< \frac{c_d-\eX}{2c_d(c_d+\eX)}\left[ -\Q+\sqrt{\Q^2+4 c_d(c_d-\Q)}\right]-\frac{2 \eO}{c_d+\eX} \frac{M}{N} \label{eq:phiLB}, \\
\Q &:= \frac{\R_{\bO,\bX}}{N}+\eO\frac{M}{N} +  \eX.
\end{align} Then after a finite number of iterations the sequence   $\left\{\bn_k\right\}_{k \ge 0}$ converges to a unit $\ell_2$-norm vector that is orthogonal to $\Span(\bX)$.  
\end{thm}

First note that if \eqref{eq:gammaUP} is true, then the expression of \eqref{eq:phiLB} always defines an angle between $0$ and $\pi/2$.
Moreover, Theorem \ref{thm:DiscreteConvexRelaxations} can be interpreted using the same asymptotic arguments as Theorem 
\ref{thm:DiscreteNonConvex}; notice in particular that the lower bound on the
angle $\phi_0$ tends to zero as $M,N$ go to infinity with $\gamma$ constant, i.e., the more uniformly distributed inliers and outliers are, the closer $\bn_0$ is allowed to be to $\Span(\bX)$. We also emphasize that Theorem \ref{thm:DiscreteConvexRelaxations} asserts the correctness of the linear programming recursions \eqref{eq:ConvexRelaxations} as far as recovering a vector $\bn_{k^*}$ orthogonal to $\cS:=\Span(\bX)$ is concerned. Even though this was our initial motivation for posing problem \eqref{eq:ell1}, Theorem \ref{thm:DiscreteConvexRelaxations} does not assert in general that $\bn_{k^*}$ is a global minimizer of problem \eqref{eq:ell1}. However, this is indeed the case, when the inlier subspace $\cS$ is a hyperplane, i.e., $d=D-1$. This is because, up to a sign, there is a unique vector $\b \in \Sp^{D-1}$ that is orthogonal to $\cS$ (the normal vector to the hyperplane), which, under conditions \eqref{eq:gammaUP} and \eqref{eq:phiLB}, is the unique global minimizer of \eqref{eq:ell1}, as well as the limit point $\bn_{k^*}$ of Theorem \ref{thm:DiscreteConvexRelaxations}.

\section{Proofs} \label{section:Proofs}
In this section we provide the proofs of all claims stated 
in earlier sections. 

\subsection{Proof of Proposition \ref{prp:MaximalFriendlyPlanes}} \label{subsection:ProofMaximalFriendlyPlanes}
By the general position hypothesis on $\bX$ and $\bO$, any hyperplane that does not contain $\bX$ can contain at most $D-1$ points from $\btX$. We will show that there exists a hyperplane that contains more than $D-1$ points of $\btX$. Indeed, take $d$ inliers and $D-d-1$ outliers and let $\H$ be the hyperplane generated by these $D-1$ points. Denote the normal vector to that hyperplane by $\b$. Since $\H$ contains $d$ inliers, $\b$ will be orthogonal to these inliers. Since $\bX$ is in general position, every $d$-tuple of inliers is a basis for $\Span(\bX)$. As a consequence, $\b$ will be orthogonal to $\Span(\bX)$, and in particular $\b \perp \bX$. This implies that $\bX \subset \H$ and so $\H$ will contain $N + D-d-1 \ge d+1 + D-d-1  > D-1$ points of $\btX$. 

\subsection{Proof of Proposition \ref{prp:ContinuousForm}} \label{subsection:ProofContinuousForm}

Writing $\b = \big \| \b \big \|_2 \hat{\b}$, and letting $\boldsymbol{R}$ be a rotation that takes $\hat{\b}$ to the first standard basis vector $\e_1$, we see that the first expectation in the LHS of \eqref{eq:ContinuousObjective_DPCP} becomes equal to
\begin{align}
\mathbb{E}_{\mu_{\Sp^{D-1}}}(f_{\b})  
&= \int_{\z \in \Sp^{D-1}} f_{\b}(\z) d\mu_{\Sp^{D-1}}  
=\int_{\z \in \Sp^{D-1}} \left| \b^\transpose \z \right| d \mu_{\Sp^{D-1}} \\
&= \big \|\b \big \|_2 
\int_{\z \in \Sp^{D-1}} \left| \hat{\b}^\transpose \z \right| d \mu_{\Sp^{D-1}} 
= \big \|\b \big \|_2  \int_{\z \in \Sp^{D-1}} \left|\z^\transpose \boldsymbol{R}^{-1} \boldsymbol{R} \hat{\b}\right| d  \mu_{\Sp^{D-1}} \\
 &=\big \|\b \big \|_2  \int_{\z \in \Sp^{D-1}} |\z^\transpose  \e_1| d \mu_{\Sp^{D-1}}  
 =\big \|\b \big \|_2  \int_{\z \in \Sp^{D-1}} |z_1| d \mu_{\Sp^{D-1}} = \big \|\b \big \|_2  c_D,
\end{align} where $\z=(z_1,\dots,z_D)^\transpose$ is the coordinate representation of $\z$.
 To see what the second expectation in the LHS of
\eqref{eq:ContinuousObjective_DPCP} evaluates to, decompose $\b$ as $\b = \pi_{\cS}(\b) + \pi_{\cS^\perp}(\b)$, and note that because the support of the measure $\mu_{\Sp^{D-1}\cap \cS}$ is contained in $\cS$, we must have that
\begin{align}
\mathbb{E}_{\mu_{\Sp^{D-1} \cap \cS}}(f_{\b}) 
&= \int_{\z \in \Sp^{D-1}} \left| \b^\transpose \z  \right| d \mu_{\Sp^{D-1} \cap \cS}
= \int_{\z \in \Sp^{D-1} \cap \cS} \left| \b^\transpose \z \right| d \mu_{\Sp^{D-1} \cap \cS} \\
&= \int_{\z \in \Sp^{D-1} \cap \cS} \left| \left(\pi_{\cS}(\b)\right)^\transpose \z \right| d \mu_{\Sp^{D-1} \cap \cS} \\
&= \big \| \pi_{\cS}(\b)\big \|_2 \int_{\z \in \Sp^{D-1} \cap \cS} \left| \left(\widehat{\pi_{\cS}(\b)}\right)^\transpose \z \right| d \mu_{\Sp^{D-1} \cap \cS}. 
\end{align} 
Writing $\z'$ and $\b'$ for the coordinate representation of $\z$ and $\widehat{\pi_{\cS}(\b)}$ with respect to a basis of $\cS$, and noting that $\mu_{\Sp^{D-1} \cap \cS} \cong \mu_{\Sp^{d-1}}$, we have that
\begin{align}
\int_{\z \in \Sp^{D-1} \cap \cS} \left| \left(\widehat{\pi_{\cS}(\b)}\right)^\transpose \z \right| d \mu_{\Sp^{D-1} \cap \cS} = \int_{\z' \in \Sp^{d-1}} \left| \z'^\transpose \b' \right| d \mu_{\Sp^{d-1}} = c_{d},
\end{align} where now $c_d$ is the average height of the unit hemisphere of $\Re^d$. Finally, noting that 
\begin{align}
\big \| \pi_{\cS}(\b)\big \|_2 = \big \| \b \big \|_2 \cos(\phi),
\end{align} where $\phi$ is the principal angle of $\b$ from the subspace $\cS$, we have that 
\begin{align}
\mathbb{E}_{\mu_{\Sp^{D-1} \cap \cS}}(f_{\b})  = \big \| \b \big \|_2 c_d \cos(\phi). 
\end{align}

\subsection{Proof of Theorem \ref{thm:ContinuousNonConvex}} \label{subsection:Proof-thm-ContinuousNonConvex}
Because of the constraint $\b^\transpose \b = 1$ in \eqref{eq:ell1_ContinuousMeasure}, and using  \eqref{eq:ContinuousObjective_DPCP}, problem \eqref{eq:ell1_ContinuousMeasure}
can be written as 
\begin{align}
\min_{\b} \, \, \, \left[M c_D + N c_d \cos(\phi) \right] \, \, \, \text{s.t.} \, \, \, \b^\transpose \b=1.
\end{align} It is then immediate that the global minimum is equal to $M c_D$ and it is attained if and only if $\phi = \pi/2$, which corresponds to $\b \perp \cS$.

\subsection{Proof of Theorem \ref{thm:ConvexRelaxationsContinuous}} \label{subsection:Proof-thm-ConvexRelaxationsContinuous}

At iteration $k$ the optimization problem associated with \eqref{eq:ConvexRelaxations_ContinuousMeasure} is 
\begin{align}
\min_{\b \in \Re^D} \, \, \, \J(\b)=\big \| \b \big \|_2 \left(M c_D + Nc_d \cos(\phi) \right) \, \, \, \text{s.t.} \, \, \, \b^\transpose \hat{\bn}_k = 1, \label{eq:ProofContinuousRecursion_b_phi}
\end{align} where $\phi$ is the principal angle of $\b$ from the subspace $\cS$. 

Let $\phi_k$ be the principal angle of $\hat{\bn}_{k}$ from $\cS$, and let $\bn_{k+1}$ be a global minimizer of \eqref{eq:ProofContinuousRecursion_b_phi}, with principal angle from $\cS$ equal to $\phi_{k+1}$. We show that $\phi_{k+1} \ge \phi_k$. To see this, note that the decrease in the objective function at iteration $k$ is	
\begin{align}
		\J(\hat{\bn}_{k})-\J(\bn_{k+1}) :=& M\, c_D \, \big \|\hat{\bn}_{k}\big \|_2 + N\, c_d \, \big \|\hat{\bn}_{k}\big \|_2 \, \cos(\phi_k) \nonumber \\
		&- M\, c_D \, \big \|\bn_{k+1}\big \|_2 - N\, c_d \, \big \|\bn_{k+1}\big \|_2 \, \cos(\phi_{k+1}).	\label{eq:DPCP-cont-Daniel}			
	\end{align} Since $\bn_{k+1}^\transpose \hat{\bn}_k=1$, we must have that 
	$\big \|\bn_{k+1}\big \|_2 \ge 1=\big \|\hat{\bn}_k \big \|_2$. Now if $\phi_{k+1} < \phi_k$, then $\cos(\phi_{k+1}) > \cos(\phi_k)$. But then \eqref{eq:DPCP-cont-Daniel} implies that $\J(\bn_{k+1})> \J(\hat{\bn}_{k})$, which
	is a contradiction on the optimality of $\bn_{k+1}$. Hence it must be the case that $\phi_{k+1} \ge \phi_k$, and so the sequence $\left\{\phi_k \right\}_k$ is non-decreasing.
In particular, since $\phi_0>0$ by hypothesis, we must also have $\phi_k >0$, i.e., $\hat{\bn}_k \not\in \cS, \, \forall k \ge 0$.

Letting $\psi_k$ be the angle of $\b$ from $\hat{\bn}_k$, the constraint 
$\b^\transpose \hat{\bn}_k = 1$ gives $0 \le \psi_k < \pi/2$ and
$\big \| \b \big \|_2 = 1 / \cos(\psi_k)$, and so we can write the optimization problem \eqref{eq:ProofContinuousRecursion_b_phi} equivalently as
\begin{align}
\min_{\b \in \Re^D} \, \, \, \frac{M c_D + Nc_d \cos(\phi)}{\cos(\psi_k)}
\, \, \, \text{s.t.} \, \, \, \b^\transpose \hat{\bn}_k = 1. \label{eq:ProofContinuousRecursion_b_phi_psi}
\end{align} If $\hat{\bn}_k$ is orthogonal to $\cS$, i.e., $\phi_k = \pi/2$, then 
$\J(\hat{\bn}_k) = M c_D \le \J(\b), \, \, \forall \b: \, \b^\transpose \hat{\bn}_k = 1$, with 
equality only if $\b  = \hat{\bn}_k$. As a consequence, $\bn_{k'} = \hat{\bn}_k, \, \forall k'> k$, and in particular if $\phi_0=\pi/2$, then $k^*=0$. 

So suppose that $\phi_k < \pi/2$ and let $\hat{\bn}_k^\perp$ be the normalized orthogonal projection of $\hat{\bn}_k$ onto $\cS^\perp$. We will prove that every global minimizer of 
problem \eqref{eq:ProofContinuousRecursion_b_phi_psi} must lie 
in the two-dimensional plane $\H:=\Span(\hat{\bn}_k,\hat{\bn}_k^\perp)$. To see this, let $\b$ have norm $1 / \cos(\psi_k)$ for some $\psi_k<\pi/2$. If $\psi_k > \pi/2 - \phi_k$, then such a $\b$ can not be a global minimizer of \eqref{eq:ProofContinuousRecursion_b_phi_psi}, as the feasible vector 
$\hat{\bn}_k^\perp / \sin(\phi_k) \in \H$ already gives a smaller objective, since 
\begin{align}
\J(\hat{\bn}_k^\perp / \sin(\phi_k)) = \frac{M c_D}{\sin(\phi_k)}  = 
\frac{M c_D}{\cos(\pi/2 - \phi_k)} < \frac{M c_D + N c_d \cos(\phi)}{\cos(\psi_k)} = \J(\b).
\end{align} Thus, without loss of generality, we may restrict to the case where $ \psi_k \le \pi/2-\phi_k$. Denote by $\hat{\h}_k$ the normalized projection of $\hat{\bn}_k$ onto $\cS$ and by $\hat{\bn}^\dagger $ the vector that is obtained from $\hat{\bn}_k$ by rotating it towards $\hat{\bn}_k^\perp$ by $\psi_k$. Note that both $\hat{\h}_k$ and $\hat{\bn}_k^\dagger$ lie in $\H$. Letting $\Psi_k \in [0,\pi]$ be the spherical angle 
between the spherical arc formed by $\hat{\bn}_k, \hat{\b}$ and the spherical arc formed by $\hat{\bn}_k,\hat{\h}_k$, the spherical law of cosines gives
\begin{align}
\cos(\angle \b, \hat{\h}_k) = \cos(\phi_k) \cos(\psi_k) + \sin(\phi_k) \sin(\psi_k) \cos(\Psi_k).
\end{align} Now, $\Psi_k$ is equal to $\pi$ if and only if $\hat{\bn}_k,\hat{\h}_k,\b$ are coplanar, i.e., if and only if $\b \in \H$. Suppose that $\b \not\in \H$. Then $\Psi_k < \pi$, and so $\cos(\Psi_k) > -1$, which implies that
\begin{align}
\cos(\angle \b, \hat{\h}_k) > \cos(\phi_k) \cos(\psi_k) - \sin(\phi_k) \sin(\psi_k) = \cos(\phi_k+\psi_k).
\end{align} This in turn implies that the principal angle $\phi$ of $\b$ from $\cS$ is strictly smaller than $\phi_k+\psi_k$, and so
\begin{align}
\J(\b) = \frac{Mc_D + N c_d \cos(\phi)}{\cos(\psi_k)} >
 \frac{Mc_D + N c_d \cos(\phi_k+\psi_k)}{\cos(\psi_k)} = \J(\hat{\bn}_k^\dagger / \cos(\psi_k)), \end{align} i.e., the feasible vector $\hat{\bn}_k^\dagger / \cos(\psi_k) \in \H$ gives strictly smaller objective than $\b$. 
 
To summarize, for the case where $\phi_k<\pi/2$, we have shown that any global minimizer $\b$ of \eqref{eq:ProofContinuousRecursion_b_phi_psi} must i) have angle $\psi_k$ from $\hat{\bn}_k$ less or equal to $\pi/2-\phi_k$, and ii) it must lie in $\Span(\hat{\bn}_k,\hat{\bn}_k^\perp)$. Hence, we can rewrite \eqref{eq:ProofContinuousRecursion_b_phi_psi} in the equivalent form
\begin{align}
\min_{\psi \in [-\pi/2+\phi_k,\pi/2-\phi_k]} \, \, \, \J_k(\psi):=\frac{M c_D + Nc_d \cos(\phi_k+\psi)}{\cos(\psi_k)}
\label{eq:ProofContinuousRecursion_phi_psi},
\end{align} where now $\psi_k$ takes positive values as $\b$ approaches $\hat{\bn}_k^\perp$ and negative values as it approaches $\hat{\h}_k$. The function $\J_k$ is continuous and differentiable in the interval $[-\pi/2+\phi_k,\pi/2-\phi_k]$, with derivative given by
\begin{align}
\frac{\partial \J_k}{\partial \psi} = \frac{Mc_D \sin(\psi)-Nc_d \sin(\phi_k)}{\cos^2(\psi)}.
\end{align} Setting the derivative to zero gives
\begin{align}
\sin(\psi) = \alpha \sin(\phi_k). \label{eq:ProofContinuousRecursionDer0}
\end{align} If $\alpha \sin(\phi_k) \ge \sin(\pi/2-\phi_k) = \cos(\phi_k)$, or equivalently 
$\tan(\phi_k) \ge 1/\alpha$, then $\J_k$ is strictly decreasing in the interval
$[-\pi/2+\phi_k,\pi/2-\phi_k]$, and so it must attain its minimum precisely at $\psi = \pi/2 -\phi_k$, which corresponds to the choice $\bn_{k+1}=\hat{\bn}_k^\perp / \sin(\phi_k)$. 
Then by an earlier argument we must have that $\hat{\bn}_{k'} \perp \cS, \, \forall k' \ge k+1$. If, on the other hand, $\tan(\phi_k) < 1/\alpha$, then 
the equation \eqref{eq:ProofContinuousRecursionDer0} defines an angle 
\begin{align}
\psi^*_k := \sin^{-1}(\alpha \sin(\phi_k)) \in (0, \pi/2-\phi_k), \label{eq:ProofContinuousRecursionPsi}
\end{align} at which $\J_k$ must attain its global minimum, since 
\begin{align}
\frac{\partial^2 \J_k}{\partial \psi ^2}\left(\psi^*_k\right) =  \frac{1}{\cos(\psi_k^*)} >0. 
\end{align} As a consequence, if $\tan(\phi_k) < 1/\alpha$, then 
\begin{align}
\phi_{k+1} = \phi_k + \sin^{-1}(\alpha \sin(\phi_k)) < \pi/2.
\end{align} We then see inductively that as long as $\tan(\phi_k) < 1/\alpha$, $\phi_k$ increases by a quantity which is bounded from below by  $\sin^{-1}(\alpha \sin(\phi_0))$. Thus, $\phi_k$ will keep increasing until it becomes greater than the solution to the 
equation $\tan(\phi) = 1 / \alpha$, at which point the global minimizer will be the vector
$\bn_{k+1} = \hat{\bn}_k^\perp / \sin(\phi_k)$, and so $\hat{\bn}_{k'} = \hat{\bn}_{k+1}, \, \forall k' \ge k+1$. 
Finally, under the hypothesis that $\phi_k < \tan^{-1}(1/\alpha)$, we have 
\begin{align}
\phi_k = \phi_0 + \sum_{j=0}^{k-1} \sin^{-1}(\alpha \sin(\phi_j)) \ge \phi_0 + k \sin^{-1}(\alpha \sin(\phi_0)),
\end{align} from where it follows that the maximal number of iterations needed for $\phi_k$ to become
larger than $\tan^{-1}(1/\alpha)$ is $\ceil[\bigg]{\frac{\tan^{-1}(1/\alpha)-\phi_0}{\sin^{-1}(\alpha \sin(\phi_0))}} $, at which point at most one more iteration will be needed
to achieve orthogonality to $\cS$. 

\subsection{Proof of Lemma \ref{lem:VectorIntegral}} \label{subsection:Proof-lem-VectorIntegral}

Letting $\boldsymbol{R}$ be a rotation that takes $\b$ to the first canonical vector $\e_1$, i.e., $\boldsymbol{R} \b = \e_1$, we have that 
\begin{align}
\int_{\z \in \Sp^{D-1}} \Sign(\b^\transpose \z)\z d \mu_{\Sp^{D-1}} &= \int_{\z \in \Sp^{D-1}} \Sign(\b^\transpose \boldsymbol{R}^\transpose \boldsymbol{R} \z)\z d \mu_{\Sp^{D-1}} \\
&=\int_{\z \in \Sp^{D-1}} \Sign(\e_1^\transpose \z)\boldsymbol{R}^\transpose \z d \mu_{\Sp^{D-1}} \\
& = \boldsymbol{R}^\transpose \int_{\z \in \Sp^{D-1}} \Sign(z_1) \z d \mu_{\Sp^{D-1}}, \label{eq:VectorIntegralIntermediate}
\end{align} where $z_1$ is the first cartesian coordinate of $\z$. Recalling the definition of $c_D$ in equation \eqref{eq:cD}, we see that
\begin{align}
\int_{\z \in \Sp^{D-1}} \Sign(z_1) z_1 d \mu_{\Sp^{D-1}} = \int_{\z \in \Sp^{D-1}} \left|z_1\right| d \mu_{\Sp^{D-1}} = c_D.
\end{align} Moreover, for any $i>1$, we have  
\begin{align}
\int_{\z \in \Sp^{D-1}} \Sign(z_1) z_i d \mu_{\Sp^{D-1}} = 0.
\end{align} Consequently, the integral in \eqref{eq:VectorIntegralIntermediate} becomes 
\begin{align}
\int_{\z \in \Sp^{D-1}} \Sign(\b^\transpose \z)\z d \mu_{\Sp^{D-1}} &= \boldsymbol{R}^\transpose \int_{\z \in \Sp^{D-1}} \Sign(z_1) \z d \mu_{\Sp^{D-1}} 
= \boldsymbol{R}^\transpose \left(c_D \e_1 \right) = c_D \b.
\end{align}	

\subsection{Proof of Lemma \ref{lem:Koksma-eO}} \label{subsection:Proof-lem-Koksma-eO}

For any $\b \in \Sp^{D-1}$ we can write
\begin{align}
c_D \b - \o_{\b} = \rho_1 \b + \rho_2 \bzeta, \label{eq:VectorIntegrationError}
\end{align} for some vector $\bzeta \in \Sp^{D-1}$ orthogonal to $\b$, and so it is enough to show that $\sqrt{\rho_1^2 + \rho_2^2} \le \sqrt{5} \mathfrak{S}_{D,M}(\bO)$.
Let us first bound from above $\left|\rho_1 \right|$ in terms of
$\mathfrak{S}_{D,M}(\bO)$. Towards that end, observe that
\begin{align}
\rho_1 &= \b^\transpose (c_D \b - \o_{\b}) =  c_D - \frac{1}{M}\sum_{j=1}^M \big | \b^\transpose \o_j \big | 
= \int_{\z \in \Sp^{D-1}} f_{\b}(\z) d \mu_{\Sp^{D-1}} - \frac{1}{M} \sum_{j=1}^M f_{\b}(\o_j),
\end{align} 
where the equality follows from the definition of $c_D$ in \eqref{eq:cD} and recalling that $f_{\b}(\z) = \left| \b^\transpose \z \right|$. In other words, $\rho_1$ is the error in approximating the integral of $f_{\b}$ on $\Sp^{D-1}$ by the average of $f_{\b}$ on the point set $\bO$. 

Now, notice that each \emph{super-level set} $\left\{\z \in \Sp^{D-1}: f_{\b}(\z) \ge \alpha \right\}$ for $\alpha \in [0,1]$, is the union of two spherical caps, and also that 
\begin{align}
\sup_{\z \in \Sp^{D-1}} f_{\b}(\z)  - \inf_{\z \in \Sp^{D-1}} f_{\b}(\z) =1 - 0= 1.
\end{align} We these in mind, repeating the entire argument of the proof of Theorem 1 in 
\citep{Harman:UDT10} that lead to inequality $(9)$ in \citep{Harman:UDT10}, but now
for a measurable function with respect to $\mu_{\Sp^{D-1}}$ (that would be $f_{\b}$), leads directly to 
\begin{align}
\left| \rho_1 \right| \le  \mathfrak{S}_{D,M}(\bO). \label{eq:Proof_rho1}
\end{align} For $\rho_2$ we have that 
\begin{align}
&\rho_2 = \bzeta^\transpose \left(c_D \b \right)- \bzeta ^\transpose \o_{\b} \\
&= \int_{\z \in \Sp^{D-1}} \Sign \left( \b ^\transpose \z \right) 
\bzeta ^\transpose \z d \mu_{\Sp^{D-1}}  - \frac{1}{M} \sum_{j=1}^M \Sign \left( \b ^\transpose \o_j \right) 
\bzeta^\transpose \o_j \\
& = \int_{\z \in \Sp^{D-1}}  g_{\b,\bzeta}( \z) d \mu_{\Sp^{D-1}}  - \frac{1}{M} \sum_{j=1}^M g_{\b,\bzeta}(\o_j), 
\end{align} where $g_{\b,\bzeta}: \Sp^{D-1} \rightarrow \Re$ is defined as 
$g_{\b,\bzeta}(\z) = \Sign \left( \b ^\transpose \z \right) \bzeta^\transpose \z$. Then a similar argument as for $\rho_1$, with the difference that now
\begin{align}
\sup_{\z \in \Sp^{D-1}} g_{\b,\bzeta}(\z)  - \inf_{\z \in \Sp^{D-1}} g_{\b,\bzeta}(\z) =1 - (-1)= 2,
\end{align} leads to 
\begin{align}
\left| \rho_2 \right| \le  2\mathfrak{S}_{D,M}(\bO). \label{eq:Proof_rho2}
\end{align} In view of \eqref{eq:Proof_rho1}, inequality \eqref{eq:Proof_rho2} establishes that
$\sqrt{\rho_1^2 + \rho_2^2} \le \sqrt{5} \mathfrak{S}_{D,M}(\bO)$, which concludes the proof of the lemma.

\subsection{Proof of Lemma \ref{lem:VectorIntegralInliers}} \label{subsection:Proof-lem-VectorIntegralInliers}

Since $\x$ lies in $\cS$, we have $\f_{\b}(\x) = \f_{\v}(\x) = \f_{\hat{\v}}(\x)$, so that 
\begin{align}
\int_{\x \in \Sp^{D-1} \cap \cS} \Sign(\b^\transpose \x)\x d \mu_{\Sp^{D-1}} = \int_{\x \in \Sp^{D-1} \cap \cS} \Sign(\hat{\v}^\transpose \x)\x d \mu_{\Sp^{D-1}}.
\end{align} Now express $\x$ and $\hat{\v}$ on a basis of $\cS$, use Lemma \ref{lem:VectorIntegral} replacing $D$ with $d$, and then switch back to the standard basis of $\Re^D$.

\subsection{Proof of Theorem \ref{thm:DiscreteNonConvex}} \label{subsection:Proof-thm-DiscreteNonConvex}

To prove the theorem we need the following lemma.
\begin{lem} \label{lem:eNaturalBounds}
For any $\b \in \Sp^{D-1}$ we have that
\begin{align}
& M (c_D +  \eO) \ge \big \| \bO^\transpose \b \big \|_1 \ge M (c_D -  \eO)\\
& N(c_d + \eX) \, \cos(\phi) \ge \big \| \bX^\transpose \b \big \|_1 \ge N(c_d - \eX) \ \cos(\phi).
\end{align}
\end{lem}
\begin{proof}
We only prove the second inequality as the first is even simpler. Let $\v \neq \0$ be the orthogonal projection of $\b$ onto $\cS$.  By definition of $\eX$, there exists a vector $\bxi \in \cS$ of $\ell_2$ norm less or equal to $\eX$, such that 
\begin{align}
 \x_{\v} = \x_{\b} = \frac{1}{N} \sum_{j=1}^N \Sign(\b^\transpose \x_j)\x_j = c_d \hat{\v} + \bxi.
\end{align} Taking inner product of both sides with $\b$ gives
\begin{align}
\frac{1}{N} \big \| \bX^\transpose \b \big \|_1 = c_d \, \cos(\phi) + \b^\transpose \bxi.
\end{align} Now, the result follows by noting that $\left| \b^\transpose \bxi \right| \le \eX \cos(\phi)$, since the principal angle of $\b$ from $\Span(\bxi)$ can not be less then $\phi$.
\end{proof} 
Now, let $\b^*$ be an optimal solution of \eqref{eq:ell1}. Then $\b^*$ must satisfy the first order optimality relation
\begin{align}
\0 \in \lambda \b^* + \btX \Sgn(\btX^\transpose \b^*) 
\label{eq:OptimalityCondition},
\end{align} where $\lambda$ is a scalar Lagrange multiplier parameter, and
$\Sgn$ is the sub-differential of the $\ell_1$ norm. 
For the sake of contradiction, suppose that $\b^* \not\perp \cS$. If $\b^* \in \cS$, then using Lemma \ref{lem:eNaturalBounds} we have
\begin{align}
& M \, c_D + M \, \epsilon_{\bO} \ge \min_{\b \perp \cS, \b^\transpose \b=1} \big \|\bO^\transpose \b \big \|_1 \ge
\big \|\bO^\transpose \b^* \big \|_1 + \big \|\bX^\transpose \b^* \big \|_1 \nonumber \\
& \ge M \, c_D - M \, \epsilon_{\bO} + N \, c_d - N \, \epsilon_{\bX},
\end{align} which violates the first inequality of hypothesis \eqref{eq:gammaUP}. Hence, we can assume that 
$\b^* \not\in \cS$. 
 
By the general position hypothesis as well as Proposition \ref{prp:NonConvexMaximalInterpolation}, $\b^*$ will be orthogonal to precisely $D-1$ points, among which 
$K_1$ points belong to $\bO$, say, without loss of generality, $\o_1,\dots,\o_{K_1}$, and $0 \le K_2\le d-1$ points belong to $\bX$, say $\x_1,\dots,\x_{K_2}$. Then there must exist real numbers $-1 \le \alpha_j,\beta_j \le 1$, such that
\small
\begin{align}
\lambda \b^*+\sum_{j=1}^{K_1} \alpha_j \o_j 
+\sum_{j=K_1+1}^M \Sign(\o_j^\transpose \b^*) \o_j +\sum_{j=1}^{K_2} \beta_j \x_j + \sum_{j=K_2+1}^N \Sign(\x_j^\transpose \b^*) \x_j  = \0.
\end{align} \normalsize Since $\Sign(\o_j^\transpose \b^*)=0, \forall j \le K_1$ and similarly $\Sign(\x_j^\transpose \b^*)=0, \forall j \le K_2$, we can equivalently write
\begin{align}
\lambda \b^*+\sum_{j=1}^{K_1} \alpha_j \o_j 
+\sum_{j=1}^M \Sign(\o_j^\transpose \b^*) \o_j +\sum_{j=1}^{K_2} \beta_j \x_j + \sum_{j=1}^N \Sign(\x_j^\transpose \b^*) \x_j  = \0 \label{eq:OptimalityExpanded}
\end{align} or more compactly  
\begin{align}
\lambda \b^*+\bxi_{\bO}
+M \, \o_{\b^*} +\bxi_{\bX}+ N \, \x_{\hat{\v}^*}  = \0 \label{eq:Opt},
\end{align} where $\hat{\v}^*$ is the normalized projection of $\b^*$ onto $\cS$ (nonzero since $\b^* \not\perp \cS$ by hypothesis), and 
\begin{align}
\o_{\b^*} &:= \frac{1}{M} \sum_{j=1}^M \Sign(\o_j^\transpose \b^*) \o_j, &  \x_{\hat{\v}^*} &:= \frac{1}{N} \sum_{j=1}^N \Sign(\x_j^\transpose \hat{\v}^*) \x_j, \\
\bxi_{\bO} &:=\sum_{j=1}^{K_1} \alpha_j \o_j,  & \bxi_{\bX} &:=\sum_{j=1}^{K_2} \beta_j \x_j.
\end{align} From the definitions of $\eO$ and $\eX$ in \eqref{eq:epsilonO} and \eqref{eq:epsilonX} respectively, we have that  
\begin{align}
\o_{\b^*} &= c_D \, \b^* + \boldeta_{\bO},\, \, \,   ||\boldeta_{\bO}||_2 \le \eO \\
\x_{\hat{\v}^*} &= c_d \,\hat{\v}^*  + \boldeta_{\bX}, \, \, \,  ||\boldeta_{\bX}||_2 \le \eX,
\end{align} and so \eqref{eq:Opt} becomes
\begin{align}
\lambda \b^*+
\bxi_{\bO}
+M \,  c_D \, \b^* +M \, \boldeta_{\bO} +\bxi_{\bX}+ N \, c_d \,\hat{\v}^* + N\, \boldeta_{\bX}  = \0 \label{eq:OptFinal}.
\end{align} Since $\b^* \not\in \cS$, we have that $\b^*,\hat{\v}^*$ are linearly independent. Define the two-dimensional subspace $\U:= \Span\left(\b^*,\hat{\v}^*\right)$ and project
\eqref{eq:OptFinal} onto $\U$ to get
\begin{align}
&\lambda \b^*+\pi_\U(\bxi_{\bO})
+M \,  c_D \, \b^* +M \, \pi_\U(\boldeta_{\bO}) +\pi_\U(\bxi_{\bX})+ N \, c_d \,\hat{\v}^*  + N\, \pi_\U(\boldeta_{\bX})  = \0. \label{eq:OptProjected}
\end{align} Now, very vector $\u$ in the image of $\pi_{\U}$ 
can be written as a linear combination of $\b^*$ and $\hat{\v}^*$:
\begin{align}
\u = \left[\u\right]_{\b^*} \, \b^* +  \left[\u\right]_{\hat{\v}^*} \, \hat{\v}^*,\, \, \, \text{with} \, \, \, \left[\u\right]_{\b^*}, \, \left[\u\right]_{\hat{\v}^*} \in \Re.
\end{align} Taking inner product of $\u$ with $\b^*$ and $\hat{\v}^*$, we get respectively
\begin{align}
\u^\top \b^* &= \left[\u\right]_{\b^*} + \left[\u\right]_{\hat{\v}^*} \cos(\phi^*) \\
\u^\top \hat{\v}^* &= \left[\u\right]_{\b^*} \cos(\phi^*) + \left[\u\right]_{\hat{\v}^*}, 
\end{align} where $\phi^*$ is the angle between $\b^*$ and $\hat{\v}^*$, i.e., the angle of $\b^*$ from $\cS$. Solving with respect to $\left[\u\right]_{\hat{\v}^*}$, we obtain
\begin{align}
\left[\u\right]_{\hat{\v}^*} = \frac{\u^\top \hat{\v}^*-\u^\top \b^* \cos(\phi^*)}{1-\cos^2(\phi^*)},
\end{align} which in turn gives an upper bound on the magnitude of $\left[\u\right]_{\hat{\v}^*}$:
\begin{align}
\left|\left[\u\right]_{\hat{\v}^*}\right| \le \frac{1+\cos(\phi^*)}{1-\cos^2(\phi^*)} \, \left\|\u\right\|_2. \label{eq:ProofTheoremDiscreteUpperBoundCoefficient}
\end{align} Going back to \eqref{eq:OptProjected} and writing each vector as a linear combination of $\b^*$ and $\hat{\v}^*$, we obtain
\small
\begin{align}
&\lambda \b^* + \left[\pi_\U(\bxi_{\bO})\right]_{\b^*} \b^* + \left[\pi_\U(\bxi_{\bO})\right]_{\hat{\v}^*}\hat{\v}^*
+M \,  c_D \, \b^* 
+M \, \left[\pi_\U(\boldeta_{\bO})\right]_{\b^*}\b^*+ M \, \left[\pi_\U(\boldeta_{\bO})\right]_{\hat{\v}^*}\hat{\v}^* + \nonumber \\
&\left[\pi_\U(\bxi_{\bX})\right]_{\b^*}\b^*+\left[\pi_\U(\bxi_{\bX})\right]_{\hat{\v}^*}\hat{\v}^*+ N \, c_d \,\hat{\v}^* 
 + N\, \left[\pi_\U(\boldeta_{\bX})\right]_{\b^*}\b^*+N\, \left[\pi_\U(\boldeta_{\bX})\right]_{\hat{\v}^*}\hat{\v}^*  = \0.
\end{align} \normalsize Since $\U$ is a two-dimensional space, there exists a vector 
$\hat{\bzeta} \in \U$ that is orthogonal to $\b^*$ but not orthogonal to $\hat{\v}^*$. 
Projecting the above equation onto the line spanned by $\hat{\bzeta}$,
we obtain the one-dimensional equation
\small
\begin{align}
\left(\left[\pi_\U(\bxi_{\bO})\right]_{\hat{\v}^*} +M\left[\pi_\U(\boldeta_{\bO})\right]_{\hat{\v}^*}+
\left[\pi_\U(\bxi_{\bX})\right]_{\hat{\v}^*}+N \, c_d+N \, \left[\pi_\U(\boldeta_{\bX})\right]_{\hat{\v}^*} \right) \cdot  
\hat{\bzeta}^\transpose {\hat{\v}^*} = 0.
\end{align} \normalsize Since $\hat{\bzeta}$ is not orthogonal to $\hat{\v}^*$, the above equation implies that 
\begin{align}
\left[\pi_\U(\bxi_{\bO})\right]_{\hat{\v}^*} +M \, \left[\pi_\U(\boldeta_{\bO})\right]_{\hat{\v}^*}+
\left[\pi_\U(\bxi_{\bX})\right]_{\hat{\v}^*}+N \, c_d+N\, \left[\pi_\U(\boldeta_{\bX})\right]_{\hat{\v}^*}=0,
\end{align} which, in turn, implies that
\begin{align}
N \, c_d \le \left| \left[\pi_\U(\bxi_{\bO})\right]_{\hat{\v}^*}\right| +M \, \left|\left[\pi_\U(\boldeta_{\bO})\right]_{\hat{\v}^*}\right|+
\left|\left[\pi_\U(\bxi_{\bX})\right]_{\hat{\v}^*}\right|+N\, \left|\left[\pi_\U(\boldeta_{\bX})\right]_{\hat{\v}^*}\right|.
\end{align} Invoking the upper bound of \eqref{eq:ProofTheoremDiscreteUpperBoundCoefficient} together with 
\begin{align}
\big \| \bxi_{\bO} \big \|_2 \le \R_{\bO,K_1}, \, \big \| \bxi_{\bX} \big \|_2 \le \R_{\bX,K_2}, \, \big \| \boldeta_{\bO} \big \|_2 \le \eO, \, \big \| \boldeta_{\bX} \big \|_2 \le \eX,
\end{align} and the definition of $\R_{\bO,\bX}$ (Definition \ref{dfn:DPCPsingle_circumradius}), we get
\begin{align}
N \, c_d \le \frac{1+\cos(\phi^*)}{1-\cos^2(\phi^*)} \left(\R_{\bO,\bX}+M\, \eO + N \, \eX \right),
\end{align} or equivalently
\begin{align}
&N \, c_d \, \cos^2(\phi^*) +(\R_{\bO,\bX}+M\, \eO + N\, \eX) \, \cos(\phi^*) \nonumber \\
&+ (\R_{\bO,\bX}+M\, \eO + N\, \eX - N \, c_d) \ge 0.
\end{align} This is a quadratic polynomial in $\cos(\phi^*)$, whose constant term is negative by the second inequality of hypothesis \eqref{eq:gammaUP}, and thus has exactly one positive and one negative root. As consequence, this polynomial being non-negative together with the fact that $\cos(\phi^*) >0$, implies that $\cos(\phi^*)$ must be greater than the positive root of the polynomial, i.e., 
\begin{align}
\cos(\phi^*) \ge \frac{-\Q+\sqrt{\Q^2+4 c_d(c_d-\Q)}}{2c_d}, \, \, \, \Q:= \frac{\R_{\bO,\bX}}{N}+\eO\frac{M}{N} +  \eX. \label{eq:ProofTheoremDiscretePhiBound}
\end{align} On the other hand, by Lemma \ref{lem:eNaturalBounds} we have
\small
\begin{align}
M(c_D +\eO)\ge \min_{\hat{\b} \perp \cS }\left\| \btX^\top \hat{\b} \right\|_1 \ge \left\| \btX^\top \b^* \right\|_1 \ge M(c_D -\eO) + N(c_d -\eX) \cos(\phi^*), \label{eq:ProofNonConvexPhiSmallLargeObjective}
\end{align} \normalsize which implies that
\begin{align}
2 M \eO \ge N(c_d -\eX) \frac{-\Q+\sqrt{\Q^2+4 c_d(c_d-\Q)}}{2c_d}.
\end{align} This latter inequality is equivalent to the inequality
\begin{align}
& 2\eO^2(3c_d-\eX)\left(\frac{M}{N}\right)^2+\eO(c_d-\eX)\left(2\frac{\R_{\bO,\bX}}{N}+\eX+c_d\right)\frac{M}{N}\nonumber \\
&-(c_d-\eX)^2\left(c_d-\eX-\frac{\R_{\bO,\bX}}{N}\right) \ge 0,
\end{align} whose left-hand-side we view as quadratic polynomial in $M/N$. By the first two inequalities of hypothesis \eqref{eq:gammaUP}, the second term of this polynomial is positive, while the constant term is negative, and so this inequality is equivalent to $M/N$ being greater or equal than the unique positive root of that polynomial. But this contradicts the third inequality of hypothesis \eqref{eq:gammaUP}. Consequently, the initial hypothesis of the proof that $\b^* \not\perp \cS$ can not be true, and the theorem is proved.

\myparagraph{A Geometric View of the Proof of Theorem \ref{thm:DiscreteNonConvex}} 
Let us provide some geometric intuition that underlies the proof of Theorem \ref{thm:DiscreteNonConvex}. It is instructive to begin our discussion by considering the case
$d=1, D=2$, i.e. the inlier space is simply a line and the ambient space is a $2$-dimensional plane. 
Since all points have
unit $\ell_2$-norm, every column of $\bX$ will be of the form $\hat{\x}$ or $-\hat{\x}$ for a fixed
vector $\hat{\x} \in \cS^{1}$ that spans the inlier space $\cS$. In this setting, let us examine a global solution $\b^*$ of the optimization problem \eqref{eq:ell1}. We will start by assuming that such a $\b^*$ is not orthogonal to $\cS$, and intuitively arrive at the conclusion that this can not be the case as long as there are \emph{sufficiently many} inliers.

We will argue on an intuitive level that if $\b^* \not\perp \cS$, then the principal angle $\phi^*$ of $\b^*$ from $\cS$ needs to be small; this is captured precisely by \eqref{eq:ProofTheoremDiscretePhiBound} in the proof of the theorem. To see this, suppose $\b^* \not\perp \cS$; then
$\b^*$ will be non-orthogonal to every inlier, and by Proposition\ref{prp:NonConvexMaximalInterpolation} orthogonal to $D-1=1$ outlier, say $\o_1$. The optimality condition \eqref{eq:OptimalityCondition} specializes to 
\begin{align}
 \alpha_1 \o_1 
+\underbrace{\sum_{j=1}^M \Sign(\o_j^\transpose \b^*) \o_j}_{M\o_{\b^*}}  + \sum_{j=1}^N \Sign(\x_j^\transpose \b^*) \x_j + \lambda \b^* = \0,
\end{align} where $-1 \le \alpha_1 \le 1$. Notice that the third term is simply $N \, \Sign(\hat{\x}^\transpose \b^*) \hat{\x}$, and so
\begin{align}
\alpha_1 \o_1
+M \, \o_{\b^*} + \lambda \b^* =- N \, \Sign(\hat{\x}^\transpose \b^*) \hat{\x}. \label{eq:OptimalityLine}
\end{align} Now, what \eqref{eq:OptimalityLine} is saying is that the point $- N \, \Sign(\hat{\x}^\transpose \b^*)\hat{\x}$ must lie inside the 
set 
\begin{align}
\Conv(\pm \o_1)+\left\{M\o_{\b^*}\right\} + \Span(\b^*)=
\left\{ \alpha_1 \o_1
+M\o_{\b^*} + \lambda \b^* : \, |\alpha_1|\le 1, \lambda \in \Re \right\}, 
\end{align} where the $+$ operator on sets is the Minkowski sum. Notice that 
the set $\Conv(\pm \o_1)+M\o_{\b^*}$ is the translation of the 
line segment (polytope) $\Conv(\pm \o_1)$ by $M\o_{\b^*}$. Then \eqref{eq:OptimalityLine}
says that if we draw all affine lines that originate from every point of $\Conv(\pm \o_1)+M\o_{\b^*}$ and have direction $\b^*$, then one of these lines
must meet the point $- N \, \Sign(\hat{\x}^\transpose \b^*)\hat{\x}$. Let us illustrate this
for the case where $M=N=5$ and say it so happens that $\b^*$ has a rather large angle $\phi^*$ from $\cS$, say $\phi^*=45^\circ$. 
Recall that $\o_{\b^*}$ is concentrated around
$\c_D \, \b^*$ and for the case $D=2$ we have $c_D=\frac{2}{\pi}$.
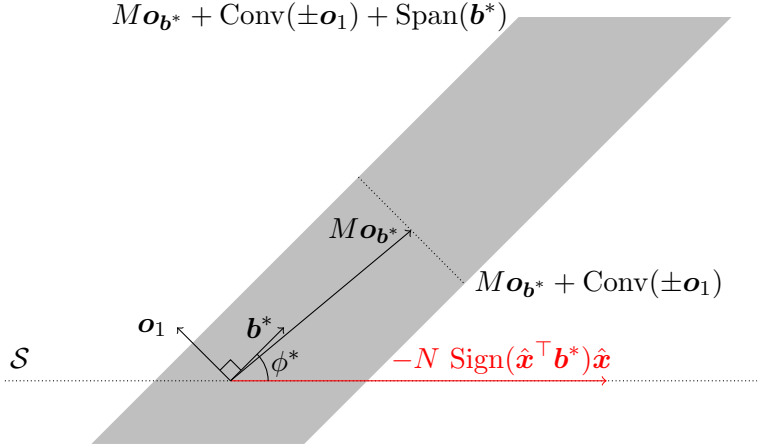
\begin{figure}[!h]
	\centering
	\begin{tikzpicture}
		\fill[fill=lightgray] (-1.85,-0.84,0)--(3.83,4.84,0)node[anchor=east]{$M\o_{\b^*}+\Conv(\pm\o_1)+\Span(\b^*)$}-- (6.66,4.84,0)--(3.11,1.29,0)--(0.98,-0.84,0)--(-1.85,-0.84,0);
	\draw[densely dotted] (-3,0,0) -- (7,0,0);	
	\node at (-2.8,0.3,0) {$\cS$};
	\draw[color=red][->] (0,0,0)  -- (5,0,0);	
	 \node[color=red][anchor=west] at (2,0.3,0) {$- N \, \Sign(\hat{\x}^\transpose \b^*)\hat{\x}$};
	\draw[densely dotted] (1.7,2.71,0) -- (3.11,1.29,0)node[anchor=west]{$M\o_{\b^*}+\Conv(\pm\o_1)$};
	\draw[->] (0,0,0) -- (2.4,2,0) node[anchor=east]{$M\o_{\b^*}$};
\draw[->] (0,0,0) -- (0.71,0.71,0) node[anchor=east]{$\b^*$};
\draw (0.5,0,0) arc (0:45:0.5); 
 \node[anchor=west] at (0.4,0.23,0) {$\phi^*$};
\draw[->] (0,0,0) -- (-0.71,0.71,0) node[anchor=east]{$\o_1$};
\draw[] (-0.14,0.14,0) -- (0,0.28,0) -- (0.14,0.14,0) ;
	\end{tikzpicture}
	\caption{Geometry of the optimality condition \eqref{eq:OptimalityCondition} and \eqref{eq:OptimalityLine} for the case $d=1,D=2,M=N=5$. The polytope $M\o_{\b^*}+\Conv(\pm\o_1)+\Span(\b^*)$ misses the point $- N \, \Sign(\hat{\x}^\transpose \b^*)\hat{\x}$ and so the optimality condition can not be true for both $\b^* \not\perp \cS = \Span(\hat{\x})$ and $\phi^*$ large.}
	\label{fig:OptimalLine45}
\end{figure} As illustrated in Figure \ref{fig:OptimalLine45}, because $\phi^*$ is large, the unbounded polytope $M\o_{\b^*}+\Conv(\pm\o_1)+\Span(\b^*)$ misses the point $- N \, \Sign(\hat{\x}^\transpose \b^*)\hat{\x}$ thus making the optimality equation \eqref{eq:OptimalityLine} infeasible. This indicates that critical vectors $\b^* \not\perp \cS$ having large angles from $\cS$ are unlikely to exist. 

On the other hand, critical points $\b^* \not\perp \cS$ may exist, but their angle $\phi^*$ from $\cS$ needs to be small, as illustrated in Figure \ref{fig:OptimalLine5}. 
\begin{figure}[!h]
	\centering
	\begin{tikzpicture}
		\fill[fill=lightgray,rotate=-35] (-1.85,-0.84,0)--(3.83,4.84,0)node[anchor=east]{$M\o_{\b^*}+\Conv(\pm\o_1)+\Span(\b^*)$}-- (6.66,4.84,0)--(3.11,1.29,0)--(0.98,-0.84,0)--(-1.85,-0.84,0);		
		\draw[densely dotted] (-3,0,0)  -- (9,0,0);
		\node[] at (-2.8,0.3,0) {$\cS$};
	\draw[color=red][->] (0,0,0)  -- (5,0,0) 
	node[anchor=south]
	{$- N \, \Sign(\hat{\x}^\transpose \b^*)\hat{\x}$};	
	\draw[densely dotted,rotate=-35] (1.7,2.71,0) -- (3.11,1.29,0) node[anchor=west]{$M\o_{\b^*}+\Conv(\pm\o_1)$};
	\draw[->,rotate=-35] (0,0,0) -- (2.4,2,0); 
	\node[] at (2.4,0.5,0) {$M\o_{\b^*}$};
\draw[->,rotate=-35] (0,0,0) -- (0.71,0.71,0) node[anchor=south]{$\b^*$};
\draw[->,rotate=-35] (0,0,0) -- (-0.71,0.71,0) node[anchor=east]{$\o_1$};
\draw[rotate=-35] (-0.14,0.14,0) -- (0,0.28,0) -- (0.14,0.14,0) ;
	\end{tikzpicture}
	\caption{Geometry of the optimality condition \eqref{eq:OptimalityCondition} and \eqref{eq:OptimalityLine} for the case $d=1,D=2,M=N=5$. A critical $\b^* \not\perp \cS$ exists, but its angle from $\cS$ is small, so that the polytope $M\o_{\b^*}+\Conv(\pm\o_1)+\Span(\b^*)$  can contain the point $- N \, \Sign(\hat{\x}^\transpose \b^*)\hat{\x}$. However, $\b^*$ can not be a global minimizer, since small angles from $\cS$ yield large objective values.}
	\label{fig:OptimalLine5}
\end{figure}
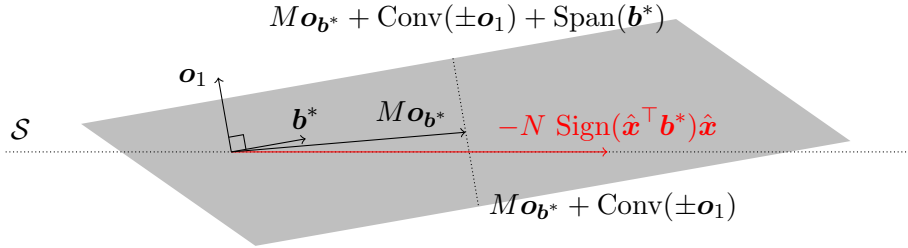
 However, such critical points can not be global minimizers, because small angles from $\cS$ yield large objective values; this is captured precisely by equation \eqref{eq:ProofNonConvexPhiSmallLargeObjective} in the proof of the theorem. Hence the only possibility that critical points $\b^* \not\perp \cS$ that are also global minimizers do exist is that the number of inliers is significantly less than the number of outliers, i.e. 
$N << M$, as illustrated in Figure \ref{fig:OptimalLineSmallN}. The precise notion of how many inliers should exist with respect to outliers is captured by condition \eqref{eq:gammaUP} of Theorem \ref{thm:DiscreteNonConvex}.
\begin{figure}[!h]
	\centering
	\begin{tikzpicture}
		\fill[fill=lightgray] (-1.85,-0.84,0)--(3.83,4.84,0)node[anchor=east]{$M\o_{\b^*}+\Conv(\pm\o_1)+\Span(\b^*)$}-- (6.66,4.84,0)--(3.11,1.29,0)--(0.98,-0.84,0)--(-1.85,-0.84,0);
	\draw[color=red][->] (0,0,0)  -- (1,0,0) node[anchor=west]{$- N \, \Sign(\hat{\x}^\transpose \b^*)\hat{\x}$};
	\draw[densely dotted] (1.7,2.71,0) -- (3.11,1.29,0)node[anchor=west]{$M\o_{\b^*}+\Conv(\pm\o_1)$};
	\draw[->] (0,0,0) -- (2.4,2,0) node[anchor=east]{$M\o_{\b^*}$};
\draw[->] (0,0,0) -- (0.71,0.71,0) node[anchor=east]{$\b^*$};
\draw[->] (0,0,0) -- (-0.71,0.71,0) node[anchor=east]{$\o_1$};
\draw[] (-0.14,0.14,0) -- (0,0.28,0) -- (0.14,0.14,0) ;
	\end{tikzpicture}
	\caption{Geometry of the optimality condition \eqref{eq:OptimalityCondition} and \eqref{eq:OptimalityLine} for the case $d=1,D=2,N<<M$. Critical points $\b^* \not\perp \cS$ do exist and moreover they can have large angle from $\cS$. This is because $N$ is small and so the polytope $M\o_{\b^*}+\Conv(\pm\o_1)+\Span(\b^*)$ contains the point $- N \, \Sign(\hat{\x}^\transpose \b^*)\hat{\x}$. Moreover, such critical points can be global minimizers. Condition \eqref{eq:gammaUP} of Theorem \ref{thm:DiscreteNonConvex} prevents such cases from occuring.}
	\label{fig:OptimalLineSmallN}
\end{figure}
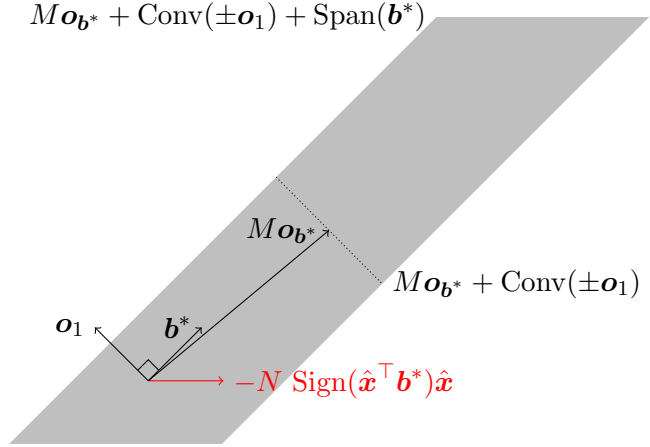 

We should note here that the picture for the general setting is analogous to what we described above, albeit harder to visualize: with reference to equation \eqref{eq:OptimalityExpanded}, the optimality condition says that every feasible point $\b^* \not\perp \cS$ must have the following property: there exist
$0 \le K_2 \le d-1$ inliers $\x_1,\dots,\x_{K_2}$ and
 $0 \le K_1 \le D-1-K_2$ outliers $\o_1,\dots,\o_{K_1}$ to which $\b^*$ is orthogonal, and two points $\bxi_{\bO} \in \Conv(\pm \o_1 \pm \cdots \pm \o_{K_1})+\o_{\b^*}$ and $\bxi_{\bX} \in \Conv(\pm \x_1 \pm \cdots \pm \x_{K_2})+\x_{\b^*}$ that are joined by an affine line that is parallel to the line spanned by $\b^*$. In fact in our proof of Theorem \ref{thm:DiscreteNonConvex} we reduced this general case to the case $d=1,D=2$ described above: this reduction is precisely taking place in equation
\eqref{eq:OptProjected}, where we project the optimality equation onto the $2$-dimensional subspace $\U$. The arguments that follow this projection consist of nothing more than a technical treatment of the intuition given above.

\subsection{Proof of Theorem \ref{thm:DiscreteConvexRelaxations}} \label{subsection:Proof-thm-DiscreteConvexRelaxations}

First note that if \eqref{eq:gammaUP} is true, then the expression of \eqref{eq:phiLB} always defines an angle between $0$ and $\pi/2$.
We start by establishing that $\hat{\bn}_k$ does not lie in the inlier space $\cS$. For $k=0$ this is true by the hypothesis \eqref{eq:phiLB}. 
For the sake of contradiction suppose that $\hat{\bn}_k \in \cS$ for some $k>0$. Note that
		\begin{align}
		\big \|\btX^\transpose \hat{\bn}_0 \big \|_1 \ge \big \|\btX^\transpose \bn_1 \big \|_1 \ge \big \|\btX^\transpose \hat{\bn}_1 \big \|_1 \ge 
		\cdots \ge \big \|\btX^\transpose \hat{\bn}_k \big \|_1. \label{eq:Jchain}
		\end{align} Suppose first that $\hat{\bn}_0 \perp \cS$. Then \eqref{eq:Jchain} gives
		\begin{align}
		\big \|\bO^\transpose \hat{\bn}_0 \big \|_1 \ge \big \|\bO^\transpose \hat{\bn}_k \big \|_1 + \big \|\bX^\transpose \hat{\v}_k \big \|_1, \label{eq:PerpContradiction}
		\end{align} where $\hat{\v}_k$ is the normalized projection of $\hat{\bn}_k$ onto $\cS$ (and since $\hat{\bn}_k \in \cS$, these two are equal). Using Lemma \ref{lem:eNaturalBounds}, we take an upper bound of the LHS and a lower bound of the RHS of \eqref{eq:PerpContradiction}, and obtain
		\begin{align}
		M \, c_D + M \, \eO \ge M \, c_D - M \, \eO + N \, c_d - N \, \eX,		
		\end{align} or equivalently
		\begin{align}
		\frac{M}{N} \ge \frac{c_d-\eX}{2\, \eO},
		\end{align} which contradicts the first inequality of \eqref{eq:gammaUP}. Consequently,
		$\hat{\bn}_0 \not\perp \cS$. Then \eqref{eq:Jchain} implies that
		\begin{align}
		\big \|\bO^\transpose \hat{\bn}_0 \big \|_1 + \big \|\bX^\transpose \hat{\bn}_0 \big \|_1 \ge \big \|\bO^\transpose \hat{\bn}_k \big \|_1 + \big \|\bX^\transpose \hat{\bn}_k \big \|_1,
		\end{align} or equivalently
		\begin{align}
		\big \|\bO^\transpose \hat{\bn}_0 \big \|_1 + \cos(\phi_0)\big \|\bX^\transpose \hat{\v}_0 \big \|_1 \ge \big \|\bO^\transpose \hat{\bn}_k \big \|_1 + \big \|\bX^\transpose \hat{\v}_k \big \|_1, \label{eq:InequalityNotInlier}
		\end{align} where $\hat{\v}_{0}$ is the normalized projection of $\hat{\bn}_{0}$ onto $\cS$. Once again, using Lemma \ref{lem:eNaturalBounds} we obtain the following contradiction to \eqref{eq:phiLB}:
\begin{align}
M\, c_D + M \, \eO + (N\, c_d + N\, \eX)\cos(\phi_0) \ge M \, c_D - M \, \eO + N \, c_d - N \, \eX.
\end{align} Now let us complete the proof of the theorem. We know by Proposition \ref{prp:LPconvergence} that the sequence $\left\{\bn_k \right\}$ converges to a critical point $\bn_{k^*}$ of problem \eqref{eq:ell1} in a finite number of steps $k^*$, and we have already shown that $\bn_{k^*} \not\in \cS$. If $\bn_{k^*}$ is not orthogonal to $\cS$, an identical argument as in the proof of Theorem \ref{thm:DiscreteNonConvex} (with $\bn_{k^*}$ in place of $\b^*$) shows that the principal angle $\phi_{k^*}$ of $\bn_{k^*}$ from $\cS$ satisfies
\begin{align}
\cos(\phi_{k^*}) \ge \frac{-\Q+\sqrt{\Q^2+4 c_d(c_d-\Q)}}{2c_d}, \, \, \, \Q:= \frac{\R_{\bO,\bX}}{N}+\eO\frac{M}{N} +  \eX. \label{eq:ProofTheoremDiscreteConvexPhiBound}
\end{align} However, due to \eqref{eq:Jchain} and Lemma \ref{lem:eNaturalBounds}, we have that 
\begin{align}
M(c_D+\eO) +N(c_d+\eX)\cos(\phi_0) \ge M(c_D-\eO)+N(c_d-\eX) \cos(\phi_{k^*}),
\end{align} which, after substituting the lower bound \eqref{eq:ProofTheoremDiscreteConvexPhiBound}, contradicts \eqref{eq:phiLB}. Thus $\bn_{k^*} \perp \cS$.

\section{Dual Principal Component Pursuit Algorithms} \label{section:Algorithm}
We present algorithms based on the ideas discussed so far, for estimating the inlier linear subspace in the presence of outliers. Specifically, in \S
\ref{subsection:DPCP-LP} we describe the main algorithmic contribution of this
paper, which is based on the implementation of the recursion \eqref{eq:ConvexRelaxations} via linear programming. In \S \ref{subsection:DPCP-IRLS} we propose an alternative way of computing dual principal components based on Iteratively Reweighted Least-Squares, which, as will be seen in \S \ref{section:Experiments} performs almost as well as recursion \eqref{subsection:DPCP-LP}, yet it is significantly more efficient. Finally, in \S \ref{subsection:DPCP-d} we present a variation of the DPCP optimization problem \eqref{eq:ell1} suitable for noisy data and propose a heuristic method for solving it.

\subsection{DPCP via Linear Programming (DPCP-LP)}\label{subsection:DPCP-LP}
For the sake of an argument, suppose that there is no noise in the inliers, i.e., the inliers $\bX$ span a linear subspace $\cS$ of dimension $d$. Then Theorem \ref{thm:DiscreteConvexRelaxations} suggests a mechanism for obtaining an element $\b_1$ of ${\cS}^\perp$: run the recursion of linear programs \eqref{eq:ConvexRelaxations} until the sequence $\hat{\bn}_{k}$ converges and identify the limit point with $\b_1$. Due to computational constraints, in practice one usually terminates the recursion when the objective value $\big \|\btX^\transpose \hat{\bn}_k\big\|_1$ converges within some small $\varepsilon$, or a maximal number $T_{\max}$ of iterations is reached, and obtains a normal vector $\b_1$. Having computed a vector $\b_1$, there are two possibilities: either $\cS$ is a hyperplane of dimension $D-1$ or $\dim {\cS}< D-1$. In the first case we can identify our subspace model with the hyperplane defined by the normal $\b_1$. If on the other hand $\dim {\cS}< D-1$, we can proceed to find a second vector $\b_2 \perp \b_1$ that is approximately orthogonal to ${\cS}$, and so on, until we have computed an orthogonal basis for the orthogonal complement of $\cS$; this process naturally leads to Algorithm \ref{alg:DPCP-LP}, in which $c$ is an estimate for the codimension $D-d$ of the inlier subspace $\Span(\bX)$.

\begin{algorithm}[t!] \caption{Dual Principal Component Pursuit via Linear Programming}\label{alg:DPCP-LP} 
\begin{algorithmic}[1] 
\Procedure{DPCP-LP}{$\btX,c,\varepsilon,T_{\max}$}		
\State $\mathcal{B} \gets \emptyset$;
\For{ $ i=1 : c$}
\State $k \gets 0; \J \gets 0; \Delta \mathcal{J} \gets \infty$; 
\State $\hat{\bn}_0 \gets \w  \in  \argmin_{\big \|\b\big \|_2=1, \, \b \perp \mathcal{B}} \big \|\btX^\transpose \b \big \|_2$; 	
\While{$k < T_{\max}$ and $\Delta \mathcal{J} > \varepsilon \J$} 
\State $\J \gets \big \| \btX^\transpose \hat{\bn}_{k} \big \|_1$;	
\State $k \gets k+1$;
\State $\bn_k \gets \w \in \argmin_{\b^\transpose \hat{\bn}_{k-1}=1, \b \perp \mathcal{B}} \big \| \btX^\transpose \b \big \|_1$;
\State $\hat{\bn}_k \gets \bn_k / \big \|\bn_k \big \|_2$;
\State $\Delta \mathcal{J} \gets \left(\J-\big \| \btX^\transpose \hat{\bn}_{k} \big \|_1\right)$;	
\EndWhile
\State $\b_i \gets \hat{\bn}_{k}$;
\State $\mathcal{B} \gets \mathcal{B} \cup \left\{\b_i \right\}$;
\EndFor
\State \Return $\mathcal{B}$;
\EndProcedure 				
\end{algorithmic} 
\end{algorithm}

Notice how the algorithm initializes $\bn_0$: This is precisely the right singular vector of $\btX^\transpose$ that corresponds to the smallest singular value, after projection of $\btX$ onto $\Span(\b_1,\dots,\b_{i-1})^\perp$. As it will be demonstrated in \S \ref{section:Experiments}, this is a key choice, since it has the effect that the angle of $\bn_0$ from the inlier subspace is typically large, a desirable property for the success of recursion \eqref{eq:ConvexRelaxations} (see Theorem \ref{thm:DiscreteConvexRelaxations}). 
We refer to Algorithm \ref{alg:DPCP-LP} as DPCP-LP, to emphasize that the optimization problem associated with each iteration of the recursion \eqref{eq:ConvexRelaxations} is a linear program. In fact, at iteration $k$ the optimization problem is 
\begin{align}
\min_{\b} \, \, \, \big \|\btX^\transpose \b \big \|_1 \, \, \, \text{s.t.} \, \, \, \b^\transpose \hat{\bn}_{k-1} = 1,
\end{align} which can equivalently be written as a standard linear program, 
\begin{align}
\min_{\b, \u^+,\u^-} \, \, \, &
\begin{bmatrix} 
\1_{1 \times N}  & \1_{1 \times N}
\end{bmatrix} 
\begin{bmatrix} 
\u^+ \\ \u^- 
\end{bmatrix} \, \, \,   \label{eq:DPCP-LP-start} \\
\text{s.t.}\, \, \, \, \, \, \, &
\begin{bmatrix}
\boldsymbol{I}_{N} & -\boldsymbol{I}_{N} &  -\btX^\transpose  \\
\0_{1 \times N} & \0_{1 \times N} &  \hat{\bn}_k^\transpose 
\end{bmatrix} 
\begin{bmatrix} 
\u^+ \\ \u^- \\ \b
\end{bmatrix} = 
\begin{bmatrix}
\0_{N \times 1} \\
1
\end{bmatrix}, \, \, \, \u^+, \u^- \ge 0, \label{eq:DPCP-LP-end}
\end{align} and can be solved efficiently with an optimized general purpose linear programming solver, such as Gurobi \citep{gurobi}.

\subsection{DPCP via Iteratively Reweighted Least-Squares (DPCP-IRLS)} \label{subsection:DPCP-IRLS} 
Even though DPCP-LP (Algorithm \ref{alg:DPCP-LP}) comes with theoretical guarantees as per Theorem \ref{thm:DiscreteConvexRelaxations}, and moreover will be shown to have a rather remarkable performance (at least for synthetic data, see Fig. \ref{figure:separation}), it has the weakness that the linear programs (which are non-sparse) may become inefficient to solve in high dimensions and for a large number of data points. Moreover, even though DPCP-LP is theoretically applicable regardless of the subspace relative dimension $d/D$, its running time increases with the subspace codimension $c=D-d$, since the $c$ basis elements of $\cS^\perp$ are computed sequentially. This motivates us to generalize the DPCP problem \eqref{eq:ell1} to an optimization problem that targets the entire orthogonal basis of $\cS^\perp$:
\begin{align}
\min_{\B \in \Re^{D \times c}} \, \big \|\btX^\transpose \B \big\|_{1,2} \, \, \, \text{s.t.} \, \, \, \B^\transpose \B = \bI_c \label{eq:L12}.
\end{align} 

\begin{algorithm}[t!] \caption{Dual Principal Component Pursuit via Iteratively Reweighted Least Squares}\label{alg:DPCP-IRLS} \begin{algorithmic}[1] 
\Procedure{DPCP-IRLS}{$\btX,c,\varepsilon,T_{\max},\delta$}		
\State $k \gets 0; \J \gets 0; \Delta \mathcal{J} \gets \infty$; 
\State $\B_0 \gets \bW \in \argmin_{\B \in \Re^{D \times c},\, \B^\transpose \B = \bI_c} \, \, \big \|\btX^\transpose \B \big \|_F$;
\While{$k < T_{\max}$ and $\Delta \mathcal{J} > \varepsilon \J$} 
\State $\J \gets \big \|\btX^\transpose \B_{k} \big \|_{1,2}$; 
\State $k \gets k+1$;	
\State $\B_k \gets \argmin_{\B \in \Re^{D \times c}, \, \B^\transpose \B = \bI_c} \, \, \sum_{\tilde{\x} \in \btX}  \big \|\B^\transpose \tilde{\x}\big \|_2^2 / \max\left\{\delta, \big \|\B_{k-1}^\transpose \tilde{\x}\big \|_2 \right\} $;
\State $\Delta \mathcal{J} \gets \big \| \btX^\transpose \B_{k-1} \big \|_{1,2}-\big \| \btX^\transpose \B_{k} \big \|_{1,2}$;  	
\EndWhile				
\State \Return $\B_{k}$;
\EndProcedure 				
\end{algorithmic} 
\end{algorithm}

Notice that in \eqref{eq:L12}, the $\ell_{1,2}$ matrix norm $\big \|\btX^\transpose \B \big\|_{1,2}$ of $\btX^\transpose \B$ is defined as the sum of the Euclidean norms of the rows of $\btX^\transpose \B$, and as such, favors a solution $\B$ that results in a 
matrix $\btX^\transpose \B$ that is row-wise sparse (notice that for $c=1$ \eqref{eq:L12} reduces precisely to the DPCP problem \eqref{eq:ell1}). In fact, \cite{Lerman:FCM15} consider exactly the same problem \eqref{eq:L12}, and proceed to relax it to a semi-definite convex program, which they solve via an \emph{Iteratively Reweighted Least-Squares (IRLS)} scheme \citep{Candes:JFAA08,Daubechies:CPAM10,Chartrand:ICASSP08}; while similar IRLS schemes appear in \cite{Zhang:JMLR14} and \cite{Lerman:IAI17}. Instead, we propose to solve \eqref{eq:L12} directly via IRLS (and not a convex relaxation of it as \cite{Lerman:FCM15}): Given a $D \times c$ orthonormal matrix $\B_{k-1}$, we define for each point $\tilde{\x}_j$ a weight
\begin{align}
w_{j,k} := \frac{1}{\max \left\{\delta,\big \| \B_{k-1}^\transpose \tilde{\x}_j \big \|_2 \right\}},
\end{align} where $\delta>0$ is a small constant that prevents division by zero. Then we obtain $\B_{k}$ as the solution to the quadratic problem
\begin{align}
\min_{\B \in \Re^{D \times c}}  \sum_{j=1}^L w_{j,k}\big\|\B^\transpose \tilde{\x}_j \big\|_2^2 \, \, \, \text{s.t.} \, \, \, \B^\transpose \B = \bI_c,
\end{align} 
which is readily seen to be the $c$ right singular vectors corresponding to the $c$ smallest singular values of the weighted data matrix $\boldsymbol{W}_k \btX^\transpose$, where $\boldsymbol{W}_{k}$ is a diagonal matrix with $\sqrt{w_{j,k}}$ at position $(j,j)$. We refer to the resulting Algorithm \ref{alg:DPCP-IRLS} as DPCP-IRLS; a study of its theoretical properties is deferred to future work.

\subsection{Denoised DPCP (DPCP-d)} \label{subsection:DPCP-d}
Clearly, problem \eqref{eq:ell1}  (and \eqref{eq:L12}) is tailored
for noise-free inliers, since, when the inliers $\bX$ are contaminated by noise,
the vector $\btX^\transpose \b$ is no longer sparse, even if $\b$ is a true normal
to the inlier subspace. As a consequence, it is natural to propose the following \emph{DPCP-denoised} (DPCP-d) problem
\begin{align}
\min_{\b,\y: \, ||\b||_2=1} \, \left[\tau  \, \big \|\y\big \|_1 + \frac{1}{2} \big \|\y - \btX^\transpose \b \big \|_2^2  \right] \label{eq:NIPS14},
\end{align} where now the vector variable $\y \in \Re^{N+M}$ is to be interpreted 
as a denoised vesion of the vector $\btX^\transpose \b$.
Interestingly, both problems \eqref{eq:ell1} and \eqref{eq:NIPS14} appear in \cite{Qu:NIPS14}, in the quite different context of dictionary learning, where the authors propose to solve \eqref{eq:NIPS14} via alternating minimization, in order to obtain an approximate solution to \eqref{eq:ell1}. Given $\b$, the optimal $\y$ is given by $ \cS_{\tau} \big(\btX^\transpose \b \big)$, where $\cS_{\tau}$ is the soft-thresholding operator applied element-wise on the vector $\btX^\transpose \b$. Given $\y$ the optimal $\b$ is a solution to the quadratically constrained least-squares problem
\begin{align}
\min_{\b \in \Re^D} \, \, \, \big\|\y - \btX^\transpose \b \big\|_2^2 \, \, \, \text{s.t.} \, \, \, \big \|\b \big \|_2 = 1. \label{eq:LeastSquaresConstrained}
\end{align} 
\begin{algorithm}[t!] \caption{Denoised Dual Principal Component Pursuit 
 }\label{alg:DPCP-d} \begin{algorithmic}[1] 
\Procedure{DPCP-d}{$\btX,\varepsilon,T_{\max},\delta,\tau$}		
\State Compute a Cholesky factorization $\bL \bL^\transpose  = \btX \btX^\transpose+\delta \bI_{D}$;			
\State $k \gets 0; \y_0 \gets \0; \J \gets 0; \Delta \mathcal{J} \gets \infty$; 
\State $\b_0 \gets \argmin_{\b \in \Re^{D}: \, \big \|\b\big \|_2=1} \, \, \, \big \|\btX^\transpose \b \big \|_2$;			
\While{$k < T_{\max}$ and $\Delta \J > \varepsilon \J$} 
\State $\J \gets \tau \big \| \y_k \big \|_1 + \frac{1}{2}  \big \|\y_k - \btX^\transpose \b_k \big \|_2^2 $	
\State $\y_{k+1} \gets \cS_{\tau} \left(\btX^\transpose \b_k \right)$;
\State $\b_{k+1} \gets$ solution of $\bL \bL^\transpose \bxi = \btX \y_{k+1}$ by backward/forward propagation;
\State $k \gets k+1$;
\State $\b_k \gets \b_k / \big \| \b_k \big \|_2$;
\State $\Delta \J \gets \J - \left( \tau \big \| \y_k \big \|_1 + \frac{1}{2}  \big \|\y_k - \btX^\transpose \b_k \big \|_2^2\right)$;	 		
\EndWhile				
\State \Return $(\y_k,\b_k)$;
\EndProcedure 				
\end{algorithmic} 
\end{algorithm}
In the context of \cite{Qu:NIPS14}, the coefficient matrix of the least-squares problem ($\btX^\transpose$ in our notation) has orthonormal columns. As a consequence, the solution to \eqref{eq:LeastSquaresConstrained} is obtained in closed form by projecting the solution of the unconstrained least-squares problem 
$\min_{\b \in \Re^D} \, \, \, ||\y - \btX^\transpose \b ||_2$ onto the unit sphere. However, in our context
the assumption that $\btX^\transpose$ has orthonormal columns is in principle violated, so that the optimal $\b$ is no longer available in closed form. Even though
using Lagrange multipliers one ends up with a polynomial equation for the Lagrange multiplier, it is known that computing the optimal value of the multiplier is a numerically challenging problem \citep{Elden:NM02,Golub:Numerische91,Gander:Numerische80}.  For this reason we leave exact approaches for solving  \eqref{eq:LeastSquaresConstrained} to future investigations, and we instead propose
to obtain a suboptimal $\b$ as \cite{Qu:NIPS14} do, i.e., by projecting onto the unit sphere the solution of the unconstrained least-squares problem.
The resulting Algorithm \ref{alg:DPCP-d} is very efficient, since the least-squares problems that appear in the various iterations have the same coefficient matrix $\btX \btX^\transpose$, a factorization of which can be precomputed.\footnote{The parameter $\delta$ in Algorithm \ref{alg:DPCP-d} is a small positive number, typically $10^{-6}$, which helps avoiding solving ill-conditioned linear systems.} Moreover, Algorithm \ref{alg:DPCP-d} can trivially be extended to compute multiple normal vectors, just as in Algorithm \ref{alg:DPCP-LP}.

\section{Experiments} \label{section:Experiments}
In this section we evaluate the proposed algorithms experimentally. In \S \ref{subsectionExperiments:Theory-Check} we investigate numerically the theoretical regime of success of recursion \eqref{eq:ConvexRelaxations} predicted by Theorems \ref{thm:DiscreteNonConvex} and \ref{thm:DiscreteConvexRelaxations}. We also show that even when these sufficient conditions are violated, \eqref{eq:ConvexRelaxations} can still converge to a normal vector to the subspace if initialized properly. Finally, in \S \ref{subsectionExperiments:ComparativeSynthetic} we compare DPCP variants with state-of-the-art robust PCA algorithms for the purpose of outlier detection using synthetic data, and similarly in \S \ref{subsectionExperiments:Real} using real images.

\subsection{Numerical evaluation of the theoretical conditions of Theorems \ref{thm:DiscreteNonConvex} and \ref{thm:DiscreteConvexRelaxations}} \label{subsectionExperiments:Theory-Check}

We begin with a numerical evaluation of the theoretical condition \eqref{eq:gammaUP} of Theorem \ref{thm:DiscreteNonConvex}, under which every global minimizer of the DPCP problem \eqref{eq:ell1} is orthogonal to the inlier subspace $\cS$. We also evaluate the initial minimal angle $\phi_0^*$ from $\cS$ given in \eqref{eq:phiLB} of Theorem \ref{thm:DiscreteConvexRelaxations}, which together with \eqref{eq:gammaUP} guarantee the convergence of the linear programming recursion \eqref{eq:ConvexRelaxations} to an element of $\cS^\perp$. 
As explained in the discussion of 
Theorems \ref{thm:DiscreteNonConvex} and \ref{thm:DiscreteConvexRelaxations}, for any fixed outlier ratio, condition \eqref{eq:gammaUP} will eventually be satisfied and also the angle $\phi_0^*$ will become arbitrarily small regardless of the subspace relative dimension $d/D$, provided that $N$ is sufficiently large and that both inliers and outliers are uniformly distributed. Hence, we check whether  \eqref{eq:gammaUP} is true and also plot $\phi_0^*$ as we vary $N$ for uniformly distributed inliers and outliers. Towards that end, we fix the ambient dimension as $D=30$ and randomly sample a subspace ${\cS}$ of varying dimension $d=[5:5:25 \, \, 29]$ so that the relative subspace dimension $d/D$ varies as $[5/30:5/30:25/30 \, \, 29/30]$. We sample $N$ inliers uniformly at random from ${\cS}\cap \Sp^{D-1}$ for different values $N=500,2000,7000$. For each value of $N$ we also sample $M$ outliers uniformly at random from $\Sp^{D-1}$ so that the percentage of outliers varies as $R:=M/(N+M)=[0.1:0.1:0.7]$. For each dataset instance as above, we estimate the parameters $\eX, \eO,  \mathcal{R}_{\bO,\bX}$ appearing in \eqref{eq:gammaUP} and \eqref{eq:phiLB} by Monte-Carlo simulation. 

The top row of Fig. \ref{figure:DPCP-theory} shows whether condition \eqref{eq:gammaUP} is true (white) or not (black) as we vary $N$. Notice that for $N=500$, Fig. \ref{figure:DPCP_theory_N500} shows a poor success regime. However, as we increase $N$ to $2000$, Fig. \ref{figure:DPCP_theory_N2000} shows that the success regime improves dramatically. Finally, as expected from our earlier theoretical arguments, for sufficiently large $N$, in particular for $N=7000$, Fig. \ref{figure:DPCP_theory_N7000} shows that the sufficient condition \eqref{eq:gammaUP} is satisfied regardless of outlier ratio or subspace relative dimension. 
Similarly, notice how the angle $\phi_0^*$, plotted in the bottom row of Fig. \ref{figure:DPCP-theory} (black for $0^\circ$ white for $90^\circ$), uniformly decreases as we increase $N$ across all outlier ratios and relative dimensions.

\begin{figure}[t!]
	\centering
	\subfigure[Check \eqref{eq:gammaUP}, $N=500$]{\label{figure:DPCP_theory_N500}\includegraphics[width=0.25\linewidth]{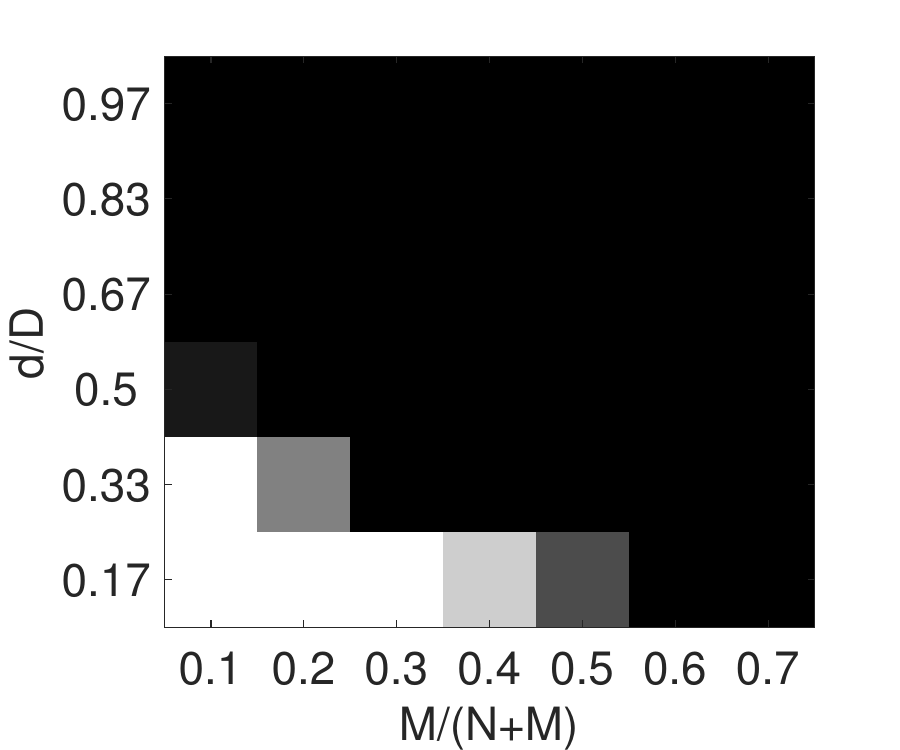}}
\subfigure[Check \eqref{eq:gammaUP}, $N=2000$]{\label{figure:DPCP_theory_N2000}\includegraphics[width=0.25\linewidth]{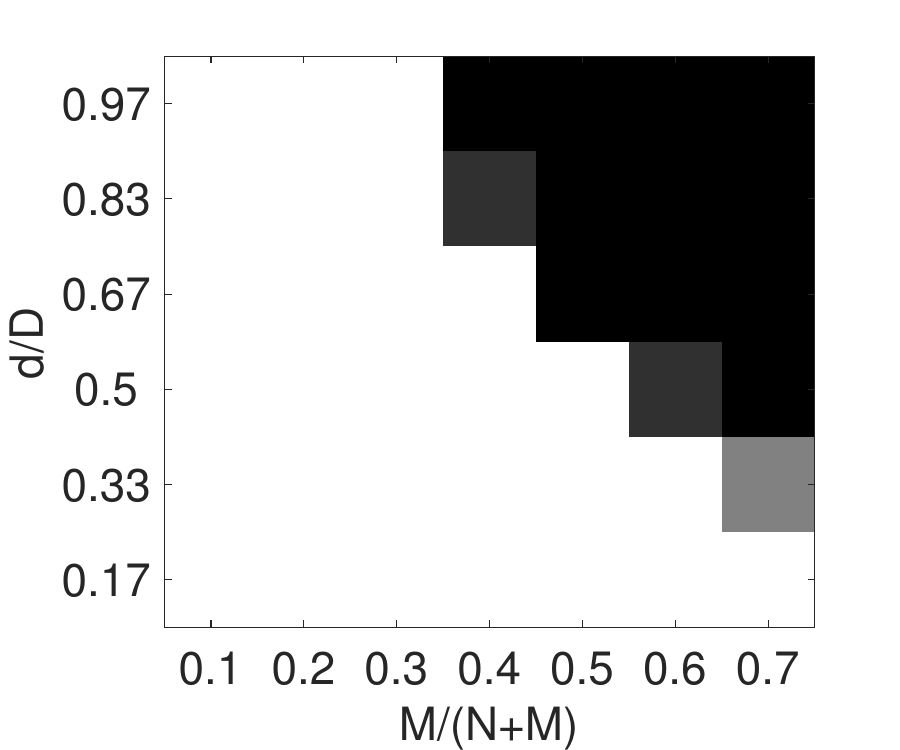}}
\subfigure[Check \eqref{eq:gammaUP}, $N=7000$]{\label{figure:DPCP_theory_N7000}\includegraphics[width=0.25\linewidth]{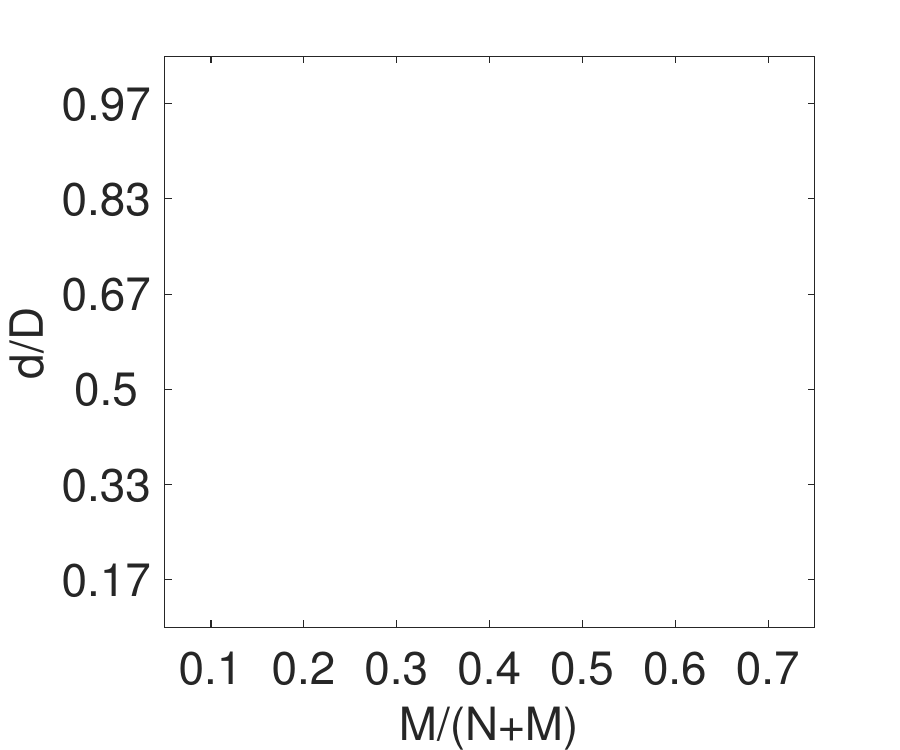}} \\
\subfigure[$\phi_0^*$ when $N=500$]{\label{figure:DPCP_theory_angle_N500}\includegraphics[width=0.25\linewidth]{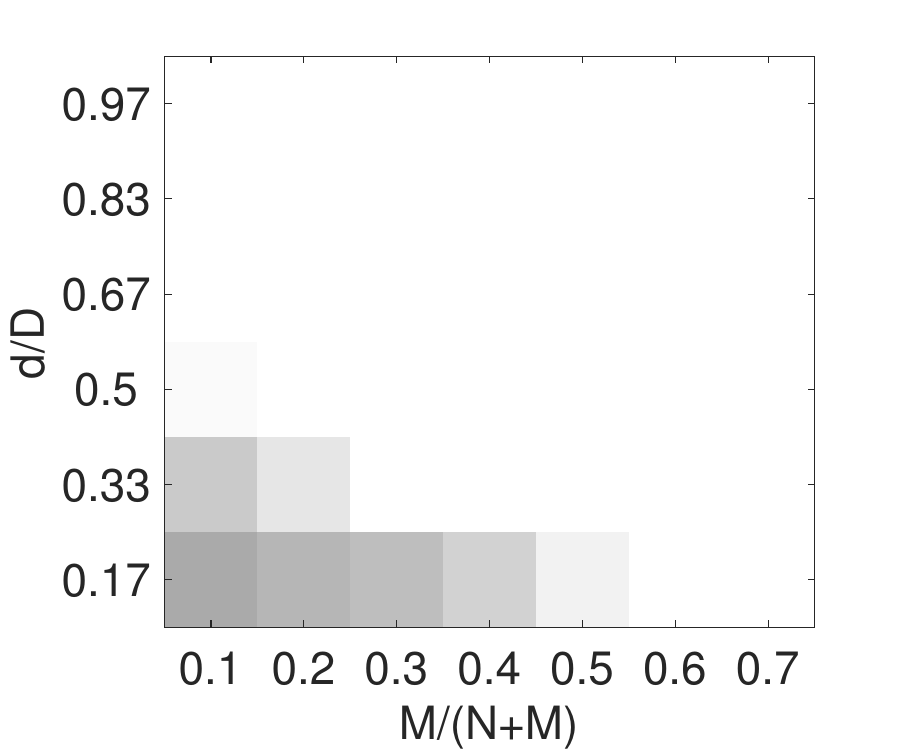}}
\subfigure[$\phi_0^*$ when $N=2000$]{\label{figure:DPCP_theory_angle_N2000}\includegraphics[width=0.25\linewidth]{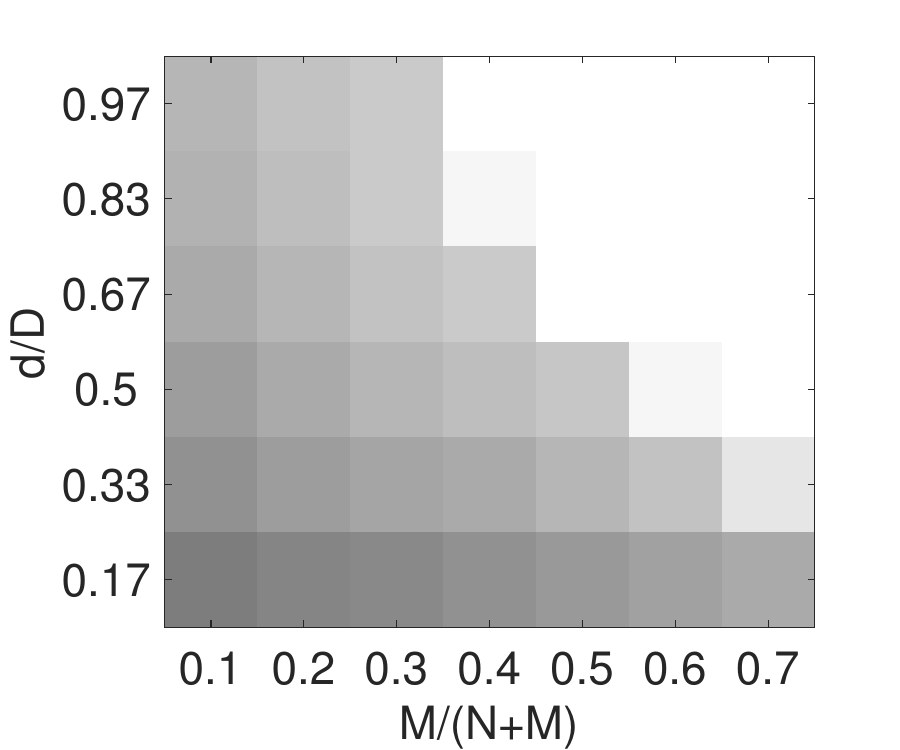}}
\subfigure[$\phi_0^*$ when $N=7000$]{\label{figure:DPCP_theory_angle_N7000}\includegraphics[width=0.25\linewidth]{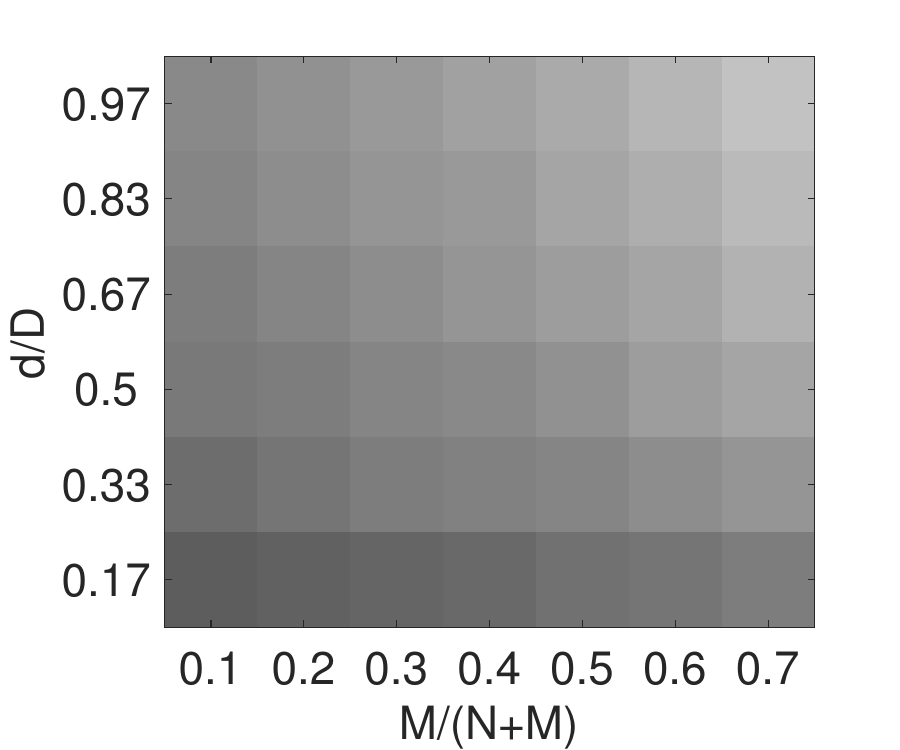}}		
\caption{ Figs. \ref{figure:DPCP_theory_N500}-\ref{figure:DPCP_theory_N7000} check whether the condition \eqref{eq:gammaUP} is satisfied (white) or not (black) for a fixed number $N$ of inliers while varying the outlier ratio $M/(N+M)$ and the subspace relative dimension $d/D$. Figs. \ref{figure:DPCP_theory_angle_N500}-\ref{figure:DPCP_theory_angle_N7000} plot the minimum initial angle $\phi_0^*$ needed for convergence of the recursion of linear programs \eqref{eq:ConvexRelaxations} as per Theorem \ref{thm:DiscreteConvexRelaxations} ($0^\circ$ corresponds to black and $90^\circ$ corresponds to white). Results are averaged over $10$ independent trials.}
\label{figure:DPCP-theory}	
\end{figure}

Next, we show that the recursion \eqref{eq:ConvexRelaxations}, if initialized properly, is in fact able to converge in just a few iterations to a vector normal to the inlier subspace, even when the sufficient conditions of Theorem \ref{thm:DiscreteConvexRelaxations} are not satisfied. Towards that end, we maintain the same experimental setting as above using $N=500$ and run \eqref{eq:ConvexRelaxations} with a maximal number of iterations set to $T_{\max}=10$ and a convergence accuracy set to $10^{-3}$. Fig. \ref{figure:DPCP_repeated_theory_N500} is a replicate of Fig. \ref{figure:DPCP_theory_N500} and serves as a reminder that $N=500$ results in a limited success regime as predicted by the theory; in particular the sufficient condition \eqref{eq:gammaUP} is satisfied only for a small outlier ratio or small subspace relative dimensions. Even so, Fig. \ref{figure:PhiStar_SVD_T10} shows that, when the recursion \eqref{eq:ConvexRelaxations} is initialized using $\hat{\bn}_0$ as the left singular vector of $\btX$ corresponding to the smallest singular value, then \eqref{eq:ConvexRelaxations} converges in at most $10$ iterations to a vector $\hat{\bn}^*$ whose angle $\phi^*$ from the subspace is precisely $90^\circ$. This suggests that the sufficient condition \eqref{eq:gammaUP} is much stronger than necessary, leaving room for future theoretical improvements. On the other hand, Fig. \ref{figure:PhiStar_random_T10} shows that when $\hat{\bn}_0$ is initialized uniformly at random, the recursion \eqref{eq:ConvexRelaxations} does not always converge to a normal vector, particularly for high outlier ratios and relative dimensions. This reveals that initializing \eqref{eq:ConvexRelaxations} from the SVD of the data is indeed a good strategy, which is further supported by Fig. \ref{figure:Phi0_SVD}, which plots the angle of the initialization from the subspace (contrast this to Fig. \ref{figure:Phi0_random}, which shows the angle of a random initialization from the subspace).

\begin{figure}[t!]
\centering
\subfigure[Check if \eqref{eq:gammaUP} is true]{\label{figure:DPCP_repeated_theory_N500}\includegraphics[width=0.25\linewidth]{theory_N500}}
\subfigure[$\phi^*$ ($\hat{\bn}_0$ from SVD)]{\label{figure:PhiStar_SVD_T10}\includegraphics[width=0.25\linewidth]{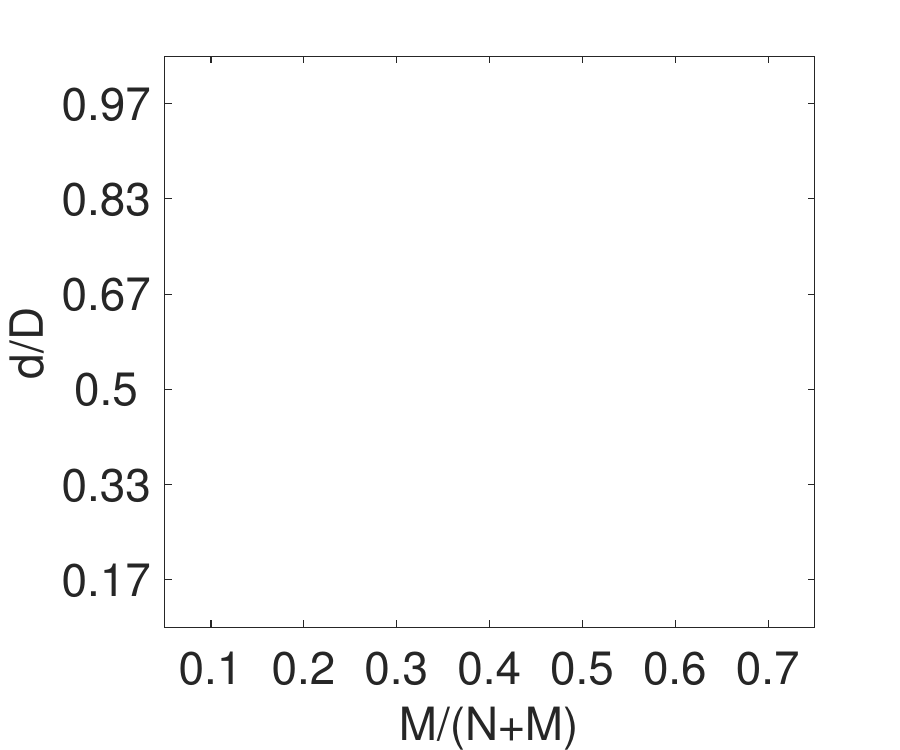}}
\subfigure[$\phi^*$ ($\hat{\bn}_0$ random)]{\label{figure:PhiStar_random_T10}\includegraphics[width=0.25\linewidth]{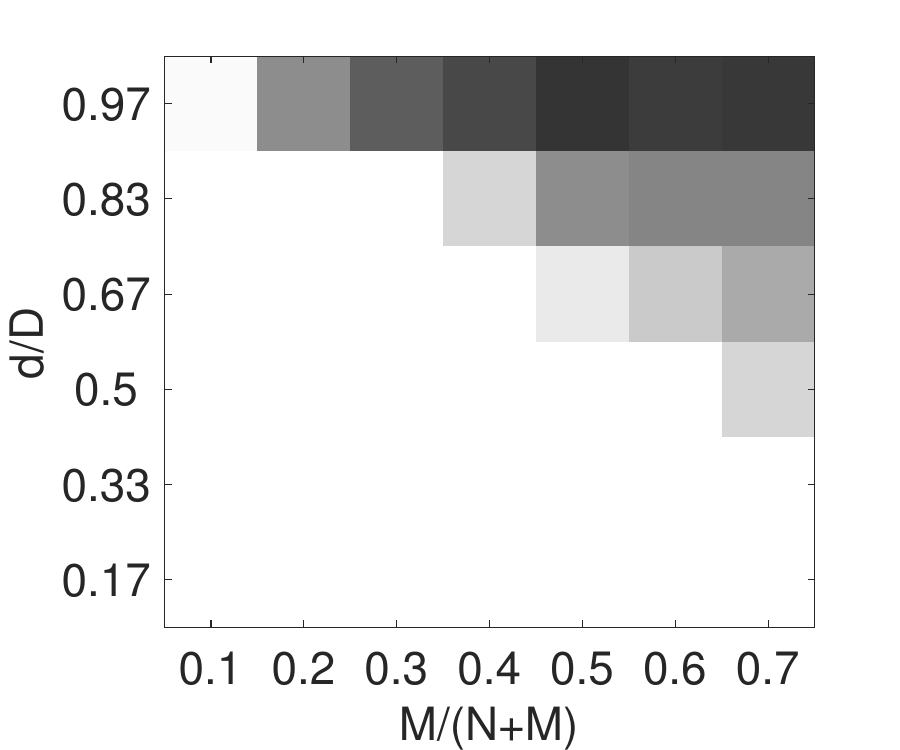}} \\
\subfigure[$\phi_0^*$ of \eqref{eq:phiLB}]{\label{figure:DPCP_repeated_theory_angle_N500}\includegraphics[width=0.25\linewidth]{theory_angle_N500}}
\subfigure[$\phi_0$ ($\hat{\bn}_0$ from SVD)]{\label{figure:Phi0_SVD}\includegraphics[width=0.25\linewidth]{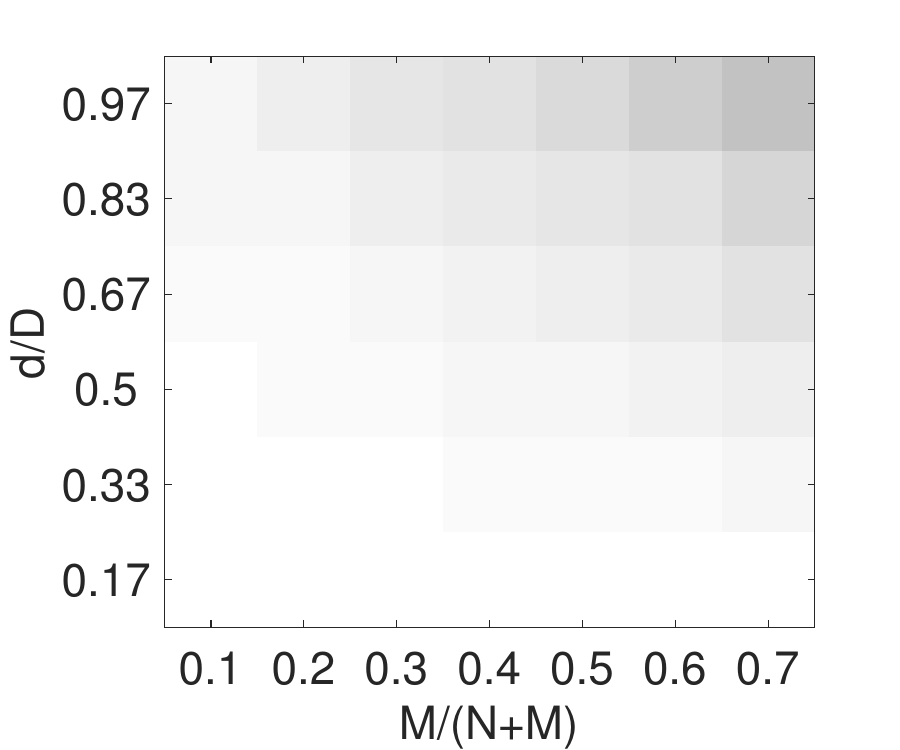}}
\subfigure[$\phi_0$ ($\hat{\bn}_0$ random)]{\label{figure:Phi0_random}\includegraphics[width=0.25\linewidth]{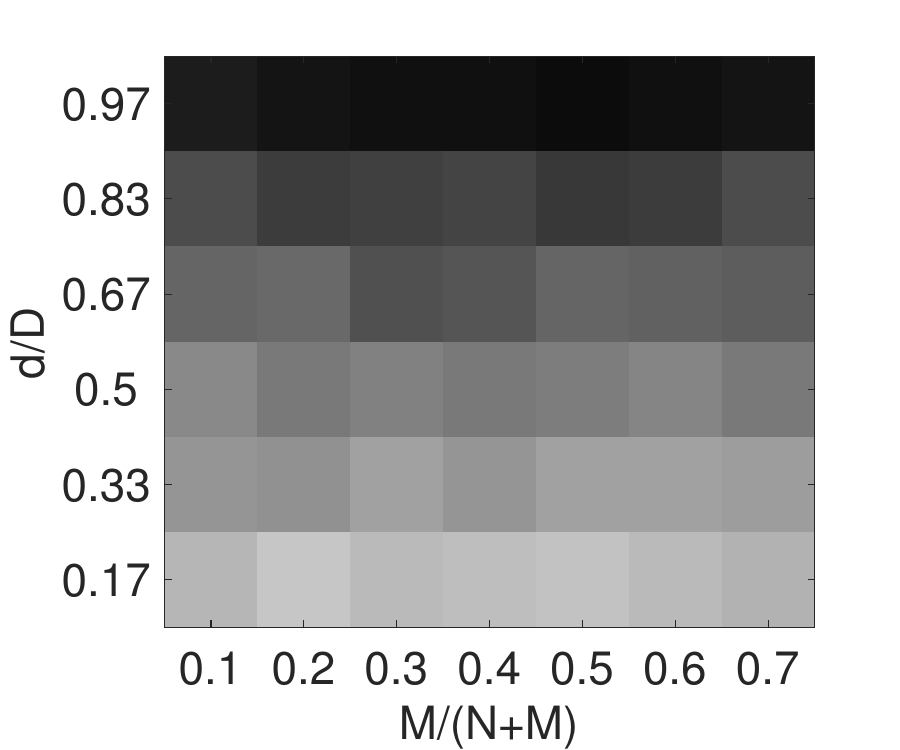}}		
\caption{Convergence of recursion \eqref{eq:ConvexRelaxations} in a regime of limited theoretical guarantees ($N=500$) as concluded from Fig. \ref{figure:DPCP-theory}. The number of inliers is fixed at $N=500$. Fig. \ref{figure:DPCP_repeated_theory_N500} plots whether the sufficient condition \eqref{eq:gammaUP} is satisfied (white) or not (black) for varying outlier ratios $M/(N+M)$ and relative dimensions $d/D$. Figs.
\ref{figure:PhiStar_SVD_T10}-\ref{figure:PhiStar_random_T10} plot the angle $\phi^*$ ($0^\circ$ corresponds to black and $90^\circ$ corresponds to white) from the inlier subspace of the vector $\hat{\bn}^*$ that recursion \eqref{eq:ConvexRelaxations} converges to, when $\hat{\bn}_0$ is initialized from the SVD of the data or uniformly at random, respectively. 
Fig. \ref{figure:DPCP_repeated_theory_angle_N500} plots the minimum angle $\phi^*_0$ needed for the convergence of \eqref{eq:ConvexRelaxations} to a normal vector in the inlier subspace as per Theorem \ref{thm:DiscreteConvexRelaxations}, while Figs. \ref{figure:Phi0_SVD}-\ref{figure:Phi0_random} plot the angle $\phi_0$ of $\hat{\bn}_0$ from the subspace, when it is initialized from the SVD of the data or uniformly at random, respectively. 
The results are averages over $10$ independent trials.}
\label{figure:DPCP-single}	
\end{figure}

\subsection{Comparative analysis using synthetic data} \label{subsectionExperiments:ComparativeSynthetic}

In this section we use the same synthetic experimental set-up as in \S \ref{subsectionExperiments:Theory-Check} (with $N=500$) to demonstrate the behavior of several methods relative to each other under uniform conditions, in the context of outlier rejection in single subspace learning. In particular, we test DPCP-LP (Algorithm \ref{alg:DPCP-LP}), DPCP-IRLS (Algorithm \ref{alg:DPCP-IRLS}), DPCP-d (Algorithm \ref{alg:DPCP-d}), RANSAC \citep{RANSAC}, SE-RPCA \citep{Soltanolkotabi:AS12}, $\ell_{21}$-RPCA \citep{Xu:TIT12}, the IRLS version of REAPER \citep{Lerman:FCM15}, as well as Coherence Pursuit (CoP) \citep{Rahmani:arXiv17}; see \S \ref{section:RelatedWork} for details on these last five existing methods. 
	
For the methods that require an estimate of the subspace dimension $d$, such as RANSAC, REAPER, CoP, and all DPCP variants, we provide as input the true subspace dimension. The convergence accuracy of all methods is set to $10^{-3}$. For REAPER we set the regularization parameter as $\delta=10^{-6}$ and the maximum number of iterations equal to $100$. For DPCP-d we set $\tau = 1/\sqrt{N+M}$ as suggested in \cite{Qu:NIPS14} and the maximum number of iterations to $1000$. For RANSAC we set its thresholding parameter to $10^{-3}$, and for fairness, we do not let it terminate earlier than the running time of DPCP-LP, unless the theoretically required number of iterations for a success probability $0.99$ is reached (here we are using the ground truth outlier ratio). 
Both SE-RPCA and $\ell_{21}$-RPCA are implemented with ADMM, with augmented Lagrange parameters $1000$ and $100$ respectively. For $\ell_{21}$-RPCA $\lambda$ is set to $3/(7\sqrt{M})$, as suggested in \cite{Xu:TIT12}. DPCP variants are initialized via the SVD of the data as in Algorithm \ref{alg:DPCP-LP}. CoP is implemented using the code provided by its authors, and selects $3d$ points upon classic PCA gives the subspace estimate. Finally, the linear programs in DPCP-LP are solved via the generic LP solver Gurobi \citep{gurobi}, while the maximum number of iterations for DPCP-LP is set to $T_{\max}=10$.

\paragraph{Absence of Noise} We first investigate the potential of each of the above methods to perfectly distinguish outliers from inliers in the absence of noise.\footnote{We do not include DPCP-d for this experiment, since it only approximately solves the DPCP optimization problem, and hence it can not be expected to perfectly separate inliers from outliers, even when there is no noise (we have confirmed this experimentally).} Note that each method returns a \emph{signal} $\ba \in \Re_+^{N+M}$, which can be thresholded for the purpose of declaring outliers and inliers. For SE-RPCA, $\ba$ is the $\ell_1$-norm of the columns of the coefficient matrix $\bC$, while for $\ell_{21}$-RPCA $\ba$ is the $\ell_2$-norm of the columns of $\bE$. Since RANSAC, REAPER, CoP, and DPCP variants directly return subspace models, for these methods $\ba$ is simply the distances of all points to the estimated subspace.

\begin{figure}[t!]
	\centering
	\subfigure[RANSAC]{\label{figure:RANSACseparation}\includegraphics[width=0.25\linewidth]{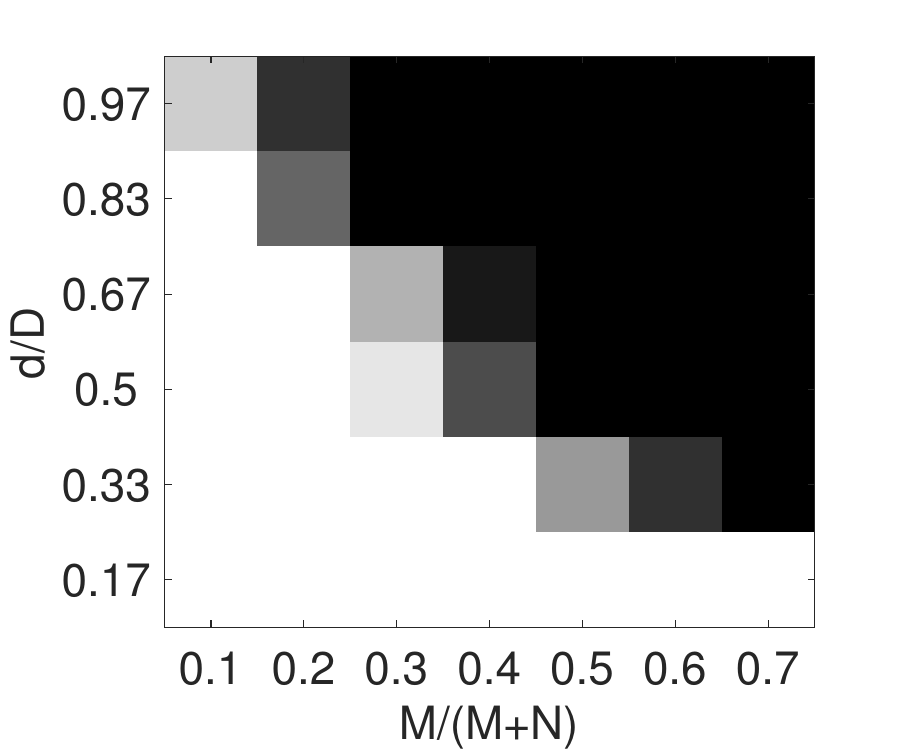}}
		\subfigure[SE-RPCA]{\label{figure:SEseparation}\includegraphics[width=0.25\linewidth]{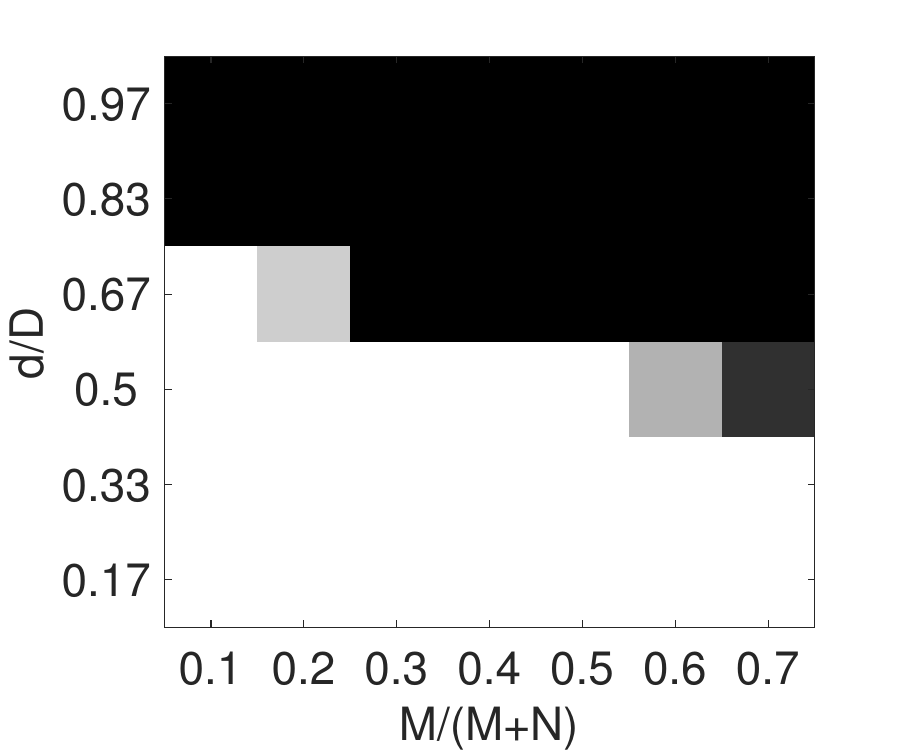}} 	
	\subfigure[$\ell_{21}$-RPCA]{\label{figure:L21separation}\includegraphics[width=0.25\linewidth]{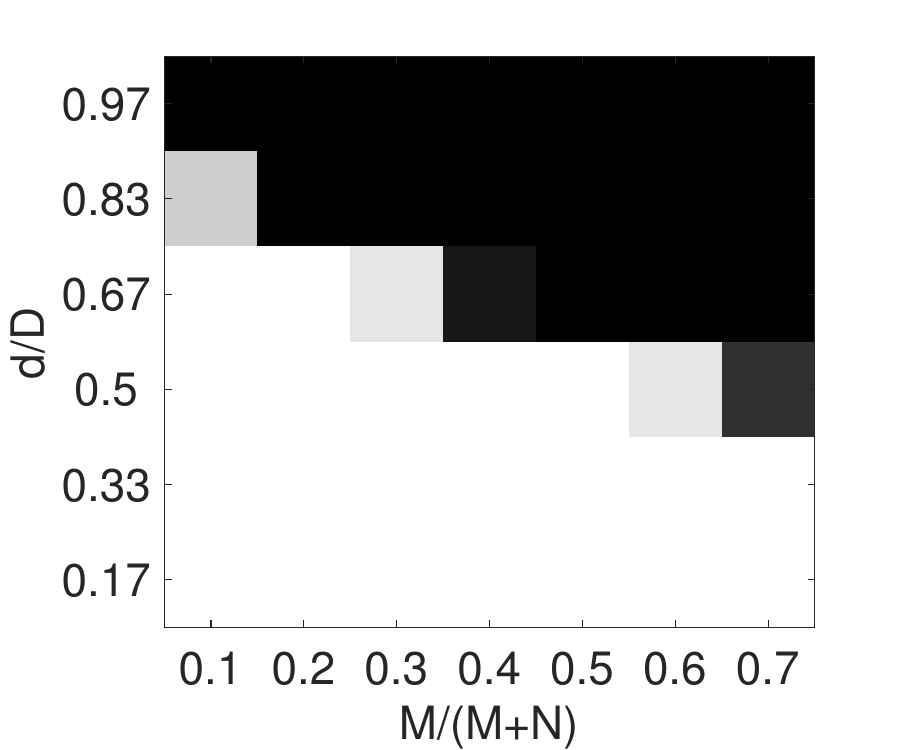}} 	\\
	\subfigure[REAPER]{\label{figure:REAPERseparation}\includegraphics[width=0.25\linewidth]{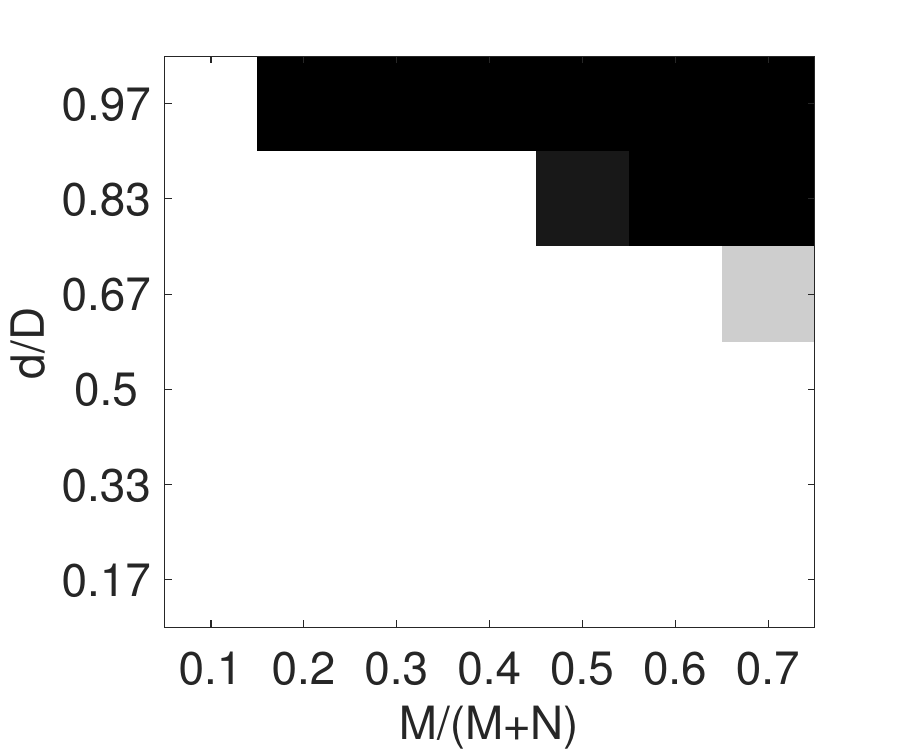}}
	\subfigure[CoP]{\label{figure:CPseparation}\includegraphics[width=0.25\linewidth]{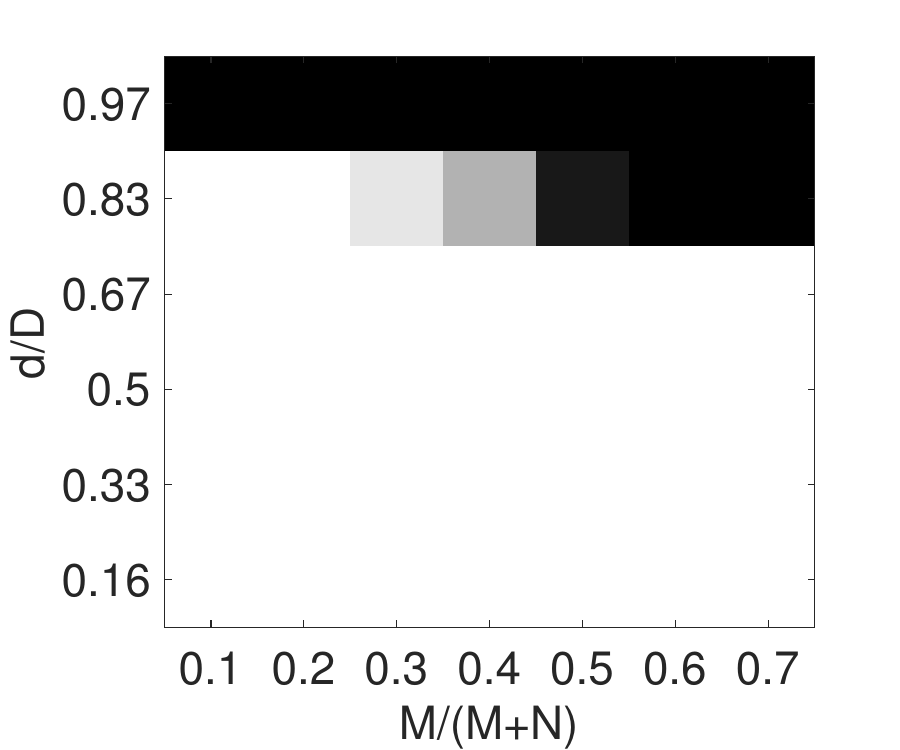}}\\
	\subfigure[DPCP-LP]{\label{figure:DPCP-LPseparation}\includegraphics[width=0.25\linewidth]{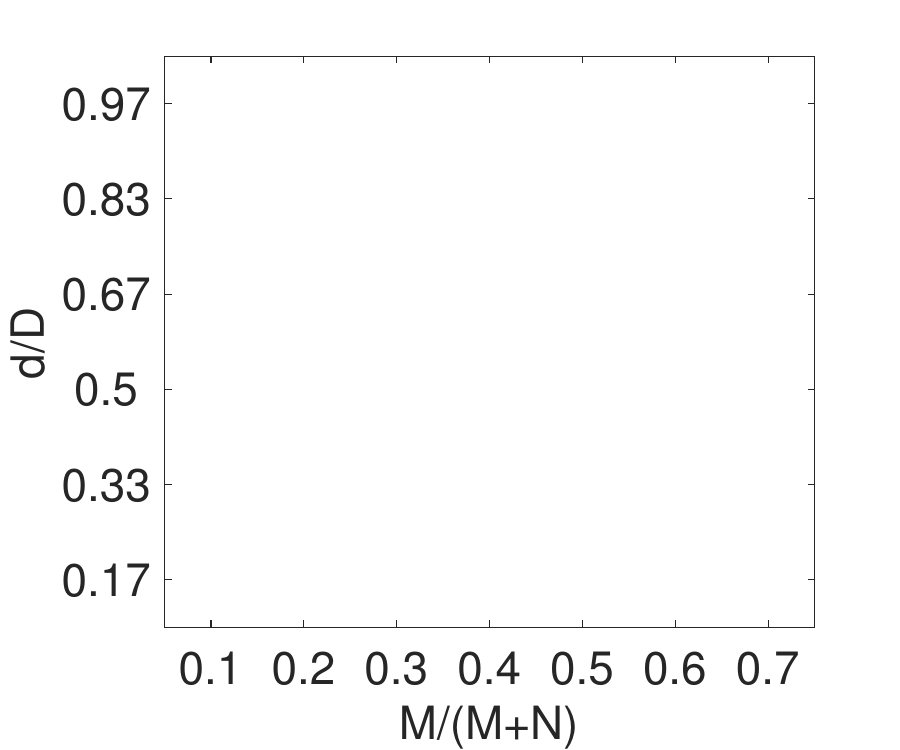}}		\subfigure[DPCP-IRLS]{\label{figure:L1_IRLS_separation}\includegraphics[width=0.25\linewidth]{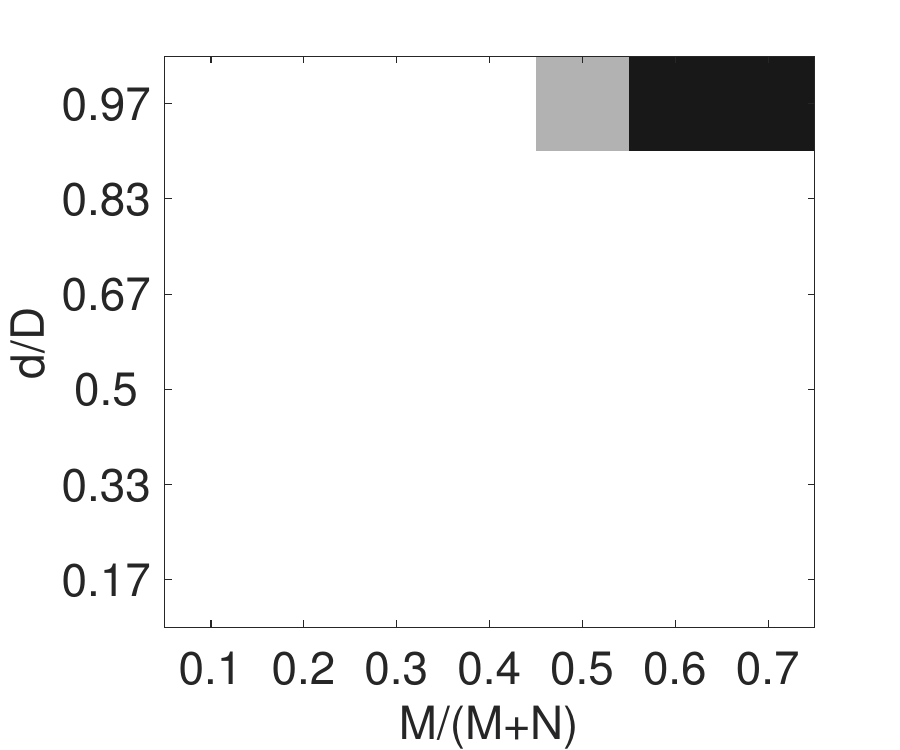}}	 	
	\caption{Outlier/Inlier separation in the absence of noise over $10$ independent trials. The horizontal axis is the outlier ratio defined as $M/(N+M)$, where $M$ is the number of outliers and $N$ is the number of inliers. The vertical axis is the relative inlier subspace dimension $d/D$; the dimension of the ambient space is $D=30$. Success (white) is declared by the existence of a threshold that, when applied to the output of each method, perfectly separates inliers from outliers.}\label{figure:separation}
\end{figure} 

\begin{figure}[t!]
\centering
\subfigure[$\sigma=0.05, d/D= 25/30$]{\label{ROC_d25_sigma05}\includegraphics[width=0.25\linewidth]{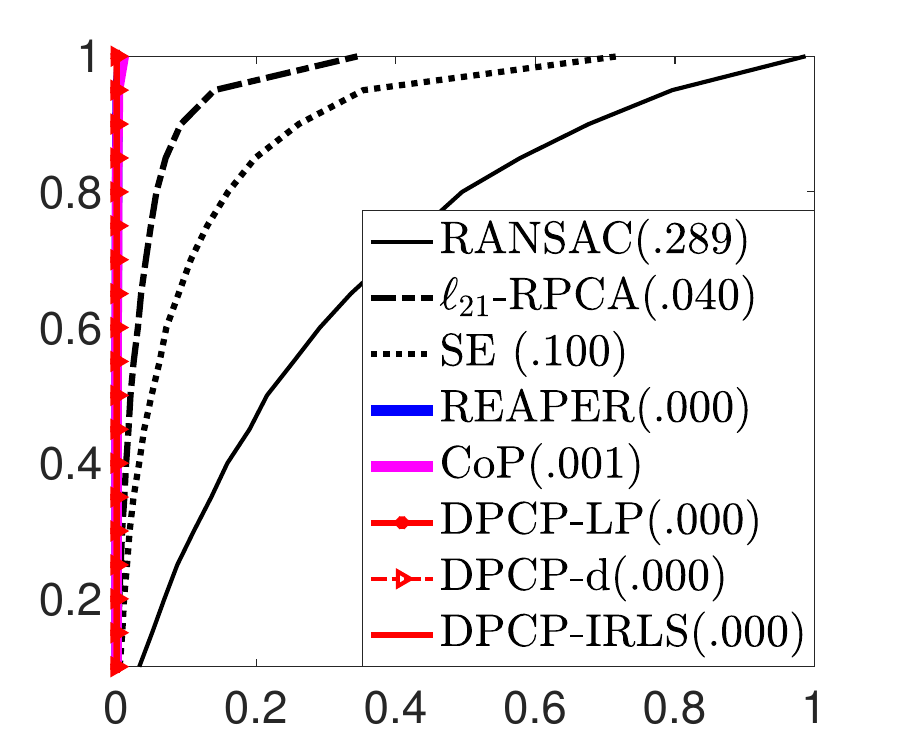}}
\subfigure[$\sigma=0.10, d/D= 25/30$]{\label{ROC_d25_sigma10}\includegraphics[width=0.25\linewidth]{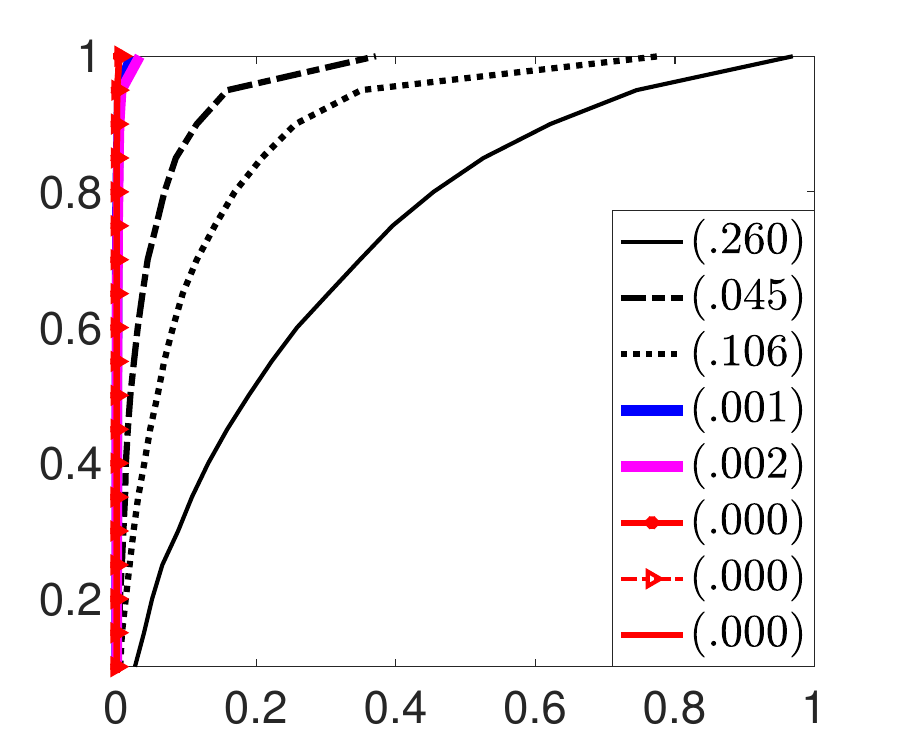}} 
\subfigure[$\sigma=0.20, d/D= 25/30$]{\label{ROC_d25_sigma20}\includegraphics[width=0.25\linewidth]{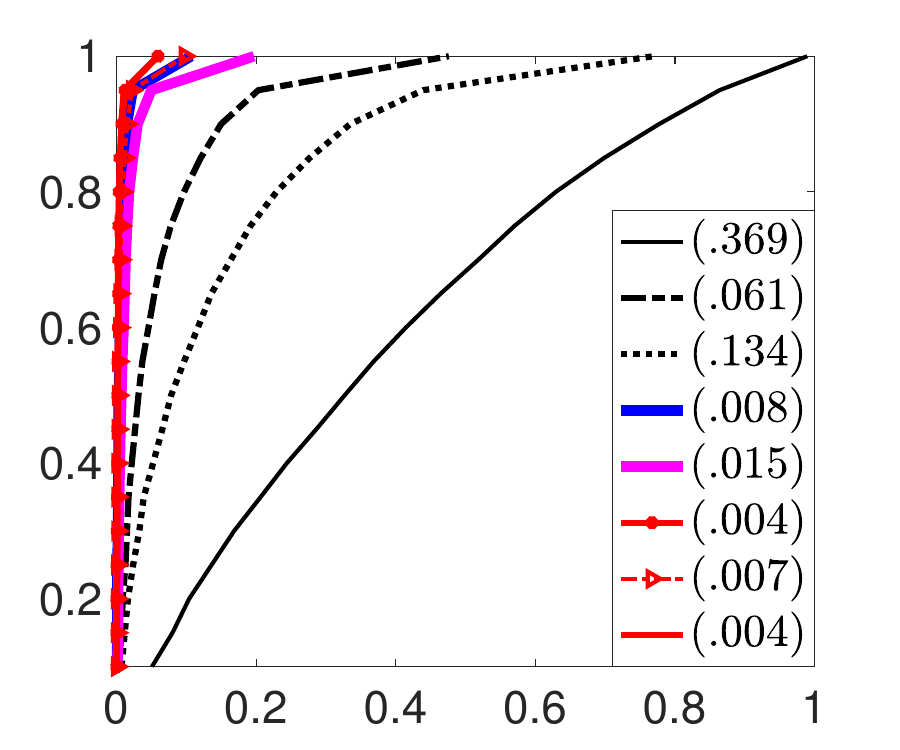}}\\
\subfigure[$\sigma=0.05, d/D= 29/30$]{\label{ROC_d29_sigma05}\includegraphics[width=0.25\linewidth]{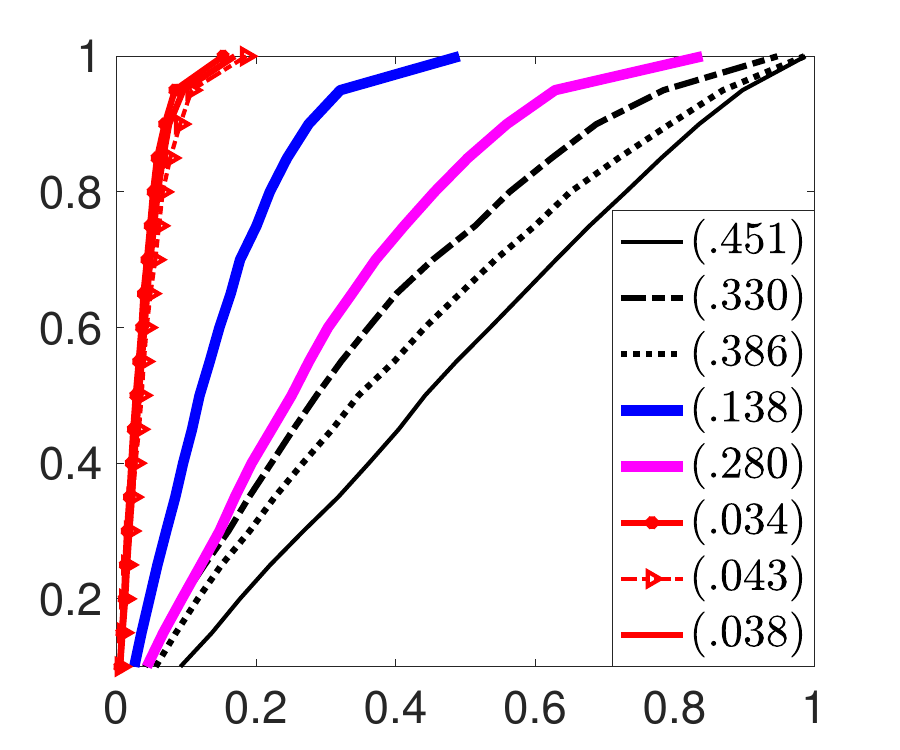}}
\subfigure[$\sigma=0.10, d/D= 29/30$]{\label{ROC_d29_sigma10}\includegraphics[width=0.25\linewidth]{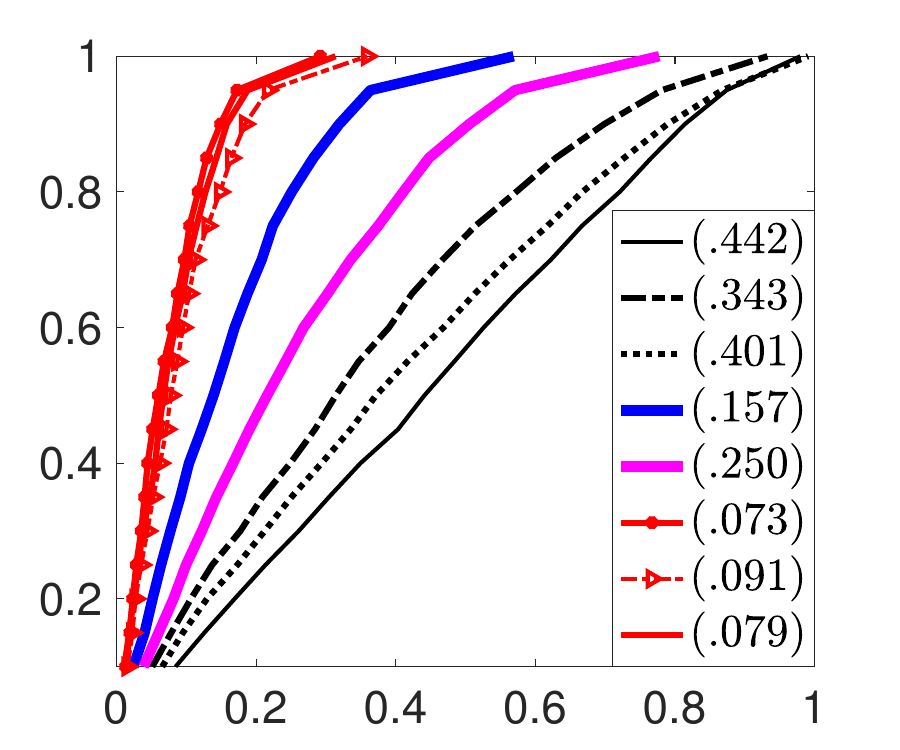}} 	
\subfigure[$\sigma=0.20, d/D= 29/30$]{\label{ROC_d29_sigma20}\includegraphics[width=0.25\linewidth]{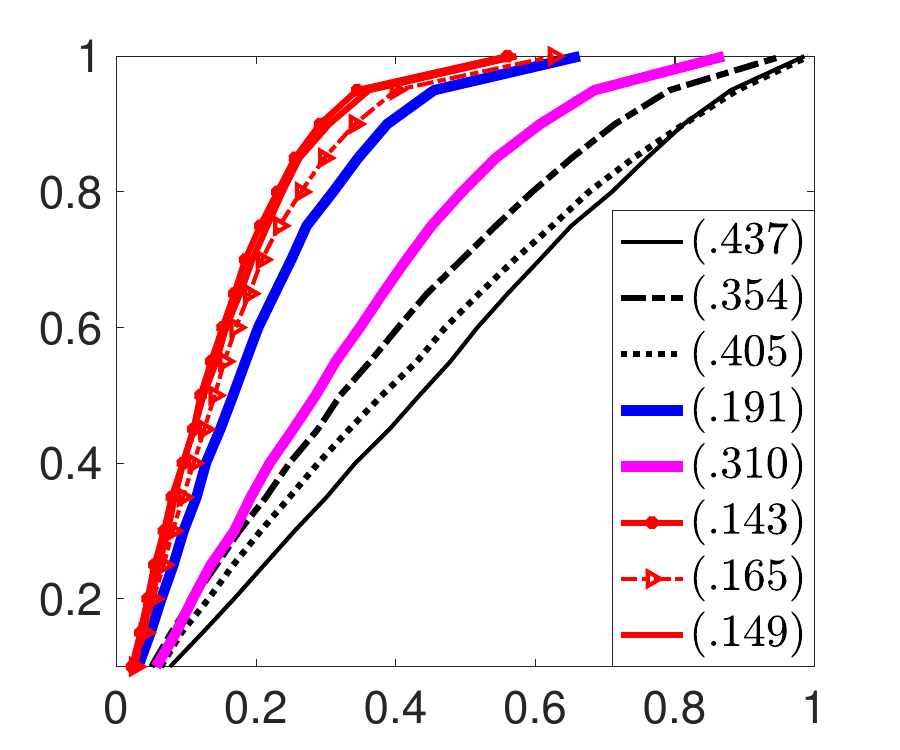}} \\
\subfigure[$\sigma=0.05, d/D= 25/30$]{\label{distance2S_d25}\includegraphics[width=0.25\linewidth]{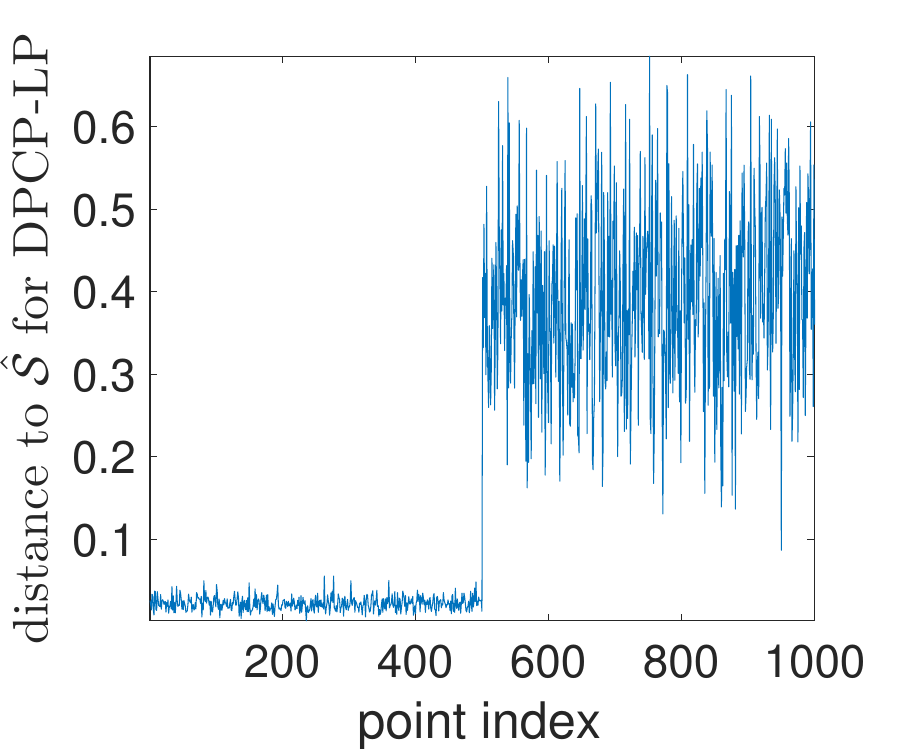}}
\subfigure[$ \sigma=0.05, d/D= 29/30$]{\label{distance2S_d29}\includegraphics[width=0.25\linewidth]{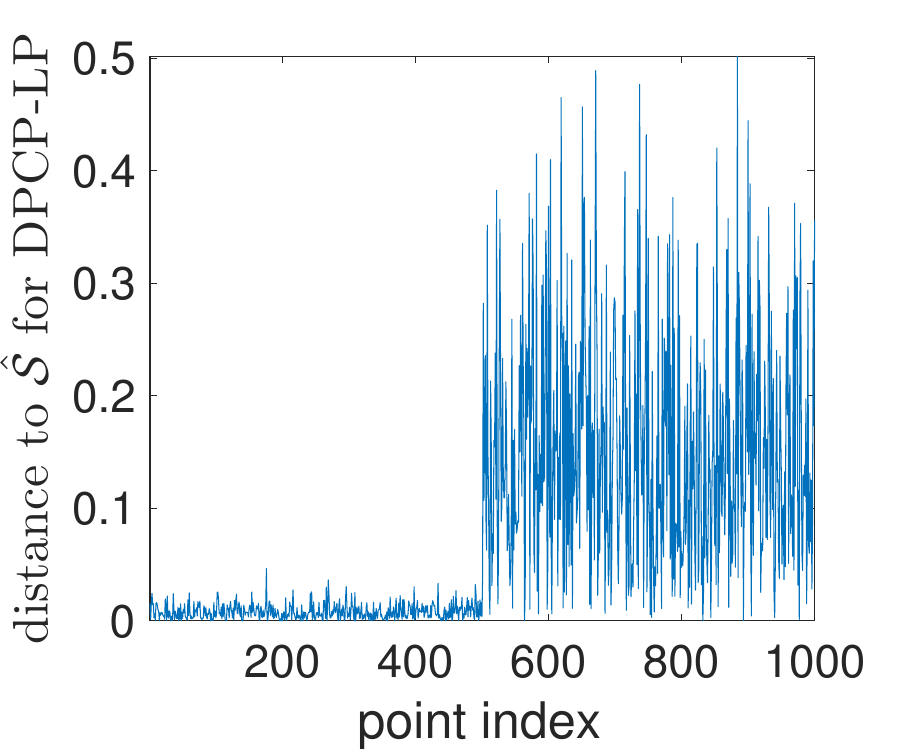}}
\caption{Figs. \ref{ROC_d25_sigma05}-\ref{ROC_d29_sigma20} show ROC curves for varying noise standard deviation $\sigma=0.05, 0.1, 0.2$, and subspace relative dimension $d/D = 25/30, 29/30$ (number of inliers is $N=500$ and outlier ratio is $M/(N+M) = 0.5$). The horizontal axis is False Positives ratio and the vertical axis is True Positives ratio. The number associated with each curve is the area above the curve; smaller numbers reflect more accurate performance. Figs. \ref{distance2S_d25}-\ref{distance2S_d29} show the distance of the noisy points to the subspace estimated by DPCP-LP for noise $\sigma=0.05$ for both relative dimensions under consideration.}
\label{figure:ROC_synthetic}
\end{figure}

\begin{table}
\centering
\caption{Mean running time of each method in seconds over $10$ independent trials for the experimental setting of \S \ref{subsectionExperiments:ComparativeSynthetic}.
We report only the extreme regimes corresponding to $d/D=5/30, 29/30$ and $M/(N+M) = 0.1, 0.7$.
The experiment is run in MATLAB on a standard Macbook-Pro with a dual core 2.5GHz Processor and a total of $4$GB Cache memory.}
\label{table:RunningTimes}
\ra{0.7}
\begin{tabular}{@{}l@{\, \, \, \,}c@{\, \, \, \,}c@{\, \, \, \,}c@{\, \, \, \,}c@{\, \, \, \,}c@{\, \, \, \,}c@{\, \, \, \,}c@{\, \, \, \,}c}\toprule[1pt]  $d/D:M/(N+M)$ &  $5/30:0.1$ & $29/30: 0.1$ & $5/30:0.7$ &  $29/30:0.7$  \\ 
\midrule[0.5pt]
RANSAC & $0.097$ & $0.410$ & $23.31$ & $2.83$   \\
SE-RPCA & $4.485$ & $4.519$ & $58.94$ & $79.07$  \\
$\ell_{21}$-RPCA & $0.048$ & $0.014$ & $0.185$ & $0.180$ \\
REAPER & $0.050$ & $0.058$ & $0.153$ & $0.042$  \\
CoP & $0.014$ & $0.014$ & $0.062$ & $0.061$  \\
DPCP-LP & $16.87$ & $0.407$ & $95.35$ & $2.822$  \\
DPCP-IRLS & $0.046$ & $0.038$ & $0.121$ & $0.415$  \\
\bottomrule[1pt]
\end{tabular}
\end{table}

In Fig. \ref{figure:separation} we depict success (white) versus failure (black), where success is interpreted as the existence of a threshold on $\ba$ that perfectly separates outliers and inliers. First observe that, as expected, RANSAC succeeds when there are very few outliers ($10\%$) regardless of the inlier relative dimension $d/D$, since in such a case the probability of sampling outlier-free points is high. Similarly, RANSAC succeeds when $d/D$ is small regardless of the outlier ratio, since in that case one needs only sample $d$ points, and for a sufficient budget (say, the running time of DPCP-LP), the probability of one of these samples being outlier-free is again high. Moving on, and again as expected, both SE-RPCA and $\ell_{21}$-RPCA succeed only for low to medium relative dimensions, since both methods are meant to exploit the low-rank structure of the inlier data, and thus fail when such a structure does not exist. Remarkably, even though CoP is a low-rank method in spirit, it performs surprisingly better than its low-rank alternatives SE-RPCA and $\ell_{21}$-RPCA, giving perfect inlier/outlier separation regardless of the outlier ratio for relative dimensions $d/D\le 2/3$. Further improvement is achieved by REAPER, which succeeds for as many as $40\%$ outliers and as high a relative dimension as $25/30 \approx 0.83$. Yet, REAPER fails in the challenging case $d/D=29/30$ as soon as there are more than $20\%$ outliers. Remarkably, DPCP-LP allows for perfect outlier rejection across all outlier ratios and all relative dimensions, thus clearly improving the state-of-art in the high relative dimension regime. Moreover, DPCP-IRLS does almost as well as DPCP-LP thus being the second best method, except that it only fails in the hardest of regimes, i.e., for relative dimension $29/30 \approx 0.97$ and for more than $50\%$ outliers. 

Finally, it is important to comment on the running time of the methods. As Table \ref{table:RunningTimes} shows, DPCP-LP is admittedly the slowest among the methods, particularly for low relative dimensions, since in that case many dual principal components need to be computed. Indeed, for $10\%$ outliers and $d/D = 5/30$ DPCP-LP computes $D-d = 25$ dual components and thus it takes about $17$ seconds, as opposed to $0.4$ seconds for the same amount of outliers but $d/D = 29/30$, since a single dual component is computed in this latter case. Similarly, for $70\%$ outliers DPCP-LP takes about $95$ seconds when $d/D=5/30$ as opposed to about $3$ seconds for $d/D = 29/30$. On the other hand, DPCP-IRLS is one order of magnitude faster than DPCP-LP and comparable to $\ell_{21}$-RPCA and REAPER, which overall are the second fastest methods, with CoP being the fastest of all.

\paragraph{Presence of Noise} 

Next, we fix $D=30, N=500, M/(N+M)=0.5$, and investigate the performance of the methods, adding DPCP-d to the mix, in the presence of varying levels of noise for two cases of high relative dimension, i.e., $d/D=25/30$ and $d/D=29/30$. The inliers are corrupted by additive white Gaussian noise of zero mean and standard deviation $\sigma=0.05,0.1,0.2$, with support in the orthogonal complement of the inlier subspace. The parameters of all methods are the same as earlier except for DPCP-d we set $\tau = \max\left\{\sigma,1/\sqrt{N+M}\right\}$, while for RANSAC we set its threshold parameter equal to $\sigma$.

We evaluate the performance of each method by its corresponding ROC curve. Each point of an ROC curve corresponds to a certain value of a threshold, with the vertical coordinate of the point giving the percentage of inliers being correctly identified as inliers (True Positives), and the horizontal coordinate giving the number of outliers erroneously identified as inliers (False Positives). As a consequence, an ideal ROC curve should be concentrated to the top left of the plot, i.e., the area above the curve should be zero. 
The ROC curves\footnote{We note that the vertical axis of all ROC curves in this paper starts from a ratio of $0.1$ True Positives.} as well as the area above each curve are shown in Fig. \ref{figure:ROC_synthetic}. As expected, the low-rank methods
RANSAC, SE-RPCA and $\ell_{21}$-RPCA  perform poorly for either relative dimension with performance being close to that of a random guess (inlier vs. outlier) for relative dimension $29/30$. On the other hand REAPER, CoP, DPCP-LP, DPCP-IRLS and DPCP-d perform almost perfectly well for $d/D=25/30$, while REAPER starts failing for very high relative dimension $29/30$, even for as low noise standard devision as $\sigma=0.05$ (of course this is to be expected because we already know from Fig. \ref{figure:separation} that REAPER fails at this regime even in the absence of noise), and CoP fails completely. In contrast, the DPCP variants remain robust to noise in this 
challenging regime for as high noise as $\sigma = 0.1$. What is remarkable, is that both DPCP-LP and DPCP-IRLS, which are designed for noiseless data, are surprisingly robust to noise, and slightly outperform DPCP-d, the latter meant to handle noisy data. We attribute this fact to the suboptimal approach we followed in solving the DPCP-d problem, as well as to the lack of a suitable mechanism for optimally tuning its thresholding parameter. 

\subsection{Comparative analysis using real data: Three-view geometry} \label{subsectionExperiments:Real}

In this section we perform an experimental evaluation of the proposed methods in the context of the three-view problem in computer vision using real data. In that setting one is given three images of the same static scene taken from different views, and the goal is to estimate the relative view poses, i.e., the rotations and translations that relate, say, view $1$ to view $2$ and $3$ (e.g., see Figs \ref{figure:trifocal_3view}-\ref{figure:trifocal_vis_cameras}). This task is of fundamental importance in many computer vision applications, such as $3$D reconstruction, where a $3$D model of a real-world scene is constructed from $2$D images of the scene. 
 
\begin{figure}[t!]
\centering
\subfigure[Three views of the same scene.]{\label{figure:trifocal_3view}\includegraphics[width=0.7\linewidth]{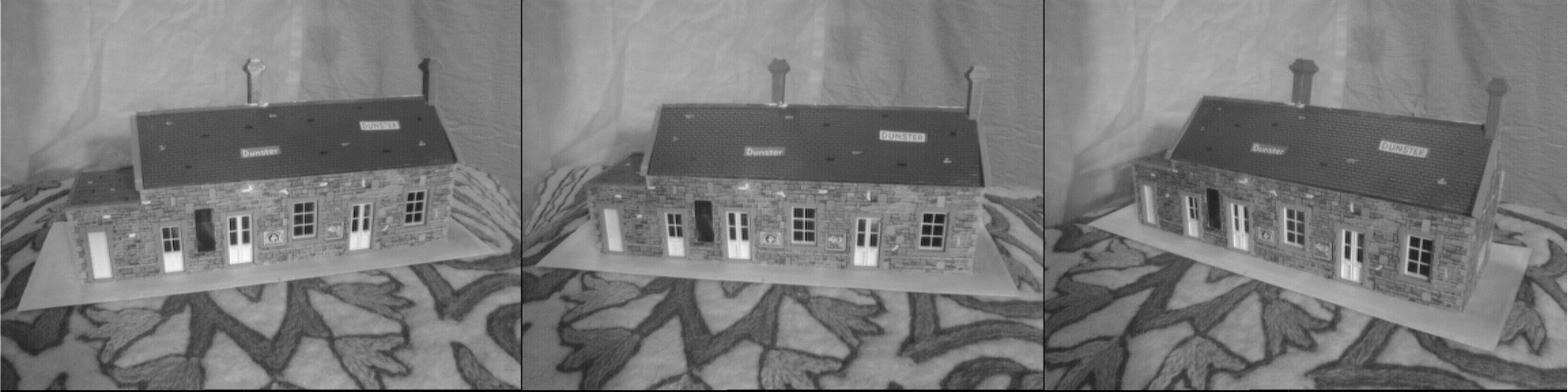}}\\
\subfigure[Left: Example of points viewed by all three cameras. Right: Configuration of cameras and the same points depicted in $3$D space. Color  represents height.]{\label{figure:trifocal_vis_cameras}\includegraphics[width=0.7\linewidth]{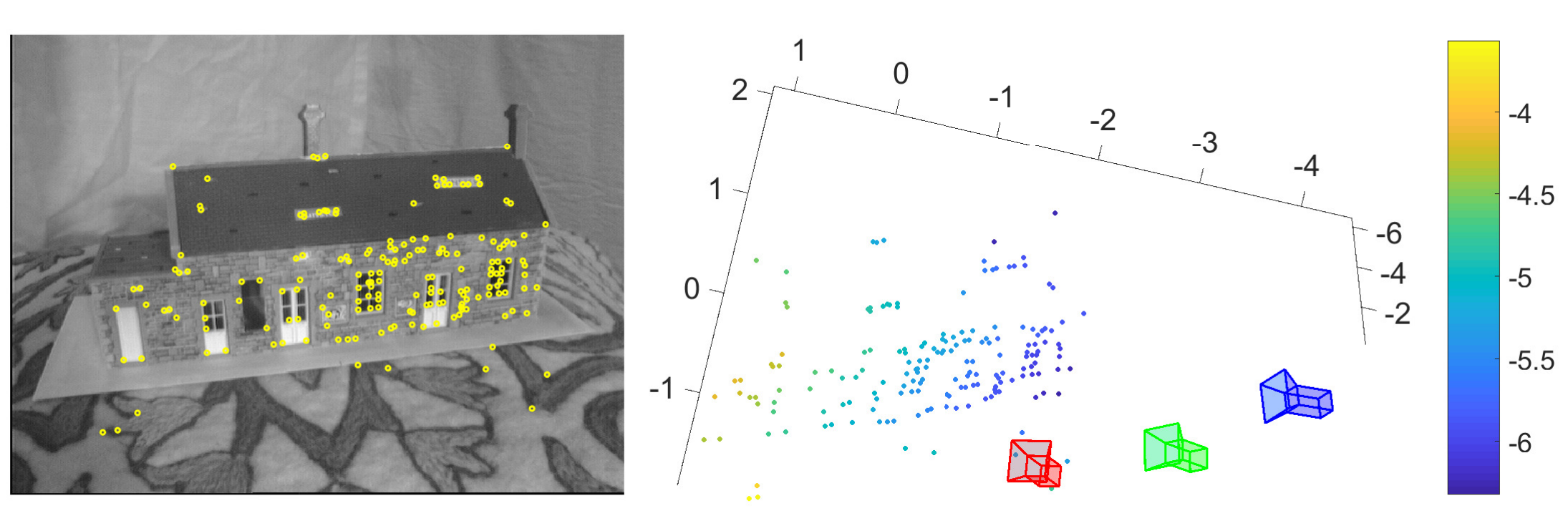}}\\
\subfigure[Examples of correct (yellow) and incorrect (red) point correspondences between the three views.]{\label{figure:trifocal_3view_corres}\includegraphics[width=0.7\linewidth]{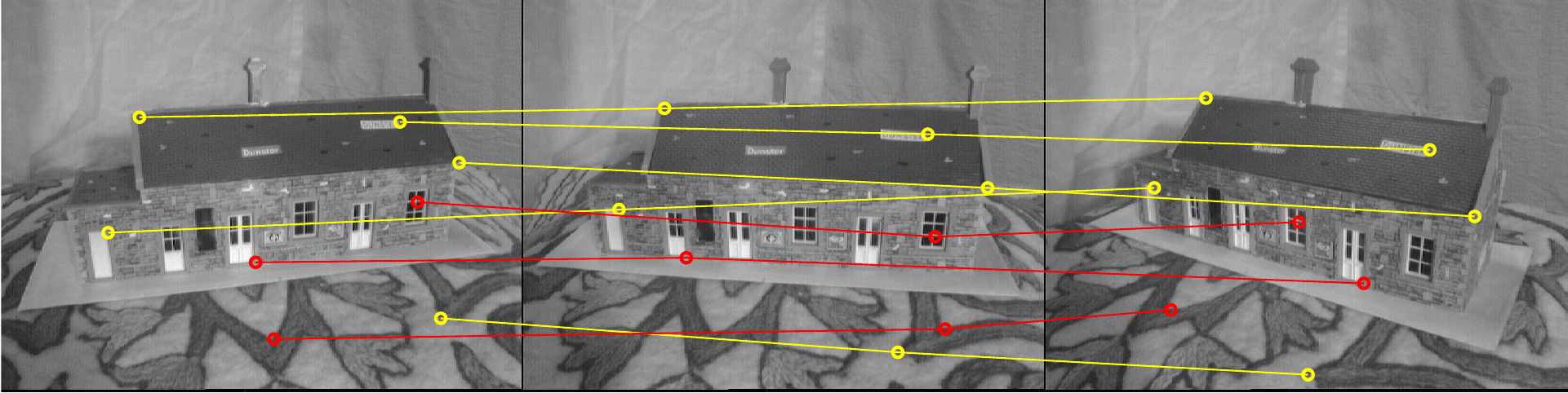}}\\	
	\caption{An example of three views of a static scene along with camera configurations and point correspondences (views $2, 4, 6$ of the Model House dataset, provided by the Visual Geometry Group, University of Oxford).}\label{figure:trifocal}
\end{figure} 
\myparagraph{The trifocal tensor}
The three images of the static scene may have been taken from three different cameras, or from a single moving camera. Regardless, the underlying three-view geometry is characterized by the constraints satisfied by any points lying in views $1,2$ and $3$ respectively, that correspond to the same $3$D point (e.g., see Fig. \ref{figure:trifocal_vis_cameras}). To describe the nature of these constraints, we fix a coordinate system $(x,y,z)$ for the $3$D space, and identify view $1$ with the plane $z=1$ and its optical center with the origin $\0$ of the coordinate system. We refer to this view as \emph{canonical view} $\V$. Then the projection $\X$ of a $3$D point $\Xi=(\xi_1,\xi_2,\xi_3)^\top$ onto $\V$ is the intersection point of the plane $z=1$ with the line that passes through $\Xi$ and the origin, i.e., $\X=\lambda \Xi$ with $\lambda = 1/\xi_3$. For simplicity, we assume that views $2$ ($\V'$) and $3$ ($\V''$) are rotated and translated versions of the canonical view $\V$, i.e., there exist rotations $\bR', \bR'' \in \text{SO}(3)$ and translations $\t',\t'' \in \Re^3$, such that\footnote{This geometry corresponds to \emph{calibrated cameras}, where the camera calibration parameters are known.}\footnote{In the computer vision literature it is customary to represent each camera with its projection matrix. In our discussion this representation takes the form $[\bI \, \, \0], \, [\bR' \, \, \t'], \,  [\bR'' \, \, \t'']$.}
\begin{align}
 \V = \bR'(\V')+\t'  = \bR''(\V'')+\t''.
\end{align} The projection $\X'$ of the $3$D point $\Xi$ onto $\V'$ is the intersection of $\V'$ with the line that passes from $\Xi$ and the optical center $-\bR'^\top\t'$ of view $2$. However, in practice $\X'$ is only known up to local pixel coordinates with respect to view $2$. That is, we can only know the representation $\x'$ of $\X'$ with respect to a coordinate system where view $2$ is the canonical view. In such a system of coordinates the point $\Xi$ is represented as $\bR' \Xi +\t'$ and hence $\x' = \lambda' (\bR' \Xi +\t')$, where $\lambda'$ is the inverse of the third coordinate of the vector $\bR' \Xi +\t'$. Substituting $\Xi = (1/\lambda) \x$ in this equation yields a relation between the local representations\footnote{Without loss of generality we take the local system of coordinates of view $1$ to be the same as the global system of coordinates.} $\x,\x'$ of $\X,\X'$ in $\V$ and $\V'$ respectively as follows: 
\begin{align}
\frac{1}{\lambda'} \x' = \frac{1}{\lambda} \bR' \x +\t'. \label{eq:xx'}
\end{align} Now, for a vector $\v = (\alpha,\beta,\gamma)^\top \in \Re^3$, denote by $[\v]$ the skew-symmetric matrix 
\begin{align}
[\v] = \begin{bmatrix}
0 & \gamma & -\beta \\
-\gamma & 0 & \alpha \\
\beta & -\alpha & 0
\end{bmatrix},
\end{align} and note that $[\v] \v = \0$. Multiplying equation \eqref{eq:xx'} from the left by $[\x']$ gives
\begin{align}
\frac{1}{\lambda} [\x']\bR' \x +[\x'] \t' = \0. \label{eq:xx'h}
\end{align} In exactly the same way, and letting $\x''$ be the representation of $\X''$ in the canonical coordinate system for view $3$, we have a relationship
\begin{align}
\frac{1}{\lambda} [\x'']\bR'' \x +[\x''] \t'' = \0. \label{eq:xx''h}
\end{align} Degenerate cases aside, equations \eqref{eq:xx'h}-\eqref{eq:xx''h} are equivalent to the condition 
\begin{align}
\Rank \left(\begin{bmatrix}
[\x']\bR' \x &  [\x']\t' \\
[\x'']\bR'' \x &  [\x'']\t''
\end{bmatrix} \right) \le 1, \label{eq:RankCondition}
\end{align} which in turn is equivalent to the matrix equation\footnote{Here we have used the fact that for vectors $\a,\b,\c,\d \in \Re^n$ the $2n \times 2$ matrix $\begin{bmatrix} \a & \b \\ \c & \d \end{bmatrix}$ has rank at most $1$ if and only if $\a \d^\top - \b \c^\top = \0_{n \times n}$; thanks to Tianjiao Ding for this observation.}
\begin{align}
[\x']\bR' \x \, \t''^\top [\x'']^\top- [\x']\t' \x^\top \, \bR''^\top [\x'']^\top  = \0_{3 \times 3}, \label{eq:NineEquations}
\end{align} or more elegantly written as 
\begin{align}
[\x'] \Big(\sum_{i=1}^3 x_i \bT_i \Big) [\x'']^\top =\0_{3 \times 3}, \, \, \, \x=(x_1,x_2,x_3)^\top, \, \, \, \bT_i := \r_i' \t''^\top -  \t' \r_i''^\top, \, i=1,2,3, \label{eq:slices}
\end{align} where $\r_i', \r_i''$ is the $i$th column of $\bR', \bR''$ respectively.
Equation \eqref{eq:slices} consists of $9$ \emph{trilinear} constraints on the local representations $\x,\x',\x''$ of the imaged $3$D point $\Xi$, among which a maximal number of four are linearly independent \citep{Hartley-Zisserman04}. The matrices $(\bT_1,\bT_2,\bT_3)$ are the slices of the so-called \emph{trifocal tensor}\footnote{In the uncalibrated case, which is more relevant in computer vision applications, the trifocal tensor has exactly the same structure as in \eqref{eq:slices}, with the only difference being that the matrices $\bR', \bR''$ are no longer rotations.} $\mathcal{T} \in \Re^{3 \times 3} \times \Re^{3 \times 3} \times \Re^{3 \times 3}$, which is the mathematical object that encodes the relative geometry of the calibrated three views: indeed, up to a change of coordinates there is a $1-1$ correspondence between trifocal tensors and camera views $\V, \V', \V''$; see Proposition $15$ and Theorem $16$ in \cite{Kileel:SIAM17}.

\myparagraph{Trifocal tensor estimation as a hyperplane learning problem}
Notice that the trilinear constraints \eqref{eq:slices} are linear in the entries of the tensor $\T=(\bT_1,\bT_2,\bT_3)$, which in its unfolded form can be regarded as a vector $\boldsymbol{\mathfrak{t}} \in \Re^{27}$. In fact, the space of (uncalibrated) trifocal tensors is an algebraic variety of $\Re^{27}$ of dimension $19$\footnote{$18$ in projective space $\mathbb{P}^{26}$.} \citep{Alzati:2010,Oeding:14}. As already noted, every point correspondence $(\x,\x',\x'')$ contributes four linearly independent equations in $\boldsymbol{\mathfrak{t}}$; equivalently, every point correspondence cuts the variety with four hyperplanes. As it turns out though, only $3$ of these hyperplanes are algebraically independent with respect to the variety\footnote{A more precise way to state this in algebraic-geometric language is that the ideal generated by these four equations has depth $3$ in the quotient ring of the trifocal variety.}, i.e., every generic point correspondence reduces the dimension of the variety by three \citep{Kileel:SIAM17}. As a result, one needs $6$ correspondences to get a finite number of candidate trifocal tensors that agree with them. Adding a $7$th correspondence allows us to uniquely determine $\boldsymbol{\mathfrak{t}}$ via solving a $28 \times 27$ homogeneous linear system of equations, while the relative poses $(\bR',\t')$ and $(\bR'',\t'')$ can be extracted from $\boldsymbol{\mathfrak{t}}$ by, e.g., the procedure described by \cite{Hartley-Zisserman04}. 

The above discussion suggests that given a set of $N' \ge 7$ generic and exact point correspondences $\{(\x_j,\x'_j,\x''_j)\}_{j=1}^{N'}$, the coefficient vectors 
\begin{align}
\c_1^{(1)},\c_2^{(1)},\c_3^{(1)},\c_4^{(1)},\dots,\c_1^{(j)},\c_2^{(j)},\c_3^{(j)},\c_4^{(j)}\dots,\c_1^{(N')},\c_2^{(N')},\c_3^{(N')},\c_4^{(N')} \in \Re^{27}, \label{eq:TriEmb}
\end{align} of the resulting $N=4N'$ linear equations span a hyperplane in $\Re^{27}$ with normal vector $\boldsymbol{\mathfrak{t}}$. We will be referring to the vectors \eqref{eq:TriEmb} as the \emph{trilinear embeddings} of the point correspondences. However, in practice one obtains such correspondences by matching points across images based on the similarity of some local features, such as \emph{SIFT} \citep{Lowe:CVPR99}. As a result, it is typically the case that many of the produced correspondences are incorrect (see Fig. \ref{figure:trifocal_3view_corres}), and the problem then becomes that of detecting inliers lying close to a hyperplane of $\Re^{27}$, from a dataset corrupted by outliers. Equivalently, one is presented with a codimension $1$ subspace learning problem, for which the proposed Dual Principal Component Pursuit (DPCP) is ideally suited; as we show next, the method can achieve superior performance than RANSAC, the latter being the traditional and up to date one of the most popular options in the computer vision community for such problems.

\myparagraph{Data} We use the first three views of the datasets \emph{Model House}, \emph{Corridor} and \emph{Merton College III}, provided by the Visual Geometry Group at Oxford University. Each of such dataset contains different views of the same static scene, together with the projection matrices of each view and high-quality (inlier) point correspondences. From each dataset we randomly pick $N'=125$ inlier correspondences $\{(\x_j,\x'_j,\x''_j)\}_{j=1}^{N'}$. We further generate $100\cdot M'/(N'+M')=30\%,\, 40\%, 50\%$ outlier correspondences $\{(\o_j,\o'_j,\o''_j)\}_{j=1}^{M'}$ as follows: For each triplet $(\V, \V', \V'')$ of views we randomly sample $M'$ points in each view from a Gaussian distribution with mean and covariance equal to that of the points associated with the inlier correspondences, and we randomly match them in $M'$ triples. We normalize all data according to \cite{Hartley:InDefense} and form a unit $\ell_2$-norm dataset $\btX=[\bX \, \, \bO] \boldsymbol{\Gamma} \in \Re^{27 \times 4(M+N)}$ that consists of the trilinear embeddings (see \eqref{eq:TriEmb}) of all inlier/outlier correspondences.

\myparagraph{Algorithms} We compare RANSAC and REAPER (see \S \ref{section:RelatedWork}) with the proposed fast DPCP variants DPCP-IRLS (Algorithm \ref{alg:DPCP-IRLS}) and DPCP-d (Algorithm \ref{alg:DPCP-d}) in the context of outlier detection for trifocal tensor estimation\footnote{The purely low-rank methods SE-RPCA and $\ell_2$-RPCA are unsuitable for learning a hyperplane; since Coherence Pursuit (CoP) performed much better than them with synthetic data, we also included it in our experiments, but do not report its performance as it was not competitive with the other tested methods, i.e., RANSAC, REAPER and DPCP. Notice from Fig. \ref{figure:ROC_synthetic} that with synthetic data CoP fails precisely at the case of a hyperplane.}. REAPER, DPCP-IRLS and DPCP-d receive as input the dataset $\btX$ and are configured to learn a subspace of dimension $26$ in $\Re^{27}$ (a hyperplane). We compare with two variations of RANSAC. The first variation, called H-RANSAC (\emph{hyperplane} RANSAC), is exactly as the standard RANSAC, except that it samples trilinear embeddings in groups of four (instead of individually), where each group is associated with a point correspondence\footnote{The standard RANSAC in this context samples trilinear embeddings disregarding the known information of which trilinear embeddings are associated to the same point-point-point correspondence. This performs poorly and for this reason we do not report it here.}. The second variation, called R-RANSAC (\emph{re-projection} RANSAC) is a more common choice in the computer vision community, and searches directly for the relative poses of the three views by minimizing the \emph{re-projection error}. That is, within its given time budget R-RANSAC repeats the following steps:
\begin{enumerate}
\item Randomly samples $7$ point correspondences.
\item Computes through SVD on the $28$ trilinear embeddings an approximate trifocal tensor $\boldsymbol{\mathfrak{t}}$.
\item  Runs Alg. $16.2$ of \cite{Hartley-Zisserman04} to obtain a valid trifocal tensor $\boldsymbol{\mathfrak{t}}^*$ from $\boldsymbol{\mathfrak{t}}$.
\item Extracts the relative poses from $\boldsymbol{\mathfrak{t}}^*$ (see \cite{Hartley-Zisserman04} for details).
\item For every correspondence $(\tilde{\x}_j,\tilde{\x}'_j,\tilde{\x}''_j)$ computes via triangulation the viewed $3$D point $\tilde{X}$, and through the extracted poses obtains its re-projections $(\hat{\tilde{\x}}_j,\hat{\tilde{\x}}'_j,\hat{\tilde{\x}}''_j)$ onto the three views.
\item For every point correspondence computes the associated re-projection error
\begin{align}
\frac{1}{3} \left(\left\|\tilde{\x}_j - \hat{\tilde{\x}}_j\right\|_2+\left\|\tilde{\x}_j' - \hat{\tilde{\x}}'_j\right\|_2+\left\|\tilde{\x}_j'' - \hat{\tilde{\x}}_j''\right\|_2\right).
\end{align}
\item Declares as inliers the correspondences with re-projection error less than a threshold. 
\end{enumerate} For both H-RANSAC and R-RANSAC we use as thresholds the corresponding maximal ground-truth distance to hyperplane and re-projection error respectively. Regarding time budget, we note that the fastest method is DPCP-d, and then follows REAPER and DPCP-IRLS. For example, for the experiment we are considering and for $40\%$ outliers, DPCP-d needs an average of about $0.03$sec to converge as opposed to $0.1$sec and $0.6$sec for REAPER and DPCP-IRLS respectively. Since RANSAC's performance is sensitive to the allocated time budget, we set this equal to the running time of DPCP-IRLS (high time budget) as well as that of DPCP-d (low time budget). In the latter case, for fairness, we also restrict the running time of REAPER and DPCP-IRLS accordingly.

\myparagraph{Results} We use as a metric the \emph{precision} of the algorithms that corresponds to a \emph{recall} value equal to $1$: given a hyperplane estimate $\hat{\H}$, this induces an ordering of all the inlier/outlier point correspondences based on increasing average distance of the corresponding $4$-tuples of trilinear embeddings to $\hat{\H}$. Letting $\alpha$ be the maximal such average distance that corresponds to an inlier point correspondence, our metric is the percentage of the inliers among all correspondences with average distance to $\hat{\H}$ less or equal than $\alpha$; for R-RANSAC we use re-projection error instead of average distance to hyperplane\footnote{When varying the outlier ratio but keeping the number of inliers fixed, the area under the ROC or precision-recall curve can be misleading, hence we do not use it here.}. Tables \ref{table:trifocal-co}, \ref{table:trifocal-mh} and \ref{table:trifocal-mc3} report the precision of the algorithms for the three different scenes, \emph{Corridor, Model House} and \emph{Merton College III} respectively, for different outlier ratios and different time budgets.

As Table \ref{table:trifocal-co} shows, for the rather easy dataset \emph{Corridor} and for $30\%$ outliers all methods give almost perfect precision regardless of time budget. The fact that R-RANSAC is slightly more computationally demanding than H-RANSAC reveals itself for $40\%$ outliers and low time budget, with the precision of the former dropping from $1$ to $0.601$, while the latter from $0.958$ to $0.903$. The robustness of the re-projection (R-RANSAC) error vs. the algebraic error (H-RANSAC) is evident from the case of $50\%$ outliers, where for high time budget the former gives a precision of $1$, while the latter essentially fails with a precision of $0.648$. At any case, both H-RANSAC and R-RANSAC fail for $50\%$ outliers and low time budget, indicating the sensitivity of RANSAC variants to the allocated running time. On the other hand, REAPER has a high precision $0.969$-$0.954$ even for the challenging case of $50\%$, while the DPCP variants are the best methods with DPCP-d maintaining a precision of $1$ across all outlier and time budget configurations.
 
Table \ref{table:trifocal-mh} shows that for the more challenging dataset \emph{Model House} all methods maintain a  high precision for up to $40\%$ outliers and high time budget, while both RANSAC variants start failing for $40\%$ outliers and low time budget. For this last case DPCP-d is only method having high precision $0.977$, while the next best method is REAPER with $0.906$. Interestingly, for $50\%$ outliers and if let run to convergence, DPCP-IRLS is the best method with precision $0.929$, while for low time budget all methods fail except DPCP-d. For the even more challenging dataset \emph{Merton College III} all methods fail for $50\%$ outliers, while DPCP-IRLS is the best method with precision $1$ for $40\%$ if let run to convergence (high time budget) and DPCP-d is the best method for low time budget.

In conclusion, even though RANSAC can have a very high precision given sufficient time budget, once the latter is restricted its performance can drop dramatically. This is particularly the case for large outlier ratios, a regime where an exponentially large time budget might be needed. On the other hand, DPCP-IRLS is the best method for large outlier ratios, striking a balance between computational requirements and precision, while DPCP-d is slightly less precise than DPCP-IRLS yet much faster. Overall, the above experiment suggests that the proposed Dual Principal Component Pursuit variants can be a useful or even superior alternative to popular approaches such as RANSAC, for three-view geometry or other computer vision applications.

\begin{table}[h!]
\caption{Algorithm precision when recall value is $1$ for the first three views of dataset \emph{Corridor}.}
\label{table:trifocal-co}
\resizebox{\textwidth}{!}{%
\begin{tabular}{@{}lcccccccc@{}}
\toprule
Algorithm vs. & \multicolumn{2}{c}{30\% outliers} & \multicolumn{1}{l}{} & \multicolumn{2}{c}{40\% outliers} & \multicolumn{1}{l}{} & \multicolumn{2}{c}{50\% outliers} \\ \cmidrule(lr){2-3} \cmidrule(lr){5-6} \cmidrule(l){8-9} 
\% of outliers \& time budget & high t.b. & low t.b. &  & high t.b. & low t.b. &  & high t.b. & low t.b. \\ \midrule
H-RANSAC & 0.996 & 0.969 &  & 0.958 & 0.903 &  & 0.648 & 0.601 \\
R-RANSAC & \textbf{1.000} & \textbf{1.000} &  & \textbf{1.000} & 0.601 &  & \textbf{1.000} & 0.506 \\
REAPER-IRLS & \textbf{1.000} & \textbf{1.000} &  & \textbf{1.000} & 0.992 &  & 0.969 & 0.954 \\
DPCP-IRLS & \textbf{1.000} & \textbf{1.000} &  & \textbf{1.000} & \textbf{1.000} &  & \textbf{1.000} & 0.965 \\
DPCP-d & \textbf{1.000} & \textbf{1.000} &  & \textbf{1.000} & \textbf{1.000} &  & \textbf{1.000} & \textbf{1.000} \\ \bottomrule
\end{tabular}%
}
\end{table}

\begin{table}[h!]
\caption{Algorithm precision when recall value is $1$, for the first three views of dataset \emph{Model House}.}
\label{table:trifocal-mh}
\resizebox{\textwidth}{!}{%
\begin{tabular}{@{}lcccccccc@{}}
\toprule
Algorithm vs. & \multicolumn{2}{c}{30\% outliers} & \multicolumn{1}{l}{} & \multicolumn{2}{c}{40\% outliers} & \multicolumn{1}{l}{} & \multicolumn{2}{c}{50\% outliers} \\ \cmidrule(lr){2-3} \cmidrule(lr){5-6} \cmidrule(l){8-9} 
\% of outliers \& time budget & high t.b. & low t.b. &  & high t.b. & low t.b. &  & high t.b. & low t.b. \\ \midrule
H-RANSAC & 0.984 & 0.962 &  & 0.940 & 0.636 &  & 0.610 & 0.510 \\
R-RANSAC & \textbf{0.992} & 0.779 &  & \textbf{0.988} & 0.601 &  & 0.503 & 0.502 \\
REAPER-IRLS & 0.984 & 0.965 &  & 0.947 & 0.906 &  & 0.700 & 0.693 \\
DPCP-IRLS & 0.973 & 0.951 &  & 0.980 & 0.856 &  & \textbf{0.929} & 0.612 \\
DPCP-d & 0.973 & \textbf{0.973} &  & 0.977 & \textbf{0.977} &  & 0.912 & \textbf{0.912} \\ \bottomrule
\end{tabular}%
}
\end{table}

\begin{table}[h!]
\caption{Algorithm precision when recall value is $1$ for the three views of dataset \emph{Merton College III}.}
\label{table:trifocal-mc3}
\resizebox{\textwidth}{!}{%
\begin{tabular}{@{}lcccccccc@{}}
\toprule
Algorithm vs. & \multicolumn{2}{c}{30\% outliers} & \multicolumn{1}{l}{} & \multicolumn{2}{c}{40\% outliers} & \multicolumn{1}{l}{} & \multicolumn{2}{c}{50\% outliers} \\ \cmidrule(lr){2-3} \cmidrule(lr){5-6} \cmidrule(l){8-9} 
\% of outlier \& time budget & high t.b. & low t.b. &  & high t.b. & low t.b. &  & high t.b. & low t.b. \\ \midrule
H-RANSAC & 0.988 & 0.900 &  & 0.751 & 0.689 &  & \textbf{0.579} & \textbf{0.558} \\
R-RANSAC & 0.702 & 0.698 &  & 0.601 & 0.601 &  & 0.507 & 0.507 \\
REAPER-IRLS & \textbf{1.000} & 0.992 &  & 0.958 & 0.923 &  & 0.502 & 0.502 \\
DPCP-IRLS & \textbf{1.000} & 0.977 &  & \textbf{1.000} & 0.784 &  & 0.513 & 0.513 \\
DPCP-d & \textbf{1.000} & \textbf{1.000} &  & 0.992 & \textbf{0.992} &  & 0.507 & 0.507 \\ \bottomrule
\end{tabular}%
}
\end{table}

\section{Conclusions}
We presented and studied a solution to the problem of robust principal component analysis in the presence of outliers, called
\emph{Dual Principal Component Pursuit (DPCP)}. The heart of the proposed method consisted of
a non-convex $\ell_1$ optimization problem on the sphere, for which a solution strategy based on a recursion of linear programs was analyzed. Rigorous mathematical analysis revealed that DPCP is a natural method for learning the inlier subspace in the presence of outliers, even in the challenging regime of large outlier ratios and high subspace relative dimensions. In fact, experiments on synthetic data showed that DPCP was the only method that could handle $70\%$ outliers inside a $30$-dimensional ambient space, irrespectively of the subspace dimension. Moreover, experiments with real images in the context of the trifocal tensor and three-view geometry suggest that DPCP could be a useful tool in computer vision problems.

\appendix		
\section{Review of existing results on Problems \eqref{eq:ell1} and \eqref{eq:ConvexRelaxations}} \label{appendix:Spath}

In this appendix we state three results that are important for our mathematical
analysis, already known in \cite{Spath:Numerische87}. For the sake of clarity and convenience, we have also taken the liberty of writing complete proofs of the statements, as not all of them can be found in \cite{Spath:Numerische87}.

\begin{prp} \label{prp:NonConvexMaximalInterpolation}
Let $\bY=[\y_1,\dots,\y_L] \in D \times L$ be full rank.
Then any global solution $\b^*$ to 
\begin{align}
\min_{\b^\transpose \b=1} \big \|\bY^\transpose \b \big \|_1, \label{eq:ell1Y}
\end{align} must be orthogonal to $(D-1)$ linearly independent columns of $\bY$.
\end{prp}
\begin{proof} Let $\b^*$ be an optimal solution of \eqref{eq:ell1Y}. Then $\b^*$ must satisfy the first order optimality relation
\begin{align}
\0 \in \bY \Sgn(\bY^\transpose \b^*) + \lambda \b^* \label{eq:OptimalityConditionY},
\end{align} where $\lambda$ is a scalar Lagrange multiplier parameter, and
$\Sgn$ is the sub-differential of the $\ell_1$ norm. Without loss of generality, let $\y_1,\dots,\y_K$ be the 
columns of $\bY$ to which $\b^*$ is orthogonal. Then 
equation \eqref{eq:OptimalityConditionY} implies that there exist real numbers 
$\alpha_1,\dots,\alpha_K \in [-1,1]$ such that
\begin{align}
\sum_{j=1}^K \alpha_j \y_j + \sum_{j=K+1}^L \Sign(\y_j^\transpose \b^*) \y_j+ \lambda \b^*= \0. \label{eq:FirstOrderExpandedY}
\end{align} Now, suppose that the span of $\y_1,\dots,\y_K$ is of dimension less than $D-1$. Then there exists a unit norm vector $\bzeta \in \Sp^{D-1}$ that is orthogonal to all $\y_1,\dots,\y_K,\b^*$, and multiplication of \eqref{eq:FirstOrderExpandedY} from the left by $\bzeta^\transpose$ gives
	\begin{align}
	\sum_{j=K+1}^L \Sign(\y_j^\transpose \b^*) \bzeta^\transpose \y_j = 0.
	\end{align} Furthermore, we can choose a sufficiently small $\varepsilon >0$, such that 
	\begin{align}
	\Sign(\y_j^\transpose \b^* + \varepsilon \y_j^\transpose \bzeta) = \Sign(\y_j^\transpose \b^*), \, \, \, \forall j \in [L].
	\end{align} The above equation is trivially true for all $j$ such that $\y_j^\transpose \b^*=0$, because in that case $\y_j^\transpose \bzeta=0$ by the definition of $\bzeta$. On the other hand, if $\y_j^\transpose \b^* \neq 0$, then a small perturbation $\epsilon$ will not change the sign of $\y_j^\transpose \b^*$. Consequently, we can write 
	\begin{align}
	\left| \y_j^\transpose (\b^* + \varepsilon \bzeta) \right| = \left| \y_j^\transpose \b^* \right| + \varepsilon \Sign(\y_j^\transpose \b^*) \y_j^\transpose \bzeta, \, \forall j \in [L]
	\end{align} and so
	\begin{align}
	\big \| \bY^\transpose (\b^* + \varepsilon \bzeta) \big \|_1 = \big \|\bY^\transpose \b^* \big \|_1 + \varepsilon\sum_{j=K+1}^L \Sign(\y_j^\transpose \b^*) \bzeta^\transpose \y_j = \big \|\bY^\transpose \b^* \big \|_1.
	\end{align} However, 
	\begin{align}
	\big \|\b^* + \varepsilon \bzeta \big \|_2 = \sqrt{1 + \varepsilon^2} >0,
	\end{align} and normalizing $\b^*+\varepsilon \bzeta$ to have unit $\ell_2$ norm, we get a contradiction on $\b^*$ being a global solution. 
\end{proof}

\begin{prp} \label{prp:LPMaximalInterpolation}
	Let $\bY=[\y_1,\dots,\y_L]$ be a $D \times L$ matrix of rank $D$. Then problem 
\begin{align}	
\min_{\b^\transpose \hat{\bn}_k=1} \big \|\bY^\transpose \b \big \|_1 \label{eq:Watson}
\end{align} admits a computable solution $\bn_{k+1}$ that is orthogonal to $(D-1)$ linearly independent points of $\bY$. 
\end{prp}
\begin{proof}
	Let $\bn_{k+1}$ be a solution to $\min_{\b^\transpose \hat{\bn}_k=1} \big \|\bY^\transpose \b \big \|_1$ that is orthogonal to less than $D-1$ linearly independent points of $\bY$. Then we can find a unit norm vector $\bzeta$ that is orthogonal to the same points of $\bY$ that $\bn_{k+1}$ is orthogonal to, and moreover $\bzeta \perp \bn_{k+1}$. In addition, we can find a sufficiently small $\varepsilon>0$ such that
	\small
	\begin{align}
	\big \| \bY^\transpose (\bn_{k+1} + \varepsilon \bzeta) \big \|_1 = \big \|\bY^\transpose \bn_{k+1} \big \|_1 + \varepsilon\sum_{j: \, \bn_{k+1} \not\perp \y_j} \Sign(\y_j^\transpose \bn_{k+1}) \bzeta^\transpose \y_j ,
	\end{align} \normalsize where 
	\begin{align}
	\sum_{j: \, \bn_{k+1} \not\perp \y_j} \Sign(\y_j^\transpose \bn_{k+1}) \bzeta^\transpose \y_j \le 0.
	\end{align} Since $\bn_{k+1}$ is optimal, it must be the case 
	that 
	\begin{align}
	\sum_{j: \, \bn_{k+1} \not\perp \y_j} \Sign(\y_j^\transpose \bn_{k+1}) \bzeta^\transpose \y_j = 0, \label{eq:DifferentSigns}
	\end{align} and so 
	\begin{align}
	\big \| \bY^\transpose (\bn_{k+1} + \varepsilon \bzeta) \big \|_1 = \big \|\bY^\transpose \bn_{k+1} \big \|_1. \label{eq:VaryingEpsilon}
	\end{align} By \eqref{eq:VaryingEpsilon} we see that as we vary $\varepsilon$ the objective remains unchanged. Notice also that varying $\varepsilon$ preserves all zero entries appearing in the vector $\bY^\transpose \bn_{k+1}$. Furthermore, because of \eqref{eq:DifferentSigns}, it is always possible to either decrease or increase $\varepsilon$ until an additional zero is achieved, i.e., until $\bn_{k+1} + \varepsilon \bzeta$ becomes orthogonal to a point of $\bY$ that $\bn_{k+1}$ is not orthogonal to. Then we can replace $\bn_{k+1}$ with $\bn_{k+1} + \varepsilon \bzeta$ and repeat the process, until we get some $\bn_{k+1}$ that is orthogonal to $D-1$ linearly independent points of $\bY$.
\end{proof}

\begin{prp} \label{prp:LPconvergence}
Let $\bY=[\y_1,\dots,\y_L]$ be a $D \times L$ matrix of rank $D$. Suppose that for each problem \eqref{eq:Watson} a solution $\bn_{k+1}$ is chosen such that $\bn_{k+1}$ is orthogonal to $D-1$ linearly independent points of $\bY$, in accordance with Proposition \ref{prp:LPMaximalInterpolation}. Then the sequence $\left\{\bn_k\right\}$ converges to a critical point of problem \eqref{eq:ell1Y} in a finite number of steps.
\end{prp}
\begin{proof}	
	If $\bn_{k+1} = \hat{\bn}_k$, then inspection of the first order optimality conditions of the two problems, reveals that $\hat{\bn}_k$ is a critical point of $\min_{\b^\transpose \b=1} \big \|\bY^\transpose \b \big \|_1$. If $\bn_{k+1} \neq \hat{\bn}_k$, then $\big \| \bn_{k+1} \big \|_2>1$, and so 
	$\big \|\bY^\transpose \hat{\bn}_{k+1} \big \|_1 < \big \|\bY^\transpose \hat{\bn}_{k} \big \|_1$. As a consequence, if $\bn_{k+1} \neq \hat{\bn}_k$, then $\hat{\bn}_k$ can not arise as a solution for some $k' > k$. 
	Now, because of Proposition \ref{prp:LPMaximalInterpolation}, for each $k$, there is a finite number of candidate directions $\bn_{k+1}$. These last two observations imply that the sequence $\left\{\bn_k\right\}$ must converge in a finite number of steps to a critical point of $\min_{\b^\transpose \b=1} \big \|\bY^\transpose \b \big \|_1$.
\end{proof}

\section*{Acknowledgement}
This work was supported by NSF grants 1447822, 1618637 and 1704458.  The first author is thankful to Dr. Daniel P. Robinson for many comments that helped improve this manuscript, to Dr. Glyn Harman for indicating the proof of Lemma \ref{lem:Koksma-eO}, to Dr. Zhihui Zhu for pointing out the approximate asymptotic behavior $M<\mathcal{O}(N^2)$, and to Tianjiao Ding for his help in producing the experiments on the trifocal tensor. The authors also thank Dr. Gilad Lerman for useful discussions regarding robust PCA, as well as the two anonymous reviewers for their constructive comments.

\bibliography{DPCP-JMLR18_final.bbl}

\begin{thebibliography}{59}
\providecommand{\natexlab}[1]{#1}
\providecommand{\url}[1]{\texttt{#1}}
\expandafter\ifx\csname urlstyle\endcsname\relax
  \providecommand{\doi}[1]{doi: #1}\else
  \providecommand{\doi}{doi: \begingroup \urlstyle{rm}\Url}\fi

\bibitem[Balzano et~al.(2010)Balzano, Nowak, and Recht]{Balzano:Allerton10}
L.~Balzano, R.~Nowak, and B.~Recht.
\newblock Online identification and tracking of subspaces from highly
  incomplete information.
\newblock In \emph{Communication, Control, and Computing (Allerton), 2010 48th
  Annual Allerton Conference on}, pages 704--711. IEEE, 2010.

\bibitem[Basri and Jacobs(2003)]{Basri:PAMI03}
R.~Basri and D.~W. Jacobs.
\newblock Lambertian reflectance and linear subspaces.
\newblock \emph{IEEE transactions on pattern analysis and machine
  intelligence}, 25\penalty0 (2):\penalty0 218--233, 2003.

\bibitem[Beck(1984)]{Beck-Mathematika84}
J.~Beck.
\newblock Sums of distances between points on a sphere---an application of the
  theory of irregularities of distribution to discrete geometry.
\newblock \emph{Mathematika}, 31\penalty0 (01):\penalty0 33--41, 1984.

\bibitem[Brauchart and Grabner(2015)]{Brauchart:JC15}
J.~S. Brauchart and P.~J. Grabner.
\newblock Distributing many points on spheres: minimal energy and designs.
\newblock \emph{Journal of Complexity}, 31\penalty0 (3):\penalty0 293--326,
  2015.

\bibitem[Brooks et~al.(2013)Brooks, Dul{\'a}, and Boone]{Brooks:CSDA13}
J.P. Brooks, J.H. Dul{\'a}, and E.L. Boone.
\newblock A pure l1-norm principal component analysis.
\newblock \emph{Computational statistics \& data analysis}, 61:\penalty0
  83--98, 2013.

\bibitem[Cand\`es and Wakin(2008)]{Candes:SPM08}
E.~Cand\`es and M.~Wakin.
\newblock An introduction to compressive sampling.
\newblock \emph{IEEE Signal Processing Magazine}, 25\penalty0 (2):\penalty0
  21--30, 2008.

\bibitem[Cand\`es et~al.(2008)Cand\`es, Wakin, and Boyd]{Candes:JFAA08}
E.~Cand\`es, M.~Wakin, and S.~Boyd.
\newblock Enhancing sparsity by reweighted $\ell_1$ minimization.
\newblock \emph{Journal of Fourier Analysis and Applications}, 14\penalty0
  (5):\penalty0 877--905, 2008.

\bibitem[Cand\`es et~al.(2011)Cand\`es, Li, Ma, and Wright]{Candes:ACM11}
E.~Cand\`es, X.~Li, Y.~Ma, and J.~Wright.
\newblock Robust principal component analysis?
\newblock \emph{Journal of the ACM}, 58\penalty0 (3), 2011.

\bibitem[Chartrand and Yin(2008)]{Chartrand:ICASSP08}
R.~Chartrand and W.~Yin.
\newblock Iteratively reweighted algorithms for compressive sensing.
\newblock In \emph{2008 IEEE International Conference on Acoustics, Speech and
  Signal Processing}, pages 3869--3872. IEEE, 2008.

\bibitem[Daubechies et~al.(2010)Daubechies, DeVore, Fornasier, and
  G{\"u}nt{\"u}rk]{Daubechies:CPAM10}
I.~Daubechies, R.~DeVore, M.~Fornasier, and C.~S. G{\"u}nt{\"u}rk.
\newblock Iteratively reweighted least squares minimization for sparse
  recovery.
\newblock \emph{Communications on Pure and Applied Mathematics}, 63\penalty0
  (1):\penalty0 1--38, 2010.

\bibitem[Dick(2014)]{Dick-AANT14}
J.~Dick.
\newblock Applications of geometric discrepancy in numerical analysis and
  statistics.
\newblock \emph{Applied Algebra and Number Theory}, 2014.

\bibitem[Elden(2002)]{Elden:NM02}
L.~Elden.
\newblock Solving quadratically constrained least squares problems using a
  differential-geometric approach.
\newblock \emph{BIT Numerical Mathematics}, 42\penalty0 (2):\penalty0 323--335,
  2002.

\bibitem[Elhamifar and Vidal(2011)]{Elhamifar:CVPR11}
E.~Elhamifar and R.~Vidal.
\newblock Robust classification using structured sparse representation.
\newblock In \emph{{IEEE} Conference on Computer Vision and Pattern
  Recognition}, 2011.

\bibitem[Elhamifar and Vidal(2013)]{Elhamifar:TPAMI13}
E.~Elhamifar and R.~Vidal.
\newblock Sparse subspace clustering: Algorithm, theory, and applications.
\newblock \emph{{IEEE} Transactions on Pattern Analysis and Machine
  Intelligence}, 35\penalty0 (11):\penalty0 2765--2781, 2013.

\bibitem[Fei-Fei et~al.(2007)Fei-Fei, Fergus, and Perona]{Fei-Fe07}
L.~Fei-Fei, R.~Fergus, and P.~Perona.
\newblock Learning generative visual models from few training examples: An
  incremental bayesian approach tested on 101 object categories.
\newblock \emph{Comput. Vis. Image Underst.}, 106\penalty0 (1):\penalty0
  59--70, April 2007.
\newblock ISSN 1077-3142.
\newblock \doi{10.1016/j.cviu.2005.09.012}.
\newblock URL \url{http://dx.doi.org/10.1016/j.cviu.2005.09.012}.

\bibitem[Feng et~al.(2013)Feng, Xu, and Yan]{Feng:NIPS13}
J.~Feng, H.~Xu, and S.~Yan.
\newblock Online robust pca via stochastic optimization.
\newblock In \emph{Advances in Neural Information Processing Systems}, pages
  404--412, 2013.

\bibitem[Fischler and Bolles(1981)]{RANSAC}
M.~A. Fischler and R.~C. Bolles.
\newblock {RANSAC} random sample consensus: A paradigm for model fitting with
  applications to image analysis and automated cartography.
\newblock \emph{Communications of the ACM}, 26:\penalty0 381--395, 1981.

\bibitem[Gabay and Mercier(1976)]{Gabay:CMA76}
D.~Gabay and B.~Mercier.
\newblock A dual algorithm for the solution of nonlinear variational problems
  via finite-element approximations.
\newblock \emph{Comp. Math. Appl.}, 2:\penalty0 17--40, 1976.

\bibitem[Gander(1980)]{Gander:Numerische80}
W.~Gander.
\newblock Least squares with a quadratic constraint.
\newblock \emph{Numerische Mathematik}, 36\penalty0 (3):\penalty0 291--307,
  1980.

\bibitem[Georghiades et~al.(2001)Georghiades, Belhumeur, and
  Kriegman]{Georghiades01}
A.S. Georghiades, P.N. Belhumeur, and D.J. Kriegman.
\newblock From few to many: Illumination cone models for face recognition under
  variable lighting and pose.
\newblock \emph{IEEE Trans. Pattern Anal. Mach. Intelligence}, 23\penalty0
  (6):\penalty0 643--660, 2001.

\bibitem[Golub and Von~Matt(1991)]{Golub:Numerische91}
G.~H Golub and U.~Von~Matt.
\newblock Quadratically constrained least squares and quadratic problems.
\newblock \emph{Numerische Mathematik}, 59\penalty0 (1):\penalty0 561--580,
  1991.

\bibitem[Grabner and Tichy(1993)]{Grabner:MathComp93}
P.~J. Grabner and R.F. Tichy.
\newblock Spherical designs, discrepancy and numerical integration.
\newblock \emph{Math. Comp.}, 60\penalty0 (201):\penalty0 327--336, 1993.
\newblock ISSN 0025-5718.
\newblock \doi{10.2307/2153170}.
\newblock URL \url{http://dx.doi.org/10.2307/2153170}.

\bibitem[Grabner et~al.(1997)Grabner, Klinger, and Tichy]{Grabner:MR97}
P.~J. Grabner, B.~Klinger, and R.F. Tichy.
\newblock Discrepancies of point sequences on the sphere and numerical
  integration.
\newblock \emph{Mathematical Research}, 101:\penalty0 95--112, 1997.

\bibitem[Gurobi~Optimization(2015)]{gurobi}
Inc. Gurobi~Optimization.
\newblock Gurobi optimizer reference manual, 2015.
\newblock URL \url{http://www.gurobi.com}.

\bibitem[Harman(2010)]{Harman:UDT10}
G.~Harman.
\newblock Variations on the koksma-hlawka inequality.
\newblock \emph{Uniform Distribution Theory}, 5\penalty0 (1):\penalty0 65--78,
  2010.

\bibitem[Hlawka(1971)]{Hlawka:SPM71}
E.~Hlawka.
\newblock Discrepancy and riemann integration.
\newblock \emph{Studies in Pure Mathematics}, pages 121--129, 1971.

\bibitem[Hotelling(1933)]{Hotelling-1933}
H.~Hotelling.
\newblock Analysis of a complex of statistical variables into principal
  components.
\newblock \emph{Journal of Educational Psychology}, 24:\penalty0 417--441,
  1933.

\bibitem[Huber(1981)]{Huber-1981}
P.~Huber.
\newblock \emph{Robust Statistics}.
\newblock John Wiley \& Sons, New York, 1981.

\bibitem[Jolliffe(2002)]{Jolliffe-2002}
I.~Jolliffe.
\newblock \emph{Principal Component Analysis}.
\newblock Springer-Verlag, 2nd edition, 2002.

\bibitem[Ku et~al.(1995)Ku, Storer, and Georgakis]{Ku:Chemometrics95}
W.~Ku, R.~H. Storer, and C.~Georgakis.
\newblock Disturbance detection and isolation by dynamic principal component
  analysis.
\newblock \emph{Chemometrics and Intelligent Laboratory Systems}, 30:\penalty0
  179--196, 1995.

\bibitem[Kuipers and Niederreiter(2012)]{Kuipers:UDS12}
L.~Kuipers and H.~Niederreiter.
\newblock Uniform distribution of sequences.
\newblock \emph{Courier Corporation}, 2012.

\bibitem[Lerman and Maunu(2017)]{Lerman:IAI17}
G.~Lerman and T.~Maunu.
\newblock Fast, robust and non-convex subspace recovery.
\newblock \emph{Information and Inference: A Journal of the IMA}, 2017.

\bibitem[Lerman and Maunu(2018)]{Lerman:arXiv18}
G.~Lerman and T.~Maunu.
\newblock An overview of robust subspace recovery.
\newblock \emph{arXiv:1803.01013v1}, 2018.

\bibitem[Lerman and Zhang(2014)]{Lerman:CA14}
G.~Lerman and T.~Zhang.
\newblock $\ell_p$-recovery of the most significant subspace among multiple
  subspaces with outliers.
\newblock \emph{Constructive Approximation}, 40\penalty0 (3):\penalty0
  329--385, 2014.

\bibitem[Lerman et~al.(2015)Lerman, McCoy, Tropp, and Zhang]{Lerman:FCM15}
G.~Lerman, M.~B. McCoy, J.~A Tropp, and T.~Zhang.
\newblock Robust computation of linear models by convex relaxation.
\newblock \emph{Foundations of Computational Mathematics}, 15\penalty0
  (2):\penalty0 363--410, 2015.

\bibitem[Liu et~al.(2010)Liu, Lin, and Yu]{Liu:ICML10}
G.~Liu, Z.~Lin, and Y.~Yu.
\newblock Robust subspace segmentation by low-rank representation.
\newblock In \emph{International Conference on Machine Learning}, pages
  663--670, 2010.

\bibitem[Loyd et~al.(2014)Loyd, Mohseni, and
  Rebentrost]{Loyd:NaturePhysics2014}
S.~Loyd, M.~Mohseni, and P.~Rebentrost.
\newblock Quantum principal component analysis.
\newblock \emph{Nature Physics}, 10\penalty0 (9):\penalty0 631--633, 2014.

\bibitem[Moore(1981)]{moore1981principal}
B.~C. Moore.
\newblock Principal component analysis in linear systems: Controllability,
  observability, and model reduction.
\newblock \emph{IEEE Transactions on Automatic Control}, 26\penalty0
  (1):\penalty0 17--32, 1981.

\bibitem[Nam et~al.(2013)Nam, Davies, Elad, and Gribonval]{nam2013cosparse}
S.~Nam, M.E. Davies, M.~Elad, and R.~Gribonval.
\newblock The cosparse analysis model and algorithms.
\newblock \emph{Applied and Computational Harmonic Analysis}, 34\penalty0
  (1):\penalty0 30--56, 2013.

\bibitem[Pearson(1901)]{Pearson-1901}
K.~Pearson.
\newblock On lines and planes of closest fit to systems of points in space.
\newblock \emph{The London, Edinburgh and Dublin Philosphical Magazine and
  Journal of Science}, 2:\penalty0 559--572, 1901.

\bibitem[Price et~al.(2006)Price, Patterson, Plenge, Weinblatt, Shadick, and
  Reich]{Price:NatureGenetics2006}
A.~Price, N.~Patterson, R.~M. Plenge, M.~E. Weinblatt, N.~A. Shadick, and
  D.~Reich.
\newblock Principal components analysis corrects for stratification in
  genome-wide association studies.
\newblock \emph{Nature Genetics}, 38\penalty0 (8):\penalty0 904--909, 2006.

\bibitem[Qu et~al.(2014)Qu, Sun, and Wright]{Qu:NIPS14}
Q.~Qu, J.~Sun, and J.~Wright.
\newblock Finding a sparse vector in a subspace: Linear sparsity using
  alternating directions.
\newblock In \emph{Advances in Neural Information Processing Systems}, pages
  3401--3409, 2014.

\bibitem[Rahmani and Atia(2017)]{Rahmani:arXiv17}
M.~Rahmani and G.~Atia.
\newblock Coherence pursuit: Fast, simple, and robust principal component
  analysis.
\newblock \emph{arXiv:1609.04789v3}, 2017.

\bibitem[Soltanolkotabi and Cand\`es(2012)]{Soltanolkotabi:AS12}
M.~Soltanolkotabi and E.~J. Cand\`es.
\newblock A geometric analysis of subspace clustering with outliers.
\newblock \emph{Annals of Statistics}, 40\penalty0 (4):\penalty0 2195--2238,
  2012.

\bibitem[Sp{\"a}th and Watson(1987)]{Spath:Numerische87}
H.~Sp{\"a}th and G.A. Watson.
\newblock On orthogonal linear $\ell_1$ approximation.
\newblock \emph{Numerische Mathematik}, 51\penalty0 (5):\penalty0 531--543,
  1987.

\bibitem[Spielman et~al.(2013)Spielman, Wang, and Wright]{spielman2013exact}
D.A. Spielman, H.~Wang, and J.~Wright.
\newblock Exact recovery of sparsely-used dictionaries.
\newblock In \emph{Proceedings of the 23d international joint conference on
  Artificial Intelligence}, pages 3087--3090. AAAI Press, 2013.

\bibitem[Sun et~al.(2015{\natexlab{a}})Sun, Qu, and
  Wright]{Sun:CompleteIIarXiv15}
J.~Sun, Q.~Qu, and J.~Wright.
\newblock Complete dictionary recovery over the sphere ii: Recovery by
  riemannian trust-region method.
\newblock \emph{arXiv preprint arXiv:1511.04777}, 2015{\natexlab{a}}.

\bibitem[Sun et~al.(2015{\natexlab{b}})Sun, Qu, and
  Wright]{Sun:CompleteIarXiv15}
J.~Sun, Q.~Qu, and J.~Wright.
\newblock Complete dictionary recovery over the sphere i: Overview and the
  geometric picture.
\newblock \emph{arXiv preprint arXiv:1511.03607}, 2015{\natexlab{b}}.

\bibitem[Sun et~al.(2015{\natexlab{c}})Sun, Qu, and Wright]{Sun:ICML15}
J.~Sun, Q.~Qu, and J.~Wright.
\newblock Complete dictionary recovery using nonconvex optimization.
\newblock In \emph{Proceedings of the 32nd International Conference on Machine
  Learning (ICML-15)}, pages 2351--2360, 2015{\natexlab{c}}.

\bibitem[Sun et~al.(2015{\natexlab{d}})Sun, Qu, and Wright]{Sun:SampTA15}
J.~Sun, Q.~Qu, and J.~Wright.
\newblock Complete dictionary recovery over the sphere.
\newblock In \emph{Sampling Theory and Applications (SampTA), 2015
  International Conference on}, pages 407--410. IEEE, 2015{\natexlab{d}}.

\bibitem[Tsakiris and Vidal(2017)]{Tsakiris:ICML17}
M.~C. Tsakiris and R.~Vidal.
\newblock Hyperplane clustering via dual principal component pursuit.
\newblock In \emph{International Conference on Machine Learning}, 2017.

\bibitem[Tsakiris and Vidal(2018)]{Tsakiris:AffinePAMI17}
M.~C. Tsakiris and R.~Vidal.
\newblock Algebraic clustering of affine subspaces.
\newblock \emph{{IEEE} Transactions on Pattern Analysis and Machine
  Intelligence}, 2\penalty0 (40):\penalty0 482, 2018.

\bibitem[Tsakiris and Vidal(2015)]{Tsakiris:DPCPICCV15}
M.C. Tsakiris and R.~Vidal.
\newblock Dual principal component pursuit.
\newblock In \emph{ICCV Workshop on Robust Subspace Learning and Computer
  Vision}, pages 10--18, 2015.

\bibitem[Vidal(2011)]{Vidal:SPM11-SC}
R.~Vidal.
\newblock Subspace clustering.
\newblock \emph{{IEEE} Signal Processing Magazine}, 28\penalty0 (3):\penalty0
  52--68, March 2011.

\bibitem[Vyas and Kumaranayake(2006)]{Vyas:HPP2006}
S.~Vyas and L.~Kumaranayake.
\newblock Constructing socio-economic status indices: how to use principal
  components analysis.
\newblock \emph{Health Policy and Planning}, 21:\penalty0 459--468, 2006.

\bibitem[Wang et~al.(2015)Wang, Dicle, Sznaier, and Camps]{WangCamps:CVPR15}
Y.~Wang, C.~Dicle, M.~Sznaier, and O.~Camps.
\newblock Self scaled regularized robust regression.
\newblock In \emph{Proceedings of the IEEE Conference on Computer Vision and
  Pattern Recognition}, pages 3261--3269, 2015.

\bibitem[Xu et~al.(2012)Xu, Caramanis, and Sanghavi]{Xu:TIT12}
H.~Xu, C.~Caramanis, and S.~Sanghavi.
\newblock Robust pca via outlier pursuit.
\newblock \emph{IEEE transactions on information theory}, 58\penalty0
  (5):\penalty0 3047--3064, 2012.

\bibitem[You et~al.(2017)You, Robinson, and Vidal]{You:CVPR17}
C.~You, D.~Robinson, and R.~Vidal.
\newblock Provable self-representation based outlier detection in a union of
  subspaces.
\newblock In \emph{{IEEE} Conference on Computer Vision and Pattern
  Recognition}, 2017.

\bibitem[Zhang and Lerman(2014)]{Zhang:JMLR14}
Teng Zhang and Gilad Lerman.
\newblock A novel m-estimator for robust pca.
\newblock \emph{The Journal of Machine Learning Research}, 15\penalty0
  (1):\penalty0 749--808, 2014.

\end{thebibliography}


\end{document}